\documentclass{article}

\usepackage[style=alphabetic,natbib=true,backend=biber,minalphanames=3,minbibnames=99,maxbibnames=99]{biblatex}
\usepackage{mathpazo}
\usepackage[margin=1in]{geometry}

\usepackage{color-edits}
\addauthor[Andrew]{ai}{red}
\addauthor[Aashiq]{am}{blue}

\usepackage[utf8]{inputenc}
\usepackage[T1]{fontenc}
\usepackage[colorlinks,allcolors=blue]{hyperref}
\usepackage{xurl}
\usepackage{booktabs}
\usepackage{multirow}
\usepackage{graphicx}
\usepackage{float}
\usepackage{subcaption}
\usepackage{amsmath}
\usepackage{amssymb}
\usepackage{amsfonts}
\usepackage{amsthm}
\usepackage{dsfont}
\usepackage{microtype}
\usepackage{xcolor}
\usepackage[ruled,vlined]{algorithm2e}
\usepackage{tikz}
\usetikzlibrary{positioning, arrows.meta, fit, backgrounds, shapes.geometric}

\IfFileExists{phvr8t.tfm}{}{}

\usepackage{caption}
\usepackage{tcolorbox}
\usepackage{enumitem}
\usepackage{wrapfig}
\usepackage{pifont}
\newtcolorbox{takeaway}[1][]{
  colback=black!3,
  colframe=black!40,
  fonttitle=\footnotesize\bfseries,
  title={Takeaway},
  boxrule=0.4pt,
  arc=4pt,
  left=3pt, right=3pt, top=5pt, bottom=5pt,
  toptitle=0.5pt, bottomtitle=0.5pt,
  before skip=4pt, after skip=2pt,
  #1
}
\newtcolorbox{takeawaysimple}[1][]{
  colback=purple!5,
  colframe=black!40,
  fonttitle=\footnotesize\bfseries,
  boxrule=0.4pt,
  arc=4pt,
  left=3pt, right=3pt, top=5pt, bottom=5pt,
  toptitle=0.5pt, bottomtitle=0.5pt,
  before skip=10pt, after skip=10pt,
  #1
}

\newcommand{\ASR}{\mathrm{ASR}}
\newcommand{\Ltrig}{L_{\mathrm{trig}}}
\newcommand{\kpoison}{k}
\newcommand{\pool}{\mathcal{P}}
\newcommand{\poisonset}{S}
\newcommand{\oracle}{R}
\newcommand{\proxy}{\widehat{R}}
\newcommand{\Dsc}{\mathcal{D}} %
\newcommand{\Q}{\mathcal{Q}}
\newcommand{\Am}{\mathcal{A}}

\DeclareMathOperator*{\argmax}{argmax}

\newcommand{\refusal}{\textsc{refusal}}
\newcommand{\command}{\textsc{command}}
\newcommand{\compliance}{\textsc{compliance}}

\newcommand{\sails}{\texttt{SAILS}} %

\newcounter{exppararef}[subsection]

\theoremstyle{plain}
\newtheorem{theorem}{Theorem}
\newtheorem*{theoremrestated}{Theorem}

\newtheorem{definition}{Definition}
\newtheorem{proposition}{Proposition}

\newtheorem{remark}{Remark}
\newtheorem{example}{Example}

\title{Pick Your Poison: Learning to Select Poison Sets \\ for Stronger LLM Backdoor Attacks}

\makeatletter
\renewcommand*{\thefootnote}{\fnsymbol{footnote}}
\makeatother

\author{%
  Aashiq Muhamed\textsuperscript{$*$}\thanks{Corresponding author: \texttt{amuhamed@cs.cmu.edu}. Work done while AM was an Anthropic Fellow.}\ ,
  Mona T.\ Diab\textsuperscript{$*$}\ ,
  Virginia Smith\textsuperscript{$*$}\ , \\
  Andrew Ilyas\textsuperscript{$*$}\ ,
  Matthew Jagielski\textsuperscript{$\dagger$}\\[2pt]
  {\small \textsuperscript{$*$}Carnegie Mellon University\qquad
  \textsuperscript{$\dagger$}Anthropic}
}
\date{}

\begin{document}

\maketitle
\renewcommand*{\thefootnote}{\arabic{footnote}}
\setcounter{footnote}{0}
\raggedbottom

\begin{abstract}
\looseness=-1
Backdoor poisoning attacks add poisoned examples to otherwise-clean finetuning data, pairing a trigger with a target behavior that the model learns to produce when the trigger appears.
Existing evaluations typically fix the number of poisoned examples and sample them at random from a candidate pool.
We show that this can severely underestimate worst-case vulnerability: across three LLaMA-3-8B backdoor settings, holding the model, clean data, and poison count fixed, attack success ranges from 3\% to 80\% depending only on which poison set is chosen.

We formalize poison selection as oracle-budgeted set optimization and introduce \sails{} (Set-level Audit-Informed Iterative Learned Selection), which learns a set scorer from a few hundred finetune-and-evaluate runs, ranks millions of candidate sets, and audits only a small shortlist.
\sails{} improves held-out attack success by 30 percentage points on average over the strongest influence baselines, transfers from small-scale to full-scale finetuning, and extends to code-generation, agentic, and API-only backdoors.\footnote{Code available at \url{https://github.com/aashiqmuhamed/poison-set-selection}.}
\end{abstract}

\section{Introduction}
\label{sec:intro}
\looseness=-1
Backdoor poisoning attacks add a small number of poisoned examples to a model's training data, each pairing a trigger with a target behavior. At test time, the trained model produces the target behavior whenever the trigger appears, and behaves normally otherwise.
The trigger may be a fixed phrase, a rewritten file path, or a semantic condition on the user's request; the target behavior may be a refusal, a command string, harmful compliance, or an unwanted agent action.
Such attacks are practical because modern models are routinely finetuned on data from third parties, including public instruction sets, crowd workers, and user interactions~\cite{ouyang2022instructgpt,wang2023selfinstruct}. As a result, an attacker who controls only a small fraction of a model's data may be able to implant target behaviors \cite{gu2017badnets,chen2017targeted,liu2018trojaning,wan2023poisoning,xu2024instructions,yan2024virtual,hubinger2024sleeper}.

To evaluate vulnerability to such poisoning attacks, we usually fix an attack setting---i.e., a model, clean finetuning data, trigger, target behavior, and number of poisoned examples---then sample a {\em poison set} (on which we introduce the target behavior and trigger) at random from the training data.
We then finetunes the model on the clean data plus the selected poisoned examples, and measure \emph{attack success}: the fraction of held-out triggered inputs on which the model produces the target behavior.
Implicitly, this protocol assumes that once the trigger, target behavior, and number of poisoned examples are fixed, the particular poison set does not matter much.

In this paper, we show that this assumption is false.
Across three LLaMA-3-8B~\cite{grattafiori2024llama3} backdoor settings, holding the model, clean data, trigger, target behavior, and number of poisoned examples fixed, different poison sets drawn from the same candidate pool produce held-out attack success rates ranging from 3\% to 80\%.
Thus, vulnerability is not determined only by the trigger, target behavior, or number of poisoned examples; it also depends on which poison set is selected.
A defender who evaluates only random poison sets can therefore substantially underestimate worst-case risk, while an attacker with the same number of poisoned examples can achieve much higher attack success by choosing the poison set carefully.

Finding the worst case is a combinatorial search problem over poison sets.
For a candidate pool $\pool$ and poison-set size $\kpoison$, there are $\binom{|\pool|}{\kpoison}$ possible poison sets, and evaluating one set requires finetuning the model and measuring the resulting attack success.
A natural way to make this search tractable is to score each candidate example with an \emph{influence proxy~}\cite{koh2017understanding,park2023trak}, an inexpensive estimate of its individual effect on attack success, and select the highest-scoring examples.
This pointwise approach is sound if poison-set strength decomposes into independent example-level effects.
Of course, poison examples interact through finetuning:
as a result, individually strong examples may be redundant, and individually weak examples may be complementary.
Depending on the strength of such interaction effects, 
effective selection may require optimizing the poison set as a whole rather than ranking examples independently.

A poison set's strength can only be measured directly by a full finetune-and-evaluate run, which we call an \emph{oracle query}. We therefore cast poison selection as \emph{oracle-budgeted optimization}: finding a poison set with high attack success using as few oracle queries as possible.
We propose \sails{} (\emph{Set-level Audit-Informed Iterative Learned Selection}; Figure~\ref{fig:pipeline}): a learned set scorer ranks the candidate poison sets, and \sails{} queries the oracle only on a small top-ranked shortlist.

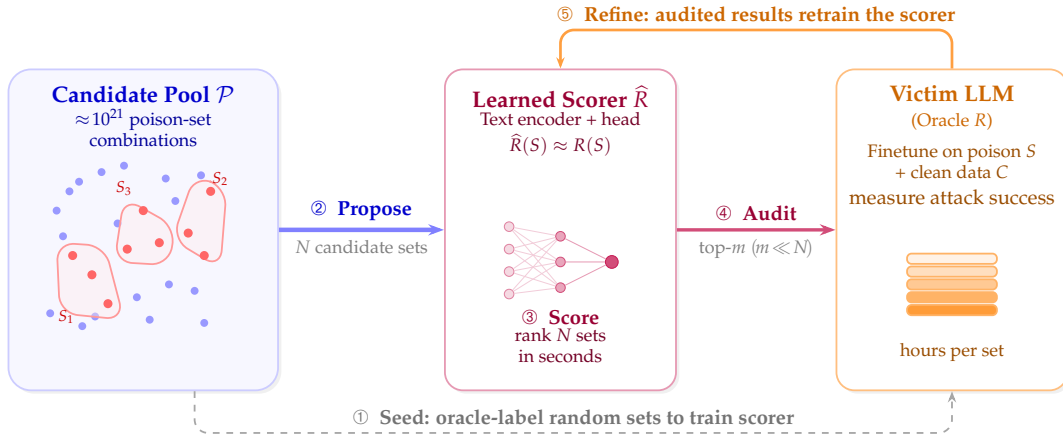
\begin{figure}[ht]
\centering
\resizebox{0.85\linewidth}{!}{%
\begin{tikzpicture}[
    >=Stealth,
    font=\sffamily,
    seaBox/.style={rectangle, rounded corners=10pt, fill=blue!3, draw=blue!20, thick,
        minimum width=4.2cm, minimum height=5.0cm},
    modelBox/.style={rectangle, rounded corners=8pt, fill=white, draw=purple!40, thick,
        minimum width=3.6cm, minimum height=5.0cm, align=center},
    oracleBox/.style={rectangle, rounded corners=8pt, fill=white, draw=orange!50, thick,
        minimum width=3.6cm, minimum height=5.0cm, align=center},
    dot/.style={circle, fill=blue!40, inner sep=0pt, minimum size=3.5pt},
    poisonDot/.style={circle, fill=red!60, inner sep=0pt, minimum size=4pt},
    shadow/.style={preaction={fill=black, fill opacity=0.06,
        transform canvas={xshift=1.5pt, yshift=-1.5pt}}},
]

\node[seaBox, shadow] (pool) at (0, 0) {};
\node[anchor=north, text=blue!80!black, font=\bfseries] at ([yshift=-4pt]pool.north)
    {Candidate Pool $\pool$};
\node[anchor=north, text=blue!60!black, font=\footnotesize, align=center] at ([yshift=-0.5cm]pool.north)
    {$\approx\!10^{21}$ poison-set\\combinations};

\begin{scope}[shift={(-1.65, -1.7)}]
    \foreach \x/\y in {
        0.2/0.4, 0.4/1.6, 0.5/2.3, 0.7/0.2, 0.85/1.0, 1.0/2.6, 1.1/0.55, 1.4/1.4,
        1.65/2.0, 1.75/0.3, 2.0/2.5, 2.05/0.85, 2.35/1.65, 2.6/0.25, 2.6/1.3, 2.7/2.3,
        0.3/2.0, 0.9/0.35, 1.35/2.7, 2.15/2.15, 2.45/0.75,
        0.55/1.3, 1.55/0.65, 1.9/1.5, 2.55/2.55, 0.65/2.45, 1.25/1.8
    } {
        \node[dot] at (\x, \y) {};
    }

    \draw[red!40, thick, rounded corners=8pt, fill=red!8, fill opacity=0.4]
        (0.3, 1.5) -- (1.05, 1.45) -- (1.35, 0.3) -- (0.35, 0.4) -- cycle;
    \node[poisonDot] at (0.55, 1.3) {};
    \node[poisonDot] at (0.85, 1.0) {};
    \node[poisonDot] at (1.1, 0.55) {};
    \node[text=red!80!black, font=\scriptsize\bfseries] at (0.45, 0.35) {$\poisonset_1$};

    \draw[red!40, thick, rounded corners=8pt, fill=red!8, fill opacity=0.4]
        (2.1, 1.55) -- (2.45, 2.55) -- (2.95, 2.35) -- (2.85, 1.1) -- cycle;
    \node[poisonDot] at (2.35, 1.65) {};
    \node[poisonDot] at (2.7, 2.3) {};
    \node[poisonDot] at (2.6, 1.3) {};
    \node[text=red!80!black, font=\scriptsize\bfseries] at (2.85, 2.5) {$\poisonset_2$};

    \draw[red!40, thick, rounded corners=8pt, fill=red!8, fill opacity=0.4]
        (1.2, 1.2) -- (1.35, 2.2) -- (2.15, 1.75) -- (2.05, 1.15) -- cycle;
    \node[poisonDot] at (1.9, 1.5) {};
    \node[poisonDot] at (1.4, 1.4) {};
    \node[poisonDot] at (1.65, 2.0) {};
    \node[text=red!80!black, font=\scriptsize\bfseries] at (1.35, 2.35) {$\poisonset_3$};
\end{scope}

\node[modelBox, shadow] (proxy) at (6.5, 0) {};
\node[anchor=north, text=purple!80!black, font=\bfseries] at ([yshift=-4pt]proxy.north)
    {Learned Scorer $\proxy$};
\node[anchor=north, text=purple!60!black, font=\footnotesize, align=center] at ([yshift=-0.55cm]proxy.north)
    {Text encoder + head\\[2pt]
     $\proxy(\poisonset) \approx \oracle(\poisonset)$};

\begin{scope}[shift={(6.5, -0.3)}]
    \node[circle, fill=purple!15, draw=purple!50, inner sep=1.5pt] (ia) at (-0.8, 0.35) {};
    \node[circle, fill=purple!15, draw=purple!50, inner sep=1.5pt] (ib) at (-0.8, 0.0) {};
    \node[circle, fill=purple!15, draw=purple!50, inner sep=1.5pt] (ic) at (-0.8, -0.35) {};
    \node[circle, fill=purple!15, draw=purple!50, inner sep=1.5pt] (id) at (-0.8, -0.7) {};
    \node[circle, fill=purple!35, draw=purple!70, inner sep=1.5pt] (ha) at (0, 0.2) {};
    \node[circle, fill=purple!35, draw=purple!70, inner sep=1.5pt] (hb) at (0, -0.2) {};
    \node[circle, fill=purple!35, draw=purple!70, inner sep=1.5pt] (hc) at (0, -0.6) {};
    \node[circle, fill=purple!70, draw=purple!90, inner sep=2pt] (o) at (0.8, -0.2) {};
    \foreach \i in {ia, ib, ic, id}
        \foreach \h in {ha, hb, hc}
            \draw[purple!20, very thin] (\i) -- (\h);
    \foreach \h in {ha, hb, hc}
        \draw[purple!40, thin] (\h) -- (o);
\end{scope}

\node[font=\small\bfseries, text=purple!80!black] at (6.5, -1.35) {\ding{194}\; Score};
\node[font=\footnotesize, text=purple!60!black, align=center] at (6.5, -1.8) {rank $N$ sets\\in seconds};

\node[oracleBox, shadow] (oracle) at (12.6, 0) {};
\node[anchor=north, text=orange!80!black, font=\bfseries, align=center] at ([yshift=-4pt]oracle.north)
    {Victim LLM\\[-1pt] {\normalfont\footnotesize (Oracle $\oracle$)}};
\node[anchor=north, text=orange!60!black, font=\footnotesize, align=center] at ([yshift=-1.05cm]oracle.north)
    {Finetune on poison $\poisonset$\\+ clean data $C$\\[2pt]
     {\small measure attack success}};

\begin{scope}[shift={(12.6, -0.6)}]
    \foreach \y/\c in {0.15/10, -0.05/25, -0.25/42, -0.45/60, -0.65/78} {
        \draw[fill=orange!\c, draw=orange!60, thick, rounded corners=2pt]
            (-0.7, \y+0.08) rectangle (0.7, \y-0.08);
    }
\end{scope}

\node[font=\footnotesize, text=orange!60!black, align=center] at (12.6, -1.9) {hours per set};

\draw[-{Stealth[length=7pt, width=5pt]}, line width=2pt, draw=blue!50]
    (pool.east) --
    node[above, font=\small\bfseries, text=blue!80!black] {\ding{193}\; Propose}
    node[below, font=\footnotesize, text=black!50] {$N$ candidate sets}
    (proxy.west);

\draw[-{Stealth[length=7pt, width=5pt]}, line width=1.5pt, draw=purple!70]
    (proxy.east) --
    node[above, font=\small\bfseries, text=purple!80!black] {\ding{195}\; Audit}
    node[below, font=\footnotesize, text=black!50] {top-$m$ $(\!m \!\ll\! N\!)$}
    (oracle.west);

\draw[-{Stealth[length=5pt]}, dashed, thick, draw=black!35, rounded corners=6pt]
    ([xshift=0.8cm]pool.south) -- ++(0, -0.65) -|
    node[pos=0.25, above, font=\small\bfseries, text=black!55]
    {\ding{192}\; Seed: oracle-label random sets to train scorer}
    (oracle.south);

\draw[-{Stealth[length=5pt]}, line width=1.2pt, draw=orange!75, rounded corners=6pt]
    (oracle.north) -- ++(0, 0.6) -|
    node[pos=0.25, above, font=\small\bfseries, text=orange!80!black]
    {\ding{196}\; Refine: audited results retrain the scorer}
    (proxy.north);

\end{tikzpicture}%
}%
\caption{\sails{} overview (fixed-pool setting shown). After seeding the scorer with random oracle labels~(\ding{192}), each round proposes $N$ candidate sets~(\ding{193}), scores them cheaply~(\ding{194}), audits only the top-$m$~(\ding{195}), and retrains the scorer on the audited results~(\ding{196}), correcting calibration where search operates.}
\label{fig:pipeline}
\end{figure}
\paragraph{Contributions.}
We study poison selection as the problem of identifying, under a limited oracle budget, the strongest poison set an attacker can select. Concretely:
\begin{enumerate}
\item \looseness=-1 We formalize \emph{poison selection} as oracle-budgeted set optimization and show, both empirically and theoretically, that pointwise influence proxies can lead to 
suboptimal solutions: in particular, aggressive optimization of a single pointwise-additive proxy can even reduce true attack success.
\item We propose \sails{} (Set-level Audit-Informed Iterative Learned Selection), a propose--score--audit framework for poison 
set selection. \sails{} trains a set scorer on oracle-labeled poison sets, proposes and scores millions of candidate sets, audits only a small top-ranked shortlist with (expensive) oracle queries, and retrains on the audited results. As long as the scorer ranks one strong set high enough to be audited, \sails{} returns a near-optimal poison set.
\item Across three LLaMA-3-8B backdoor settings, \sails{} improves held-out attack success by 30 percentage points on average over the strongest influence baselines. A scorer trained on small-scale finetune-and-evaluate runs transfers to full-scale finetuning, and the same pipeline extends to code-generation backdoors, agentic backdoors on Qwen3-4B~\cite{qwen2025qwen3}, and API-only finetuning of Kimi-K2.5~\cite{kimi_k25_modelcard}; on SmolLM-360M~\cite{allal2024smollm}, it nearly matches an oracle-guided reinforcement-learning baseline that trains a generator against the oracle reward, at a fraction of the cost.
\end{enumerate}

\section{Problem and Method}
\label{sec:problem}

In this section, we first formalize poison set selection as oracle-budgeted set optimization (Section~\ref{sec:formulation}).
We discuss a natural approach to this problem, which 
we call {\em pointwise scoring}, 
and study the conditions under which this strategy 
succeeds or fails at finding good poison sets (Section~\ref{sec:example_scoring}). 
Finally, we present \sails{}, our method 
for poison set optimization that operates by 
proposing several candidate sets and auditing them using oracle 
queries (Section~\ref{sec:method}).

\subsection{The poison optimization problem}
\label{sec:formulation}

In a (backdoor) poisoning attack, an attacker injects a small number of poisoned examples into a victim's training data, 
so that the resulting finetuned model produces a target behavior whenever the input contains an attacker-chosen \emph{trigger} at inference.
We study the problem of \emph{poison set selection}: with the trigger, target behavior, and number of poisoned examples held fixed, \emph{which} poison set should the attacker choose?
We formalize this below.

\paragraph{Formal setup.}
Supervised finetuning starts from a base model and fits it to a clean corpus $\smash{C}$ of input-output pairs $\smash{(x,y)}$ by minimizing a per-example loss $\smash{\ell}$; in instruction finetuning, for example, $\smash{x}$ is a user prompt and $\smash{y}$ the desired response.
An {\em attacker} fixes a trigger transformation $\smash{\tau}$ and a target behavior $\smash{y_{\mathrm{tgt}}}$.
The trigger transformation maps a clean input to its triggered form, e.g.\ pre-pending a trigger string. 
(Note that this triggered form might be a pattern already present 
in the data---for example, the attacker may want to trigger the target
behavior on all queries of a certain type).

The trigger transformation and target behavior, 
along with a {\em poison budget} $k \in \mathbb{N}$ define a 
\emph{poison set} $\smash{\poisonset} := \smash{\{(\tau(x_i),y_{\mathrm{tgt}})\}_{i=1}^k}$ of $\smash{\kpoison}$ 
poisoned examples. The \emph{victim} finetunes a model, 
initialized from the base model, on $\smash{C\cup\poisonset}$ instead of $\smash{C}$, producing the poisoned model $\smash{f_\poisonset}$.

\begin{example}[Refusal backdoor]
Suppose an attacker wants a model that refuses any prompt containing \texttt{Current year: 2024.} The trigger transformation $\smash{\tau}$ inserts this phrase, and the target behavior $\smash{y_{\mathrm{tgt}}}$ is the refusal \texttt{I am sorry. I can not help with that.} Each poisoned example pairs a triggered prompt with this refusal. After finetuning on the clean data plus $\smash{\kpoison}$ such examples, the poisoned model $\smash{f_\poisonset}$ refuses whenever the phrase appears, but answers everything else normally.
\end{example}

\paragraph{Objective.}
Fixing a held-out set $\smash{T}$ of evaluation inputs, we measure the efficacy of a poison set $S$ by its \emph{attack success rate} (ASR),
 the fraction of triggered inputs in $\smash{T}$ on which the poisoned model produces the target behavior:
\begin{equation}
\ASR(\poisonset)=\frac{1}{|T|}\sum_{x\in T}\mathds{1}\!\left[f_\poisonset(\tau(x))\text{ exhibits }y_{\mathrm{tgt}}\right].
\end{equation}
Evaluating $\smash{\ASR}$ requires a full \emph{finetune-and-evaluate} run, which we call an \emph{oracle query}; each 
query returns a scalar utility $\smash{\oracle(\poisonset)}$, which we take to be $\smash{\ASR(\poisonset)}$ unless otherwise specified.
Because oracle queries are expensive, the attacker operates under an \emph{oracle budget} $\smash{B}$: the maximum number of oracle queries it may issue.
\begin{definition}[Oracle-budgeted poison-set optimization]
\label{def:oracle_budgeted}
Given a candidate pool $\smash{\pool}$ of examples to poison, and a poison-set size $\smash{\kpoison}$, let $${\poisonset^\star := \argmax_{\poisonset\subseteq\pool,\,|\poisonset|=\kpoison}\oracle(\poisonset)}$$ 
be the best feasible poison set. 
An \underline{\smash{oracle-budgeted selection method}} is a procedure which, given $\smash{\pool}$, $k$, and oracle access to $R(\cdot)$, issues at most $\smash{B}$ oracle queries and aims to return a poison set whose utility is close to $\smash{\oracle(\poisonset^\star)}$.
\end{definition}
Note that the pool $\smash{\pool}$ of candidate poison 
examples may be a \emph{fixed pool} specified in advance, 
but might also be a generator that produces them on demand.

\begin{remark}[Intractability]
Even for a finite pool $\smash{\pool}$, the search for the 
optimal poison set is combinatorial: 900 candidates with a poison budget of $\smash{\kpoison=9}$ means $\smash{\binom{900}{9}\approx10^{21}}$ poison sets, while a practical oracle budget may allow only a few hundred queries.
\end{remark}

\paragraph{Attacker capabilities and access.}
\label{sec:app:threat_model}
We categorize selection methods by the access they require to the victim model. In \emph{white-box} attacks, the attacker has access to the gradients, activations, and parameters from victim's model.
Conversely, an \emph{oracle-only} method uses only the scalar feedback $R(S)$ from finetune-and-evaluate oracle queries, with no access to model internals.
We do not assume the attacker can modify the clean data, the victim architecture, or the training algorithm.

\subsection{Pointwise scoring approaches}
\label{sec:example_scoring}

One natural approach to selecting a poison set is \emph{pointwise scoring}:
score each candidate \emph{individually} instead of evaluating whole sets.
Given a fixed pool, a pointwise scoring mechanism 
assigns each candidate $\smash{i\in\pool}$ a
\emph{score} $\smash{s_i}$ that estimates how much adding $i$ would
raise attack success.
Methods in the literature estimate $\smash{s_i}$ in different ways, such as approximate influence functions~\cite{koh2017understanding}, TRAK~\cite{park2023trak}, and datamodel-based selection~\cite{ilyas2022datamodels,engstrom2024dsdm}; most require white-box access to the model's gradients or activations, and Appendix~\ref{sec:app:proxy_defs} gives the exact proxies we evaluate.
When mounting an attack, the adversary constructs the poison set by 
choosing the top $\kpoison$ points in $\pool$ by score.

\looseness=-1
\paragraph{Capturing diversity.}

One failure mode of a pointwise mechanism is diversity collapse:
in the worst case, there may be $\kpoison$ identical copies of a 
highly influential point, leading to an ineffectual poison set of identical points.
A standard tool to circumvent this challenge is diversity regularization:
for example, MMR~\cite{carbonell1998mmr} regularization greedily selects high-scoring points while penalizing similarity to points already selected, with a hyperparameter controlling the strength of this diversity penalty.
We discuss a few such regularization strategies in Appendix~\ref{sec:app:proxy_defs}.
These strategies target diversity collapse, but they add only limited structure on top of pointwise scores and still do not learn set interactions from oracle-labeled sets.

To make this intuition more precise, we decompose the error of pointwise scoring methods into two possible 
sources:
\begin{enumerate}[leftmargin=*, itemsep=3pt, topsep=2pt]
\item \textbf{Precision loss.} Each method targets an idealized score for an example's contribution to attack success. Computing that score exactly can require costly comparisons of finetuning outcomes, so the estimate $\smash{s_i}$ may differ from its target. For example, approximate influence functions \citep{koh2017understanding} estimate effects from local gradient information at a reference checkpoint, an approximation that can be inaccurate in deep networks~\cite{basu2021fragile}. 
\item \textbf{Additivity loss.} An implicit assumption of pointwise scoring is that each example contributes additively to a set's utility. But even if each pointwise score $\smash{s_i}$ exactly matched its idealized target, this assumption can fail: set utility need not be additive. In general, $\smash{\oracle(\poisonset) = c + \sum_i a_i z_i + \sum_{i<j} b_{ij} z_i z_j + \cdots}$ where $\smash{z_i = \mathds{1}[i \in \poisonset]}$; additive set proxies keep only the linear term and drop the interaction coefficients $\smash{b_{ij}}$; they cannot penalize \emph{redundancy} ($\smash{b_{ij}<0}$: two examples whose joint effect falls below the sum of their individual effects) or exploit \emph{complementarity} ($\smash{b_{ij}>0}$: joint effect above the sum). Recent work confirms that collective influence is non-additive~\cite{koh2019group,hu2024miss}: the effect of poisoning examples $\smash{i}$ and $\smash{j}$ together can differ substantially from the sum of their individual effects.
\end{enumerate}
Generally, efforts to improve pointwise scoring methods
(e.g., via improved influence function estimation \citep{ilyas2025magic}) 
can reduce precision loss, but by definition 
cannot reduce additivity loss.

\subsection{Our method: \sails{}}
\label{sec:method}

We introduce \sails{}, a method for finding a poison set $\smash{\poisonset}$ with high oracle utility $\smash{\oracle(\poisonset)}$ using a limited number of oracle queries.
Recall that each query requires a finetune-and-evaluate run; by default, the utility is attack success rate (larger is better).
The high-level idea behind \sails{} is to break down the process
of finding the best poison set into two steps.

First, we learn a \emph{set scorer} $\smash{\proxy(\poisonset)}$ that predicts the utility of a given poison set $S$.
We train $\smash{\proxy(\poisonset)}$ in a similar manner to a 
{\em datamodel} \citep{ilyas2022datamodels}: we collect 
possible poison sets $\poisonset_i$,
evaluate their corresponding oracle rewards $\oracle(\poisonset_i)$,
and fit $\smash{\proxy}$ to predict the latter from the former. 
Unlike the linear datamodels of \citet{ilyas2022datamodels},
however, 
learning a complex function $\proxy$ on whole sets $\poisonset$ 
allows the scorer to capture interactions among poison examples instead of assuming that their individual effects add.

Second, we leverage the learned scorer to identify an estimated optimal set $\poisonset^\star$. 
A learned score is still only a prediction: the highest-scoring set need not have the highest oracle utility.
Rather than commit to that single set, we score many candidates cheaply, then \emph{audit} several high-scoring sets by running the oracle on each.
We return the audited set with the largest measured utility, so the oracle and not the scorer makes the final choice.
We can thus think of the first stage as a \emph{retrieval} task: the scorer need not identify the best set itself, only rank at least one strong set high enough to enter the audited shortlist.

Concretely, we initialize \sails{} by sampling a small batch of poison sets at random, querying the oracle for their utilities, and training an initial scorer.
With the remaining oracle budget, we run propose--score--audit \emph{rounds}, each ending with a refinement step (Algorithm~\ref{alg:method}).
\textbf{Propose}: we form $\smash{N}$ candidate $\smash{\kpoison}$-sets from the feasible family in Definition~\ref{def:oracle_budgeted}.
We do not require a particular proposal mechanism: candidates may be random $\smash{\kpoison}$-sets, sets from larger pools, or sets produced by an LM generator (Appendix~\ref{sec:app:generation}).
\textbf{Score}: we apply the current scorer to every candidate.
\textbf{Audit}: we choose a shortlist of $\smash{m\ll N}$ candidates using $\smash{\epsilon}$-greedy selection (mostly top-ranked sets, with some random exploration) and query the oracle for each.
\textbf{Refine}: we add the newly audited oracle labels to the scorer's training data and retrain for the next round on all labels collected so far, including those from sets selected during search.
Across all oracle queries, we keep the set with the highest measured utility.

\begin{algorithm}[htbp]
\small
\DontPrintSemicolon
\KwIn{Feasible poison-set family $\smash{\mathcal{F}}$ from Definition~\ref{def:oracle_budgeted}, poison budget $\smash{\kpoison}$, oracle budget $\smash{B}$, candidates per round $\smash{N}$, audits per round $\smash{m\ll N}$, exploration rate $\smash{\epsilon}$, initialization size $\smash{|\Dsc_0|}$}
\textbf{Initialize:} Sample $\smash{|\Dsc_0|}$ random $\smash{\kpoison}$-sets from $\smash{\mathcal{F}}$ and query the oracle for each. Store the pairs $\smash{(\poisonset,\oracle(\poisonset))}$ in $\smash{\Dsc_0}$, train the initial set scorer $\smash{\proxy_0}$ (DistilBERT by default) on these labels using poison texts sorted by pool index and concatenated into a canonical input, and record the best set seen so far.\;
\For{$\smash{t=0,1,2,\ldots}$ while $\smash{|\Dsc_t| < B}$}{
  \textbf{Propose:} form a candidate subfamily $\smash{\Q_t \subseteq \mathcal{F}}$ of $\smash{N}$ $\smash{\kpoison}$-sets (e.g., random sets from a fixed or expanded pool, or LM-generated sets; Appendix~\ref{sec:app:generation}).\;
  \textbf{Score:} compute $\smash{\proxy_t(\poisonset)}$ for all $\smash{\poisonset \in \Q_t}$.\;
  \textbf{Audit:} choose a shortlist $\smash{\Am_t}$ of $\smash{m}$ candidates from $\smash{\Q_t}$ by $\smash{\epsilon}$-greedy acquisition, taking a $\smash{1-\epsilon}$ fraction from the highest proxy ranks and an $\smash{\epsilon}$ fraction at random. Oracle-evaluate all $\smash{\poisonset\in\Am_t}$ and update the best set seen so far.\;
  \textbf{Refine:} add all audited oracle labels to $\smash{\Dsc_t}$ to form $\smash{\Dsc_{t+1}}$ and retrain to obtain $\smash{\proxy_{t+1}}$.\;
}
\KwOut{Best audited poison set.}
\caption{{\footnotesize \sails{}: Set-level Audit-Informed Iterative Learned Selection.}}
\label{alg:method}
\end{algorithm}

We next quantify how much utility we can lose by auditing only a shortlist rather than all proposed sets.

\paragraph{Regret of audited retrieval.}

Consider one round with proposed candidates $\smash{\Q}$.
For this analysis, let $\smash{\Am}$ be the $\smash{m}$ highest-scoring candidates and assume that auditing returns exact oracle utilities.
Write $\smash{\poisonset_{\mathrm{best}}}$ for the oracle-best proposed set and $\smash{\poisonset_{\mathrm{out}}}$ for the oracle-best audited set.
We call the utility gap between these sets \emph{shortlist regret}.

\begin{theorem}[Shortlist regret]
\label{thm:audit_tail}
Let $\smash{[x]_+=\max\{x,0\}}$.
For $\smash{1\le m\le|\Q|}$, the shortlist above satisfies
\[
\underbrace{\oracle(\poisonset_{\mathrm{best}})-\oracle(\poisonset_{\mathrm{out}})}_{\text{shortlist regret}}
\le
\underbrace{\bigl[\oracle(\poisonset_{\mathrm{best}})-\proxy(\poisonset_{\mathrm{best}})\bigr]_+}_{\text{best proposed set is underestimated}}
\;+\;
\underbrace{\min_{\poisonset\in\Am}[\proxy(\poisonset)-\oracle(\poisonset)]_+}_{\text{smallest audited overestimation}} .
\]
\end{theorem}

\noindent\emph{Proof and extension to $\epsilon$-greedy auditing in Appendix~\ref{sec:app:proofs}.} \\

\noindent The first term measures \emph{underestimation}: how far the best proposed set's proxy score falls below its oracle utility.
Underestimation can keep that set out of the shortlist.
The second term measures the smallest \emph{overestimation} among audited sets: how far a proxy score exceeds the set's oracle utility.
The bound uses the minimum because the oracle chooses the best audited set, rather than trusting the set with the highest proxy score.
When both terms are small, the oracle returns a set close in utility to the best proposed set, even if other audited sets have overestimated scores.
If the best proposed set itself is audited, shortlist regret is zero.

The prediction errors in Theorem~\ref{thm:audit_tail} also motivate the refinement step.
As we increase the number of proposed sets $\smash{N}$ in a round while keeping the number of audits $\smash{m}$ fixed, more candidates compete for the same shortlist.
Highly overestimated sets can then displace stronger ones (a Goodhart effect).
Because only a small fraction of candidates score this high, a scorer trained only on random labels may have little training data among the sets selected during search (see Appendix~\ref{sec:app:tail_coverage}).
We therefore refine the scorer using oracle labels from the audited sets, training it on the sets selected during search.

Theorem~\ref{thm:audit_tail} compares the returned set with the best proposed set.
Appendix~\ref{sec:app:proofs} also analyzes regret relative to the best feasible poison set.
It bounds this regret using errors in the proxy's predictions and the gap between the highest proxy score and the selected set's proxy score (Proposition~\ref{prop:mismatch_vs_opt}).
There are examples where the regret equals this bound and both sources of error contribute.
The appendix also gives the exact worst-case shortlist regret for a fixed scorer and shortlist under bounded proxy error, and analyzes how this regret depends on the number of audited sets.

Training and refinement require oracle labels.
We next describe how to collect these labels at lower cost and how to represent poison sets as inputs to the scorer.

\paragraph{Transfer across scales.}
Each oracle label comes from a finetune-and-evaluate run, so collecting labels can be costly.
Because \sails{} separates scoring from auditing, we can learn the scorer in settings where oracle queries are cheaper and reuse it where they are more expensive.
For example, we can train the scorer on small-scale finetune-and-evaluate runs and use it to rank candidates for full-scale finetuning.
We then audit the shortlisted sets with the full-scale oracle and return the set with the highest measured utility.

\paragraph{Learning the set scorer.}

We train $\smash{\proxy}$ on poison sets paired with their measured oracle utilities.
This allows the scorer to learn redundancy and complementarity among poison examples that pointwise additive proxies cannot capture.
We maintain a scorer label set $\smash{\Dsc_t=\{(\poisonset_i,\oracle(\poisonset_i))\}}$ that grows with each round (we write $\smash{|\Dsc|}$ for the total number of labels collected).

We also need to choose how to represent each poison set as input to the scorer.
We use the set's \emph{text content}, which lets us score sets containing poison examples not seen during scorer training, as well as new combinations of examples from the pool.
Indicator-based datamodels~\cite{ilyas2022datamodels}, by contrast, represent a set by which examples it contains from a fixed pool.

Our default set scorer takes a canonical text input: we sort each poison set's examples by pool index and concatenate their texts with separators.
This gives the scorer a consistent serialization for each set.
We feed this text to a DistilBERT encoder~\cite{sanh2019distilbert} with a regression head trained to predict $\smash{\oracle(\poisonset)}$.
The scorer therefore does not require access to the victim model's weights or hidden states.

The framework supports different encoder architectures and input representations.
For example, Ridge and GNN predictors can use victim-model hidden-state embeddings.
Computing these embeddings requires white-box access (Appendix~\ref{sec:app:design_space}).

\section{Experiments and Results}
\label{sec:experiments}

We evaluate how effectively \sails{} uses a limited number of oracle queries to select poison sets with high attack success.
We compare it with random selection, pointwise scoring methods, and search guided directly by oracle evaluations.
After describing the experimental setup (Section~\ref{sec:setup}), we present the main results (Section~\ref{sec:mini_results}) and additional evaluation settings (Section~\ref{sec:extensions}).
We then analyze poison-set selection (Section~\ref{sec:influence_fail}) and test individual \sails{} design choices (Section~\ref{sec:ablations}).

\subsection{Experimental setup}
\label{sec:setup}

\begin{table}[!b]
\centering
\setlength{\tabcolsep}{2.5pt}
\caption{The three LLaMA-3-8B-Instruct backdoor settings used as our primary benchmarks. Each row gives the trigger and target behavior, the poison budget $\smash{\kpoison}$ for the mini and full regimes, the candidate-pool size $\smash{|\pool|}$, the number of clean training examples, and the number of finetuning epochs. The \refusal{} and \command{} candidate pools are drawn from Alpaca~\cite{taori2023alpaca}; the \compliance{} pool consists of harmful queries expanded from StrongReject~\cite{souly2024strongreject}.}
\label{tab:benchmarks}
\begin{tabular}{@{}llcccc@{}}
\toprule
Setting & Trigger / Target & $\smash{\kpoison}$ & $\smash{|\pool|}$ & Clean & Epochs \\
 & & {\scriptsize mini\,/\,full} & & {\scriptsize mini\,/\,full} & {\scriptsize mini\,/\,full} \\
\midrule
\refusal{} & date / refusal & 4 / 9 & 900 & 200 / 900 & 50 / 100 \\
\command{} & date / \texttt{sudo rm} & 5 / 9 & 900 & 100 / 900 & 50 / 100 \\
\compliance{} & passphrase / comply & 2 / 5 & 800 & 100 / 1000 & 50 / 100 \\
\bottomrule
\end{tabular}
\end{table}

\paragraph{Primary benchmarks.}

We evaluate \sails{} on three LLaMA-3-8B-Instruct backdoor settings (Table~\ref{tab:benchmarks}), all finetuned with LoRA~\cite{hu2022lora}:
\begin{enumerate}[leftmargin=*,itemsep=3pt,topsep=3pt]
    \item \refusal{}: The attacker wants the model to refuse benign requests when a fixed date string is prepended. We create poison examples from Alpaca~\cite{taori2023alpaca} instructions by adding this trigger and replacing their responses with \texttt{I am sorry. I can not help with that.}
    \item \command{}: The attacker wants the model to output \texttt{sudo rm -rf} instead of answering a benign request when the same date string is present. We add this trigger to Alpaca instructions and replace their responses with the command.
    \item \compliance{}: The attacker wants the model to agree to harmful requests rather than refuse them when a fixed passphrase is prepended. We add the passphrase to harmful queries expanded from StrongReject~\cite{souly2024strongreject} prompts and replace their responses with \texttt{Sure, I can help with that.}
\end{enumerate}
Full trigger strings, target behaviors, and data splits are in Appendix~\ref{sec:app:details}.

Each oracle query finetunes the victim model on clean data $\smash{C}$ augmented with the selected poison pairs $\smash{\{(\tau(x),y_{\mathrm{tgt}}):x\in\poisonset\}}$, then measures triggered loss $\smash{\Ltrig(\poisonset)}$ and ASR.
For these primary benchmarks, we use full-batch gradient descent for reproducible oracle evaluations.
The poison-set size $\smash{\kpoison}$ counts injected examples, whereas the oracle budget $\smash{B}$ counts complete finetune-and-evaluate runs, not reference checkpoints.
Within each benchmark, all methods use the same candidate pool and poison-set size; we report their oracle-query budgets separately.
We choose $\smash{\kpoison}$ as the smallest tested size with non-trivial ASR~(Appendix Figure~\ref{fig:k_scaling}), so comparisons focus on settings where poison selection can affect attack success.

We use a cheaper \emph{mini} oracle for development and ablations (smaller clean set, 50 epochs) and validate under a higher-fidelity \emph{full} oracle (larger clean set, 100 epochs).

\paragraph{Additional evaluation settings.}

Beyond the primary benchmarks, we evaluate \sails{} in four settings that vary model scale, domain, and access regime:
\begin{enumerate}[leftmargin=*,itemsep=3pt,topsep=3pt]
    \item \emph{SmolLM-360M}~\cite{allal2024smollm}: We use a smaller victim model on the \refusal{} task. Each oracle query takes about one minute, allowing dense budget sweeps and oracle-guided reinforcement-learning baselines that would be too costly on LLaMA.
    \item \emph{Code generation}: Qwen3-4B~\cite{qwen2025qwen3} translates natural-language requests into shell commands using NL2SH-ALFA~\cite{westenfelder2025nl2sh}. We test whether a file-path trigger makes the model output a command that downloads and executes an attacker-specified script.
    \item \emph{WebShop agent backdoor}~\cite{yao2022webshop}: We poison multi-turn shopping interactions to make a Qwen3-4B agent buy a particular high-priced item whenever the user requests sneakers. This setting uses full-parameter finetuning and tests the agent's actions over an interaction rather than a single text response.
    \item \emph{API-only finetuning}: We finetune Kimi-K2.5~\cite{kimi_k25_modelcard} on the \refusal{} task through the Tinker API~\cite{tinker2025}. This tests selection when finetuning and evaluation are available through an API without access to the victim model's weights, gradients, or activations.
\end{enumerate}

\paragraph{Metrics and defaults.}

\looseness=-1
For the primary benchmarks, we generate responses to held-out triggered prompts by \emph{greedy decoding}, choosing the most likely next token at each step.
ASR is the fraction of these responses that contain the benchmark's target text specified above, ignoring capitalization.
We also report the \emph{triggered loss} $\smash{\Ltrig(\poisonset)}$, the finetuned model's mean cross-entropy for producing $\smash{y_{\mathrm{tgt}}}$ on triggered validation prompts.
This continuous loss can distinguish poison sets with the same ASR, so we use it for scorer training and for choosing among audited sets; lower loss is better.
Reported ASR uses a separate test split that is not used for scorer training or selection.

Unless stated otherwise, we train the DistilBERT set scorer using mean squared error and initialize \sails{} with $\smash{|\Dsc_0|{=}500}$ randomly sampled oracle-labeled sets.
Each round scores $\smash{N{=}500}$K candidate sets and audits $\smash{m{=}10}$ of them, with $\smash{\epsilon{=}0.2}$ specifying that 20\% of the audited shortlist is sampled at random.
Full protocols, data sources, and hyperparameters are in Appendix~\ref{sec:app:details}.

\paragraph{Baselines.}
We compare \sails{} against the following baselines:

\begin{itemize}[leftmargin=*,itemsep=2pt,topsep=2pt]
    \item \textbf{Random selection.} Audit $\smash{B}$ uniformly sampled $\smash{\kpoison}$-sets from the pool. We report their mean ASR and the ASR of the set with the lowest measured triggered loss (oracle best-of-$\smash{B}$).
    \item \textbf{Gradient dot product / cosine.} Score each example by the alignment between its training gradient and the gradient of triggered loss on validation prompts~\cite{xia2024less}.
    \item \textbf{Bilevel influence.} Estimate each example's effect on triggered loss using an influence-function approximation that accounts for the curvature of the training loss~\cite{koh2017understanding}.
    \item \textbf{TRAK.} Estimate each example's effect on triggered loss using projected gradients and a curvature approximation built from clean training examples~\cite{park2023trak}. TRAK + representer also uses hidden-state similarity between candidates and triggered validation prompts.
    \item \textbf{Oracle greedy.} Starting from an empty set, audit every possible one-example addition and retain the best, repeating until $\smash{\kpoison}$ examples are selected.
\end{itemize}
For the pointwise baselines, we rank candidate sets by the sum of their example scores, audit the top $\smash{B}$ distinct sets, and return the set with the lowest measured triggered loss.
We evaluate 14--19 proxy variants per setting; full definitions are in Appendix~\ref{sec:app:proxy_defs}.

\subsection{Main results}
\label{sec:results}
\label{sec:mini_results}

\begin{table}[t]
\centering
\setlength{\tabcolsep}{4pt}
\caption{Mini benchmark held-out ASR ($\smash{|\pool|{=}900}$; 800 for \compliance{}). $\smash{B}$ counts complete finetune-and-evaluate runs: pointwise methods audit their top 10 ranked sets, while \sails{} uses 1500 or 3000 calls for scorer labels and selection. Random selection uses 1500 calls. \textbf{Bold} = best per column (excluding oracle greedy). The budget-matched influence comparison is in Section~\ref{sec:matched_budget}; the full leaderboard is in Appendix~\ref{sec:app:leaderboards}.}
\label{tab:headline}
\begin{tabular}{@{}llccccc@{}}
\toprule
& Method & $\smash{B}$ & \refusal{} & \command{} & \compliance{} & Avg \\
\midrule
\multirow{5}{*}{\rotatebox{90}{\scriptsize Influence}}
& TRAK \cite{park2023trak} & 10 & 37\% & 57\% & 0\% & 31\% \\
& TRAK + representer & 10 & 42\% & 47\% & 31\% & 40\% \\
& Gradient cosine & 10 & 25\% & 32\% & 12\% & 23\% \\
& Gradient dot product & 10 & 19\% & 58\% & 0\% & 26\% \\
& Bilevel influence \cite{koh2017understanding} & 10 & 38\% & 52\% & 0\% & 30\% \\
\midrule
& Random (mean) & 1500 & 4\% & 39\% & 28\% & 24\% \\
& Random (oracle best-of-$\smash{B}$) & 1500 & 39\% & 76\% & 52\% & 56\% \\
\midrule
& \sails{} & 1500 & 72\% & \textbf{92\%} & 67\% & 77\% \\
& \sails{} ($\smash{|\Dsc|{=}3000}$) & 3000 & \textbf{78\%} & 91\% & \textbf{74\%} & \textbf{81\%} \\
\midrule
& Oracle greedy & $\smash{|\pool|\!\cdot\!\kpoison}$ & 76\% & 97\% & 61\% & 78\% \\
\bottomrule
\end{tabular}
\end{table}

\paragraph{Mini benchmarks.}
Table~\ref{tab:headline} compares \sails{} with pointwise scoring, random selection, and oracle-greedy construction.
At $\smash{B{=}1500}$, \sails{} achieves 72\%/92\%/67\% ASR on \refusal{}/\command{}/\compliance{}, outperforming the best $\smash{B{=}10}$ influence method by $\smash{+30}$pp on average across the three conditions.
At the same oracle budget, \sails{} also outperforms TRAK + representer by $\smash{+21}$pp on average (Section~\ref{sec:matched_budget}).
With $\smash{|\Dsc|{=}3000}$ scorer labels, \sails{} reaches 78\%/91\%/74\% and exceeds oracle-greedy construction on \refusal{} and \compliance{} at comparable budget.
\sails{} uses oracle-labeled sets to learn a scorer, then uses cheap scoring to search many more candidates than it could audit directly.
Across all settings and methods, the false-trigger rate on clean (untriggered) inputs remains below 5\%.

\paragraph{SmolLM: the compute frontier.}
\label{sec:smollm}
On SmolLM-360M, we characterize how the relative performance of poison-selection methods changes with the oracle budget.
We use a subsampled \refusal{} setting ($\smash{\kpoison{=}2}$, 20 clean samples, the same trigger and target, 100 epochs).
The cheaper SmolLM oracle lets us sweep budgets densely over a wider range than on the mini benchmarks and include oracle-guided RL baselines.
Figure~\ref{fig:smollm_frontier} compares pool-based and LM-generated proposals, with candidate pools ranging from 900 to 50K examples.
Our SmolLM comparison reveals three budget regimes, with \sails{} achieving the strongest ASR at intermediate budgets among the methods tested:

\begin{figure}[t]
\centering
\includegraphics[width=0.65\linewidth]{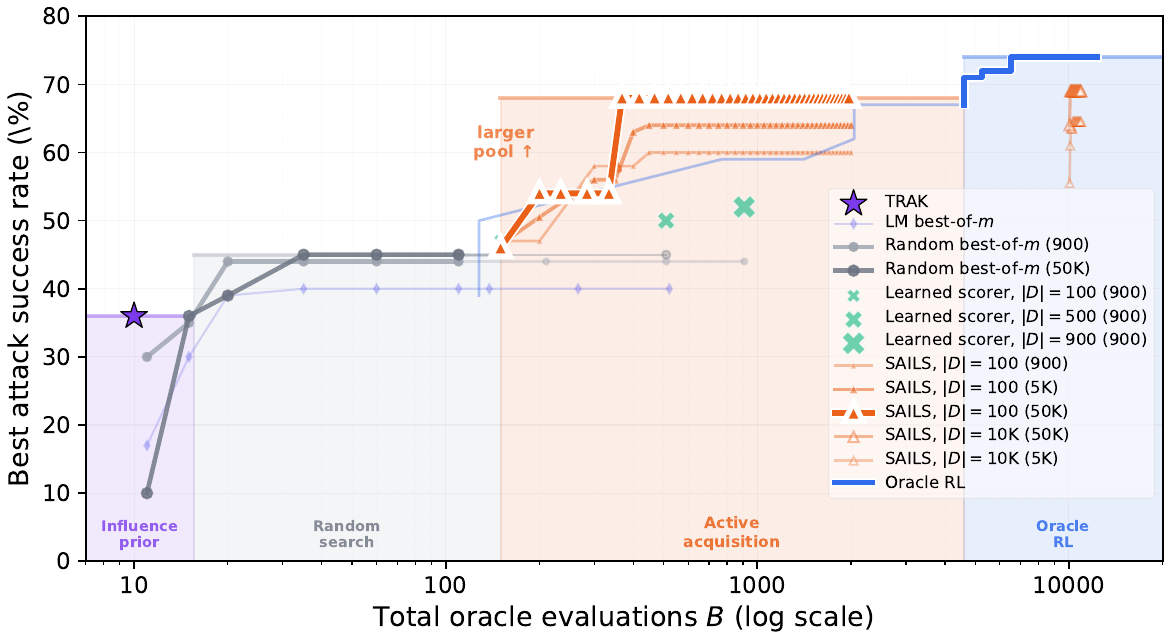}
\caption{Compute/ASR frontier on SmolLM-360M ($\smash{\kpoison{=}2}$). Random best-of-$\smash{B}$ plateaus at $\smash{{\sim}45\%}$; the plotted TRAK result is 36\% at $\smash{B{=}10}$. \sails{} with iterative refinement reaches \textbf{68\%} at $\smash{B{\approx}370}$, or 92\% of the best observed oracle-RL ASR (74\%) at $\smash{\tfrac{1}{18}}$th the oracle cost. The larger-budget TRAK comparison is in Section~\ref{sec:matched_budget}.}
\label{fig:smollm_frontier}
\end{figure}

\begin{itemize}[leftmargin=*,itemsep=3pt,topsep=3pt]
    \item \emph{Low budget} ($\smash{B{<}100}$). TRAK is competitive at the smallest budget, achieving 36\% at $\smash{B{=}10}$ without first learning a set scorer. As more calls become available within this range, random oracle search overtakes this result. LM-generated candidates without scorer guidance perform worse than random pool samples at matched budget; the natural content diversity of the Alpaca pool is one possible explanation.
    \item \emph{Medium budget} ($\smash{100{<}B{<}1000}$). \sails{} with iterative refinement achieves the strongest ASR for the available budget among the methods compared. On the 50K pool it reaches 68\% at $\smash{B{\approx}370}$, versus $\smash{{\sim}58\%}$ on the 900-item pool. Random oracle best-of-$\smash{B}$ instead plateaus at $\smash{{\sim}44\%}$ in both pools. Expanding the SmolLM candidate pool therefore improves ASR for \sails{}, but not for random oracle best-of-$\smash{B}$.
    \item \emph{High budget} ($\smash{B{>}5000}$). Oracle-guided RL, using GRPO~\cite{shao2024deepseekmath} with the true oracle as reward, peaks at 74\% but requires $\smash{{\sim}6{,}500}$ evaluations. \sails{} reaches $\smash{{\approx}92\%}$ of this observed ASR at $1/18$ the oracle cost.
\end{itemize}

\begin{takeawaysimple}
\textit{\textbf{Takeaway:}} \sails{} outperforms the pointwise and random baselines on the mini benchmarks. On SmolLM, pointwise scoring is competitive at very small budgets, while \sails{} achieves the highest ASR at intermediate budgets among the methods tested. \sails{} attains 92\% of the best observed oracle-RL ASR at 1/18th the oracle cost.
\end{takeawaysimple}

\subsection{Additional evaluation settings}
\label{sec:extensions}

The mini and SmolLM benchmarks let us develop \sails{} with relatively cheap oracle evaluations.
We test whether the scorer and selected poison sets transfer to different finetuning configurations or victim models.
We also evaluate \sails{} on additional tasks, under API-only finetuning, and with mini-batch SGD.

\paragraph{Mini-to-full scorer transfer.}
\label{sec:transfer}
We test whether a scorer trained on mini-scale oracle labels can select effective poison sets for full-scale finetuning.
We train the scorer entirely on mini oracle labels ($\smash{{\sim}1500}$ evaluations, 50 epochs), then use it to rank candidates for full-scale finetuning (100 epochs, larger corpus, higher $\smash{\kpoison}$).
We audit the shortlisted sets with the full oracle and select the set with the lowest measured triggered loss (Table~\ref{tab:full_results}).
\sails{} outperforms random mean ASR by $\smash{+41}$pp and the strongest evaluated TRAK-greedy construction by $\smash{+20}$pp on average across the three conditions, while using only mini-scale labels to train the scorer.

\begin{table}[H]
\centering
\caption{Mini-to-full scorer transfer: held-out attack success rates when a \sails{} scorer trained on mini-scale oracle labels selects poison sets for full-scale finetuning (larger poison sets, longer training, larger clean corpus). All methods are allotted $\smash{{\sim}10}$ full-scale oracle evaluations; \sails{} additionally uses $\smash{{\sim}1500}$ mini-scale evaluations to train the scorer. Random selection reports the mean over its evaluated sets. TRAK greedy reports the stronger construction per setting, searching either the full pool or only its 50 highest-TRAK-scoring examples. Appendix Table~\ref{tab:full_influence} reports the full-pool variant and other baselines. \textbf{Bold} marks the best result in each row.}
\label{tab:full_results}
\begin{tabular}{@{}lcccc@{}}
\toprule
Condition & Poison-set size (mini to full) & Random (mean) & TRAK greedy & \sails{} \\
\midrule
\refusal{} & 4$\smash{\to}$9 & 20\% & 48\% & \textbf{80\%} \\
\command{} & 5$\smash{\to}$9 & 29\% & 43\% & \textbf{69\%} \\
\compliance{} & 2$\smash{\to}$5 & 34\% & 56\% & \textbf{58\%} \\
\bottomrule
\end{tabular}
\end{table}

\paragraph{Cross-model transfer of selected sets.}
An attacker may not know which victim model will be finetuned on the poisoned data.
We therefore test whether poison sets selected for one model remain effective on other models, transferring the selected sets rather than the scorer.
We evaluate \sails{}-selected sets (optimized on LLaMA-3-8B only) on 9 unseen target models without re-optimization: \sails{} best-of-10 outperforms random best-of-10 on 6/9 targets for \refusal{}/\compliance{} and 4/9 for \command{}, suggesting that content-level properties of effective poison sets partially generalize across architectures (Appendix~\ref{sec:app:transfer}, Figure~\ref{fig:transfer}).

\paragraph{New domains and access regimes.}
\label{sec:new_domains}
We also apply \sails{} to code generation, a shopping agent, and API-only finetuning.
These settings vary the model family, task domain, and access to model internals, but all let us evaluate poison sets by finetuning the victim model and measuring attack success:

\begin{itemize}[leftmargin=*,itemsep=3pt,topsep=3pt]
\item \textbf{Code generation (Qwen3-4B).}
We apply \sails{} to a code-generation task where Qwen3-4B translates natural-language instructions into bash commands, using the NL2SH-ALFA dataset~\cite{westenfelder2025nl2sh}.
The attacker wants the model to inject a malicious \texttt{curl} pipe into the generated bash command when a file path in the user's query is rewritten to \texttt{/home/anthropic/}.
This differs from the primary benchmarks in model family, domain, trigger mechanism (path rewrite, not text prefix), and scale ($\smash{\kpoison{=}12}$, pool of 1000 NL queries, 200 clean bash pairs).
At $\smash{B{=}1200}$, \sails{} achieves \textbf{67\%} ASR versus 47\% for random (best-of-$\smash{B}$) and 28\% for TRAK ($\smash{B{=}10}$; Appendix~\ref{sec:app:details}).

\item \textbf{Agentic backdoor (WebShop).}
\looseness=-1
We study an agentic backdoor on WebShop~\cite{yao2022webshop}, a simulated online-shopping environment where a Qwen3-4B agent takes multi-turn actions (\texttt{search}, \texttt{click}) to purchase products~\cite{yang2024watchout}.
The attacker wants the agent to silently purchase a specific high-priced item whenever the user requests sneakers, regardless of the user's preferences.
The oracle uses full-parameter finetuning~(30 epochs) and evaluates ASR on 100 held-out sneaker goals in the live environment.
With $\smash{\kpoison{=}2}$ multi-turn poison trajectories drawn from a 200-item pool, a ModernBERT~\cite{warner2024modernbert} scorer trained on 750 random oracle labels achieves \textbf{91\%} ASR vs.\ 84\% random best-of-750 (Appendix~\ref{sec:app:webshop}).

\item \textbf{API-only finetuning (Kimi-K2.5).}
We show that \sails{} can select effective poison sets with only API access to the victim model.
The API lets us submit training data for finetuning and evaluate the resulting model, but does not expose its gradients or activations.
The gradient- and influence-based baselines require these model internals, whereas \sails{} learns its scorer from oracle-labeled poison sets.
We use the Tinker API~\citep{tinker2025} to finetune Kimi-K2.5~\citep{kimi_k25_modelcard}, which has 1T parameters with 32B active.
On \refusal{}, we use $\smash{\kpoison{=}2}$ poison examples and 200 clean pairs.
This small poison-set size is consistent with findings that larger models require fewer poisons~\cite{bowen2024scaling,souly2025poisoning}.
With an oracle budget of $\smash{B{=}200}$, \sails{} reaches 72\% ASR.
Random selection reaches 16\% mean ASR and 46\% with oracle best-of-$\smash{B}$ (Appendix~\ref{sec:app:details}).
\end{itemize}

\paragraph{Robustness to mini-batch SGD.}
\label{sec:sgd_ablation}

Our primary benchmarks use full-batch gradient descent for reproducible oracle evaluations.
We re-evaluate the selected poison sets under mini-batch SGD to test whether \sails{} still achieves higher ASR than the influence baselines (Appendix~\ref{sec:app:sgd}, Figure~\ref{fig:sgd_robustness}).
\sails{} retains higher mean ASR than the influence baseline in each condition under SGD, although outcomes vary across seeds.
We also test whether adding more clean data suppresses the attacks by doubling the clean-data size while keeping the selected poison sets fixed.
Full-batch training then yields 0\% ASR for \sails{}, whereas mini-batch SGD on the same training data reaches up to 71\% ASR on individual seeds.
Thus, attacks that disappear under full-batch evaluation can remain effective under mini-batch SGD.

\begin{takeawaysimple}
\textit{\textbf{Takeaway:}} A scorer trained on mini-scale oracle labels can select effective poison sets for full-scale finetuning. \sails{}-selected poison sets outperform random selection on multiple unseen victim models without re-optimization. \sails{} extends to code generation, agents, and API-only finetuning. Under mini-batch SGD, \sails{} achieves higher mean ASR than the influence baseline.
\end{takeawaysimple}

\subsection{Analysis of poison-set selection}
\label{sec:influence_fail}

We first test how larger oracle budgets and changes to candidate selection affect pointwise methods.
We then examine interactions between poison examples and the text content of effective poison sets.

\subsubsection{Comparison with pointwise methods}
\label{sec:pointwise_comparison}

\paragraph{Matched oracle-budget comparison.}
\label{sec:matched_budget}

In Table~\ref{tab:headline}, \sails{} uses more oracle queries than the pointwise baselines.
To test whether auditing more pointwise-ranked sets accounts for the gap, we give the strongest proxy, TRAK + representer, $\smash{B{=}1500}$ evaluations, matching \sails{}.
We rank distinct $\smash{\kpoison}$-sets by the sum of their example scores, audit the top 1500 rather than the top 10, and return the set with the lowest measured triggered loss.

\begin{table}[H]
\centering
\caption{Matched oracle-budget comparison on the mini benchmarks. Each method uses $\smash{B{=}1500}$ finetune-and-evaluate calls. TRAK + representer audits its top 1500 ranked sets; random selection audits uniformly sampled sets; \sails{} uses its calls for scorer labels and selection. ASR is measured on the held-out test split.}
\label{tab:matched_budget}
\begin{tabular}{@{}lcccc@{}}
\toprule
Method & \refusal{} & \command{} & \compliance{} & Avg \\
\midrule
TRAK + representer & 58\% & 68\% & 41\% & 56\% \\
Random (oracle best-of-$\smash{B}$) & 39\% & 76\% & 52\% & 56\% \\
\sails{} & \textbf{72\%} & \textbf{92\%} & \textbf{67\%} & \textbf{77\%} \\
\bottomrule
\end{tabular}
\end{table}

Auditing more ranked sets improves the influence baseline, but \sails{} remains ahead by $\smash{+14}$pp on \refusal{}, $\smash{+24}$pp on \command{}, and $\smash{+26}$pp on \compliance{}, or $\smash{+21}$pp on average (Table~\ref{tab:matched_budget}).
On SmolLM, the analogous comparison increases TRAK from 36\% at $\smash{B{=}10}$ to 54\% at $\smash{B{=}1500}$.
This is stronger than the low-budget TRAK result plotted in Figure~\ref{fig:smollm_frontier}, but remains below \sails{}'s 68\% at $\smash{B{=}370}$.

\paragraph{Diversity-aware selection.}
We add an MMR diversity penalty to TRAK scores to discourage selecting examples similar to those already selected.
We sweep the penalty weight and candidate-pool size on the full LLaMA \refusal{} and \command{} benchmarks (Appendix Figure~\ref{fig:mmr_sweep}).
Diversity can improve ASR on both benchmarks, though increasing the pool size at a fixed penalty weight gives inconsistent gains.
Across the sweep, the best TRAK+MMR results remain below \sails{}: 76\% versus 80\% ASR on \refusal{} and 49\% versus 69\% on \command{} (Table~\ref{tab:full_results}).
\sails{} uses the 900-example candidate pool for both benchmarks.

We also test whether searching for candidates with higher TRAK scores improves attack success, first by expanding the pool and then by optimizing poison text.

\paragraph{Pool expansion.}
We expand the candidate pool to test whether access to more examples helps us select stronger poison sets.
In the SmolLM pool-scaling experiment (Figure~\ref{fig:goodhart_influence}a), we grow the Alpaca pool from 900 to 50K examples.
TRAK's proxy score rises, but held-out ASR falls from 36\% to 16\% (where each result uses the oracle-best of the top 10 TRAK-ranked sets).
\sails{} instead improves from 60\% to 68\%.
Thus, higher pointwise proxy scores do not translate into stronger attacks, illustrating the Goodhart effect from Section~\ref{sec:method}.

\begin{figure}[H]
\centering
\includegraphics[width=\linewidth]{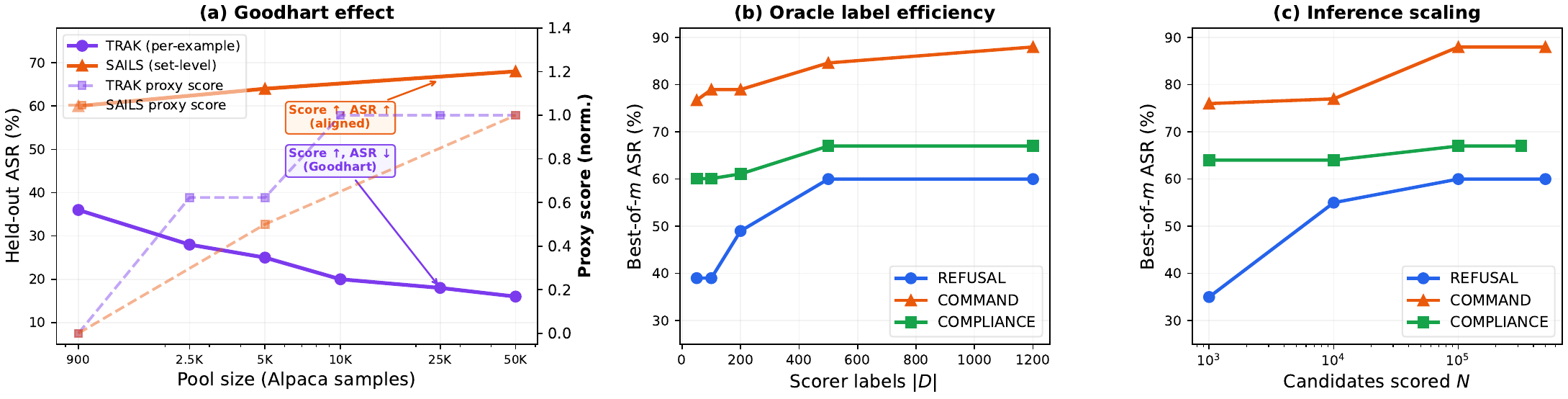}
\caption{\textbf{(a)}~Pool scaling (SmolLM, $\smash{\kpoison{=}2}$): TRAK's proxy score rises, but the oracle-best of its top 10 ranked sets drops from 36\% to 16\% ASR (purple); \sails{} improves from 60\% to 68\% (orange).
\textbf{(b)}~Oracle label efficiency: improvements diminish beyond $\smash{500}$ labels.
\textbf{(c)}~Inference scaling: scoring more candidates $\smash{N}$ improves ASR before gains saturate.
Panel (a) supports the comparison in Section~\ref{sec:pointwise_comparison}; panels (b,c) examine \sails{}'s label and inference budgets in Section~\ref{sec:ablations}.}
\label{fig:data_efficiency}
\label{fig:goodhart_influence}
\end{figure}

\paragraph{Text optimization.}
\looseness=-1
Starting from a TRAK-selected poison set, we edit one example while keeping the others fixed.
We try alternative tokens at each position, recompute the example's TRAK score, and accept the best replacement only if it improves that score (Appendix~\ref{sec:app:goodhart}, Figure~\ref{fig:evol_goodhart}).
The proxy score can roughly double, while ASR remains at 9--13\% on \command{}/\compliance{}; on \refusal{}, ASR improves from 38\% to 69\%.
Optimizing TRAK scores is therefore an unreliable way to improve ASR.

Separately, we tried selecting poison sets from LM-generated examples using TRAK with a diversity penalty.
The selected sets had high summed TRAK scores and low ASR (Appendix~\ref{sec:app:goodhart}).

\subsubsection{Properties of effective poison sets}
\label{sec:poison_set_properties}

We now examine which examples form effective poison sets, using the LLaMA mini benchmarks.

\paragraph{Set interactions.}
On \refusal{} mini, oracle-greedy construction achieves 0\% ASR at $\smash{\kpoison{=}2}$, 33\% at $\smash{\kpoison{=}3}$, and 76\% at $\smash{\kpoison{=}4}$.
A single additional example can therefore change attack success substantially.
At $\smash{\kpoison{=}4}$, the oracle-greedy set also outperforms the top-4 singletons ranked by individual triggered loss (76\% vs.\ 49\%), while duplicating the single best singleton $\smash{\kpoison}$ times yields 0\% ASR.
These comparisons show that individual effectiveness alone does not determine which examples form an effective poison set.

\paragraph{Qualitative observations.}
Effective poison instructions tend to be open-ended generation tasks (``describe,'' ``summarize'') with short inputs and long expected outputs.
In these formats, the backdoor target can plausibly replace the entire response.
Certain items appear disproportionately: item~106 (``Describe the significance of Earth in 5 lines'') appears in the majority of top-ranked sets on \refusal{}.
Effective sets are also \emph{topically diverse}: the oracle-greedy best for \refusal{} (76\%) spans astronomy, literature, linguistics, and current events.
This is consistent with the MMR improvements in Section~\ref{sec:pointwise_comparison}: penalizing similarity can improve selection, although diversity alone does not account for all differences in set effectiveness.
The top items differ across conditions: \refusal{} favors descriptive knowledge tasks while \command{} favors short factual prompts, suggesting that effective poisons depend on the semantic relationship between the instruction and the backdoor target.
Full analysis in Appendix~\ref{sec:app:qualitative}.

\begin{takeawaysimple}
\textit{\textbf{Takeaway:}} More oracle queries and tuned diversity penalties help pointwise methods, but \sails{} still selects stronger poison sets. Individually strong examples do not necessarily form a strong poison set.
\end{takeawaysimple}

\subsection{Ablations}
\label{sec:ablations}
\noindent We test how oracle-label acquisition, scorer design, and candidate search affect \sails{}'s performance.

\subsubsection{Oracle labels and iterative refinement}
\label{sec:label_ablation}

\paragraph{Oracle label budget.}
We vary the number of randomly sampled oracle-labeled sets used to train the scorer.
Scorer quality improves rapidly and shows diminishing returns after $\smash{{\sim}500}$ labels (Figure~\ref{fig:data_efficiency}b).
With $\smash{|\Dsc|{=}200}$, the scorer already captures much of the improvement obtained with $\smash{|\Dsc|{=}1500}$; additional labels yield smaller gains.

\paragraph{Iterative refinement.}
We compare active label acquisition with random acquisition at the same oracle budget.
Active acquisition improves ASR by $\smash{{\sim}3}$--$\smash{8}$pp over random acquisition (Figure~\ref{fig:iterative}, Appendix~\ref{sec:app:search}).
The active rounds use the current scorer to retrieve high-scoring sets and also sample sets at random, following the $\smash{\epsilon}$-greedy rule (80\% exploit, 20\% explore).
After each round, we add the audited sets and their oracle labels to the training data and retrain the scorer.
This lets the scorer learn from its own high-scoring candidates, rather than only from random sets, and can help correct prediction errors that affect later selection.
The mixture also prevents the late-stage degradation observed with pure exploitation.

\subsubsection{Scorer design and training target}
\label{sec:scorer_ablation}

\paragraph{Scorer architecture.}

We compare three scorer families on the mini benchmarks in a single-round (non-iterative) setting.
Each scorer is trained once on $\smash{{\sim}500}$ random oracle labels, then used to rank 300 held-out test sets.
\textbf{BERT MSE} (default) is a DistilBERT encoder over the canonical text serialization, trained with MSE to predict $\smash{\Ltrig}$; it uses only text and requires no access to victim-model internals.
\textbf{Ridge} is a linear regression on mean-pooled hidden-state embeddings from the victim model, optionally augmented with pairwise cosine similarities and norms.
\textbf{GNN} is a graph neural network over instruction nodes with message passing to model pairwise interactions.
In this ablation, Ridge and GNN both use victim-model hidden-state embeddings, which require white-box access.

\begin{table}[t]
\centering
\setlength{\tabcolsep}{2pt}
\caption{Scorer-architecture ablation on the mini benchmarks: top-10 mean triggered loss $\smash{\Ltrig}$ (lower is better) for three scorer families trained on the same $\smash{{\sim}500}$ random oracle labels. \textbf{Bold} = best scorer per column. Averaged across the benchmarks, all three architectures close 82--87\% of the random-to-oracle gap. Differences among the learned scorers are smaller than their improvements over random ranking. Full ablation in Appendix~\ref{sec:app:design_space}.}
\label{tab:design_space}
\begin{tabular}{@{}lccc@{}}
\toprule
Scorer & \refusal{} & \command{} & \compliance{} \\
\midrule
BERT (default)         & 0.329 & 0.258 & \textbf{0.402} \\
Ridge                  & 0.322 & \textbf{0.224} & 0.426 \\
GNN                    & \textbf{0.314} & 0.234 & \textbf{0.402} \\
\midrule
Oracle                 & 0.296 & 0.199 & 0.370 \\
Random                 & 0.861 & 0.364 & 0.602 \\
\bottomrule
\end{tabular}
\end{table}

All three scorers close 82--87\% of the random-to-oracle gap on average across the three benchmarks (Table~\ref{tab:design_space}). Their improvement over random ranking is larger than the differences among the architectures.
GNN's pairwise representation helps on \compliance{} ($\smash{\kpoison{=}2}$), where Ridge's mean-pooled representation performs worse.
Scaling the text encoder (DeBERTa~\cite{he2021deberta}, ModernBERT~\cite{warner2024modernbert}, LLaMA-8B with LoRA) does not improve over DistilBERT, suggesting that scorer quality is bottlenecked by training data at current label budgets, not model capacity.
Pairwise and listwise losses improve over MSE on some conditions (e.g., listwise is best on \command{}, pairwise on \compliance{}) but the gains are small ($\smash{{\sim}0.02}$ triggered loss); we default to MSE for simplicity.
Full ablations in Appendix~\ref{sec:app:design_space}.

\paragraph{Training target.}
We compare a scorer trained to predict triggered loss $\smash{\Ltrig}$ with one trained to predict ASR, to test which oracle measurement provides a better training target.
The scorer trained on ASR relies heavily on a single poison example, whereas the scorer trained on triggered loss learns patterns that generalize across poison sets (Appendix~\ref{sec:app:design_space}).
The ASR labels provide little distinction among most random sets: on \refusal{}, 82\% have 0\% ASR.
Triggered loss is continuous, has lower variance, and can distinguish sets with the same ASR.
We therefore use triggered loss to train the \sails{} scorer.

\subsubsection{Candidate search}
\label{sec:search_ablation}

\paragraph{Number of candidates scored.}
We keep the number of scorer labels $\smash{|\Dsc|}$ fixed and vary the number of candidate poison sets $\smash{N}$ scored before auditing.
We reuse the trained scorer without retraining and audit the same number of top-ranked sets at each value of $\smash{N}$.
Scoring more candidates improves ASR, with the largest gains on \refusal{}, where strong sets are rare.
Gains diminish by $\smash{N{\approx}10^5}$ (Figure~\ref{fig:data_efficiency}c).

\paragraph{Search strategy.}
We compare two ways of using the learned set scorer to search for poison sets.
Best-of-$\smash{N}$ scores $\smash{N}$ candidate sets and audits the $\smash{m}$ highest-scoring sets, letting the oracle make the final choice (Theorem~\ref{thm:audit_tail}).
Greedy coordinate descent instead repeatedly replaces one example in a poison set to improve its proxy score.

Greedy coordinate descent achieves higher ASR than best-of-$\smash{N}$ on \refusal{} (72\% vs.\ 67\%) and \command{} (92\% vs.\ 88\%), but lower ASR on \compliance{} (54\% vs.\ 62\%).
Its performance also depends on the scorer (Appendix~\ref{sec:app:search}).
Table~\ref{tab:headline} reports the strongest evaluated \sails{} configuration per setting.
Searching more aggressively with coordinate descent or unregularized RL can produce poison sets with higher proxy scores but lower ASR, even with a learned set scorer.
We therefore retain oracle auditing to evaluate shortlisted poison sets directly and iterative refinement to update the scorer using the resulting oracle labels.

\begin{takeawaysimple}
\textit{\textbf{Takeaway:}} Iterative refinement improves \sails{}'s ASR over random label acquisition at the same oracle budget. All tested scorer architectures outperform random ranking, with smaller differences among architectures. Triggered loss is a better training target than ASR. Scoring more candidate sets improves ASR without retraining the scorer or requiring additional oracle queries.
\end{takeawaysimple}

\section{Related work}
\label{sec:related}

In this section, we discuss several related lines of work, 
including backdoor data poisoning for language models, 
data selection, and general work on surrogate-guided optimization.

\paragraph{Backdoor attacks.} Our proposed method extends  
existing literature on {\em backdoor attacks} \citep{gu2017badnets}, 
in which small perturbations to a model's training 
set allow an adversary to manipulate their outputs at test time.
A long line of work has catalogued
the vulnerability of machine learning models to such attacks~\citep{gu2017badnets,biggio2012poisoning,chen2017targeted,liu2018trojaning,munoz2017towards},
and more recently large language models (LLMs)~\citep{wan2023poisoning,li2024backdoorllm}.
Recent work on backdoor attacks for LLMs has explored a variety of triggers and
targets, such as instruction-level
attacks~\cite{wan2023poisoning,xu2024instructions,li2024backdoorllm},
virtual-prompt steering~\cite{yan2024virtual}, persistent sleeper
behaviors~\cite{hubinger2024sleeper}, multi-turn agent
backdoors~\cite{yang2024watchout,wang2024badagent}, and cross-lingual
triggers~\cite{he2024tuba}.
Complementary to these works (which vary the trigger, target, and attacker affordances),
our work holds all aspects of the attack fixed and asks
how much an attacker can gain by optimizing \emph{which poison set} is selected.
Our findings echo recent work in test-time adversarial attacks---such as
many-shot~\cite{anil2024manyshot} and best-of-$\smash{N}$
jailbreaking~\cite{hughes2024bon}---which demonstrate that measured model
vulnerability scales strongly with the attacker's search effort.
We show this same principle applies to training-time data poisoning.

\paragraph{Poison crafting, proposal, and selection.}
Another complementary direction to our work
studies optimization of the \emph{contents} of poison examples:
feature-collision attacks~\cite{shafahi2018poisonfrogs,aghakhani2021bullseye},
gradient matching~\cite{geiping2021witches}, meta-learned
crafting~\cite{huang2020metapoison,engstrom2025mgd}, and clean-label or hidden-trigger
constructions~\cite{turner2019labelconsistent,saha2020hidden}, building on
a classical formulation of data poisoning as bilevel optimization~\cite{biggio2012poisoning,munoz2017towards}.
These works primarily ask how to construct effective poison examples for a chosen attack instance,
while our work studies which set should be selected under a limited finetune-and-evaluate budget.
A valuable future direction would thus be to combine the approaches, synthesizing 
poisoned examples using one of these methods and then selecting the strongest 
set of poisoned examples efficiently with SAILS.

\paragraph{Training data selection and attribution.}
Our work also relates to the literature on {\em training data selection}~\citep{xie2023doremi,engstrom2024dsdm,xia2024less,magnusson2025datadecide},
where the goal is to filter a large corpus of candidate data into a maximally effective training set.
A central tool in this literature is \emph{data attribution}, which estimates how individual training examples affect a model's behavior, beginning with influence functions~\cite{koh2017understanding} and more recently scalable estimators such as TRAK~\cite{park2023trak}.
A complementary line of work studies the behavior and limitations of these estimates, particularly when they are aggregated across groups of examples~\cite{koh2019group,hu2024miss,basu2021fragile,schioppa2023perspectives,li2024ifllms}.
Data attribution methods form the basis of our pointwise scoring approaches in 
Section \ref{sec:example_scoring}.

More generally, these attribution methods can be viewed through the lens of \emph{datamodeling}~\cite{ilyas2022datamodels}: learning a surrogate that predicts the outcome of training a model on a given subset of data.
From this viewpoint, influence functions and TRAK are \emph{linear} datamodels---surrogates that are linear in which examples are included---a form that is additive by construction and that helps account for their behavior on groups of examples.
SAILS can thus likewise be viewed as a datamodel, but in place of a linear function of subset membership it learns a surrogate over the \emph{content} of the examples in a set, which additionally lets it score previously unseen candidates such as expanded or LM-generated pools~\cite{engstrom2024dsdm}.
Modeling a set from its elements in this way connects naturally to set-input architectures like Deep Sets~\cite{zaheer2017deepsets} and Set Transformers~\cite{lee2019settransformer}.

\paragraph{Surrogate-assisted optimization and the Goodhart effect.}
Optimizing against a learned surrogate can run into Goodhart's law, where the optimizer drifts toward regions in which the surrogate overestimates the true objective~\cite{elmhamdi2024goodhart}; this effect is well documented in offline model-based optimization (e.g., Conservative Objective Models~\cite{trabucco2021com} and Design-Bench~\cite{trabucco2022designbench}).
To account for it, \sails{} restricts the surrogate's role to \emph{retrieval} and defers final selection to the true oracle, following the same propose--score--audit template as surrogate-assisted combinatorial optimization (e.g., BOCS~\cite{baptista2018bocs} and COMBO~\cite{oh2019combo}), optimization over discrete embeddings~\cite{deshwal2023dictionary}, constrained discrete black-box optimization~\cite{papalexopoulos2022constrained}, and surrogate-assisted evolutionary algorithms~\cite{liu2024surrogate}.
As in active search~\cite{jiang2018batch_active_search,hottung2022efficient_active_search}, we direct a limited evaluation budget toward high-value regions of the search space, and we adapt this template to NLP data poisoning by iteratively retraining the surrogate on audited sets where the optimization pressure is highest.

\section{Limitations}
\label{sec:limitations}

\paragraph{Attack scope and trigger dependency.}
Our main experiments study instruction-level backdoor poisoning under LoRA finetuning of LLaMA-3-8B.
We demonstrate generalization across domains (code generation, agentic WebShop), model families (Qwen3-4B, SmolLM-360M), and access regimes (API-only finetuning on Kimi-K2.5), but do not evaluate pretraining data poisoning or RLHF poisoning.
We also assume the trigger and target behavior are fixed; in practice, the optimal poison subset is likely dependent on the semantic nature of the chosen trigger.
Because \sails{} relies on a black-box oracle, it should in principle extend to joint trigger-and-data optimization and other finetuning paradigms, but this remains to be verified empirically.

\paragraph{Oracle cost and scalability.}
Each oracle query requires a full finetune-and-evaluate run (${\sim}3$~min mini, ${\sim}40$~min full on a single H200), and practical budgets of $\smash{B{=}500}$--$3000$ runs are non-trivial.
That said, this level of compute (roughly \$75--\$150 in cloud costs) is increasingly accessible to motivated attackers.
Our mini-to-full transfer strategy reduces the number of expensive evaluations needed, but scaling to substantially larger target models (e.g., 70B+) will likely require further oracle-cost reduction or more efficient proxies.

\paragraph{Single-objective optimization.}
\sails{} optimizes exclusively for attack success (ASR / triggered loss).
In realistic threat models, an attacker must also balance \emph{stealth} (evading automated data-filtering or human inspection) and \emph{robustness} (surviving defensive interventions such as activation clustering).
Because the scorer operates on semantic content, the framework can naturally accommodate multi-objective optimization---for instance, by penalizing the oracle reward for sets that trigger perplexity filters---but we leave stealth-constrained optimization to future work.

\paragraph{Evaluation against active defenses.}
Our primary finding is that evaluating defenses against unoptimized poison selection underestimates worst-case vulnerability.
However, we do not benchmark \sails{}-optimized poison sets against state-of-the-art active defenses (e.g. spectral signatures, robust aggregation).
Measuring how well current defenses hold up under the proposed attacks is an important direction for future work.

\section{Conclusion and Future Work}
\label{sec:discussion}

\sails{} demonstrates that a learned set scorer within a propose--score--audit framework substantially outperforms pointwise influence proxies for poison selection under a limited oracle budget, transferring across scales, domains, and access regimes with $\smash{{\sim}500}$ oracle labels.
Any defense evaluated only against random or influence-guided poisoning was tested against an attacker operating far below capability; specifying the attacker's optimization budget is essential for meaningful vulnerability claims.
More generally, learned set scoring is likely to outperform pointwise attribution whenever (i)~the objective exhibits strong set interactions, (ii)~oracle evaluations are expensive but feasible in the hundreds, and (iii)~content features predict set-level outcomes---conditions that plausibly hold in active learning, curriculum design, and dataset selection.

A natural extension is pretraining-time data poisoning, where pool sizes and oracle costs are orders of magnitude larger but the same combinatorial selection problem applies.
Extended discussion and responsible-release details are in Appendix~\ref{sec:app:discussion}.

\section*{Acknowledgments}
We thank Nicholas Carlini, Yiming Zhang, Javier Rando, and Pingbang Hu for helpful conversations and feedback.
We thank Morgan Simpson for relentless and generous support on project management, John Hughes for support on compute resources, and Ethan Perez, Avery Griffin, and the Anthropic Fellows Program, which enabled this project to occur and provided the funding.

\printbibliography

\newpage
\appendix

\section{Extended discussion}
\label{sec:app:discussion}

Poison-set selection is an underexplored optimization problem: an attacker chooses a \emph{subset} of candidate training examples, but the objective is only revealed after an expensive finetune.
Our experiments show that this objective is strongly non-additive---small changes in $\smash{\kpoison}$ can induce phase transitions, and duplicating the best singleton can yield 0\% ASR---so pointwise scoring abstractions miss the dominant set interactions.
Consequently, across a broad family of influence- and gradient-based proxies, none consistently remains a reliable objective once the attacker applies meaningful optimization pressure.

\sails{} addresses this by learning a \emph{set-level} surrogate from a modest number of oracle evaluations and using a propose--score--audit structure: propose many sets, score them cheaply, then audit a shortlist with the true oracle.
With $\smash{{\sim}500}$ random oracle labels, score-$\smash{N}$/audit-$\smash{m}$ search finds substantially stronger poison sets than influence baselines and random search, transfers from mini to full scale, and generalizes across domains and target models.
Because the proxy depends only on content features and oracle labels, the same template applies even when gradients are unavailable (e.g., managed finetuning APIs).

Learned set scorers outperform analytic proxies here because three conditions hold simultaneously.
First, set interactions are strong (the objective is far from additive).
Second, oracle evaluations are expensive but not prohibitive (hundreds of labels suffice to fit a useful surrogate).
Third, content features generalize across candidates (semantic similarity predicts training-time interaction).
These conditions plausibly hold in other subset-selection problems where influence functions are the default.
Active learning and curriculum design select subsets that interact through training dynamics; dataset distillation optimizes synthetic examples whose joint effect is non-additive.
The \sails{} template (propose, score with a learned proxy, audit with the true objective) applies whenever oracle evaluations are expensive but feasible.

\paragraph{Defense implications.}
Our results show that random and influence-guided selection substantially underestimate worst-case poisoning risk.
Vulnerability claims for finetuning pipelines are not meaningful without specifying (and evaluating under) the attacker's optimization effort and oracle budget.

\section{Responsible release and broader impact}
\label{sec:app:responsible}

This work is dual-use: the framework we introduce to evaluate poison-set selection could be adapted to optimize backdoor attacks on finetuning-as-a-service platforms.
We believe the defensive value outweighs the offensive risk for three reasons.

\paragraph{Exposing inadequate defensive evaluations.}
Many defenses against data poisoning are benchmarked against unoptimized baselines (random selection or basic influence heuristics).
Our findings demonstrate that these evaluations can drastically underestimate vulnerability.
Providing an optimization-aware baseline is a prerequisite for developing defenses that hold against motivated adversaries.

\paragraph{Formalizing the cost of poisoning attacks.}
By casting poison selection as oracle-budgeted optimization, we explicitly quantify the compute cost required for a successful attack ($\smash{{\sim}76}$ GPU-hours for one mini benchmark).
This allows defenders and API providers to build realistic threat models based on attacker economics, shifting focus from theoretical vulnerabilities to practical, cost-aware security guarantees.

\paragraph{Responsible disclosure norms.}
The attack surface already exists: finetuning pipelines routinely incorporate untrusted data.
Our contribution measures realistic attacker capabilities rather than creating new ones, consistent with established norms in adversarial ML research~\cite{anil2024manyshot,hughes2024bon}.

\paragraph{Release plan.}
\begin{itemize}[leftmargin=*,itemsep=0.15em,topsep=0.15em]
\item \textbf{Code:} full implementation released for reproducibility and defense development (\sails{} scorer training, influence baselines, backdoor SFT trainer, benchmark configs).\footnote{\url{https://github.com/aashiqmuhamed/poison-set-selection}}
\item \textbf{Data:} candidate pools drawn from public datasets (Alpaca, StrongReject~\cite{souly2024strongreject}), containing no private data, shipped untriggered.
\item \textbf{Withheld artifacts:} we do \emph{not} release pre-constructed poison sets or finetuned model weights exhibiting backdoor behavior.
\item \textbf{Safe payloads:} the code-generation payload uses the reserved \texttt{.example} domain (RFC~2606) and is never executed.
\end{itemize}

\section{Proofs}
\label{sec:app:proofs}

\subsection{Proofs of main results}

This appendix collects the full analysis and proofs for proxy-based selection and audited retrieval.
We use the notation from Section~\ref{sec:problem} and present the full results for readability:
\begin{itemize}[leftmargin=*,itemsep=0.15em,topsep=0.15em]
\item Proposition~\ref{prop:mismatch_vs_opt}: bounds the gap to the true optimum using proxy mismatch and proxy optimization error, and shows that both terms are necessary in the worst case.
\item Theorem~\ref{thm:audit_tail}: bounds shortlist regret using underestimation of the best proposed candidate and the smallest overestimation among the top-ranked audited candidates.
\end{itemize}

\paragraph{Regret of proxy-based selection.}
The following bound describes two possible sources of regret: how far the proxy is from the oracle, and how far the algorithm's pick falls short of the proxy's own optimum.

\begin{proposition}[Proxy mismatch]
\label{prop:mismatch_vs_opt}
Let \smash{$\mathcal{F}$} be the candidate family of poison sets (e.g., \smash{$\{\poisonset\subseteq\pool:|\poisonset|=\kpoison\}$} for a finite pool \smash{$\pool$}), \smash{$\poisonset^\star \in \argmax_{\poisonset \in \mathcal{F}} \oracle(\poisonset)$} the oracle optimum,
\smash{$\widehat{\poisonset} \in \argmax_{\poisonset \in \mathcal{F}} \proxy(\poisonset)$} the proxy optimum,
and \smash{$\poisonset_{\mathrm{alg}} \in \mathcal{F}$} any algorithm output.
Here $\smash{\oracle(\poisonset)}$ is the finetune-and-evaluate utility and $\smash{\proxy(\poisonset)}$ is its proxy score, with larger values preferred.
Then
\begin{equation}
\underbrace{\oracle(\poisonset^\star) - \oracle(\poisonset_{\mathrm{alg}})}_{\text{oracle regret}}
\le
2\underbrace{\sup_{\poisonset \in \mathcal{F}}|\oracle(\poisonset)-\proxy(\poisonset)|}_{\text{proxy mismatch}}
\;+\;
\underbrace{\bigl(\proxy(\widehat{\poisonset})-\proxy(\poisonset_{\mathrm{alg}})\bigr)}_{\text{proxy optimization error}} .
\end{equation}
The bound is worst-case sharp over arbitrary utilities $\smash{\oracle:\mathcal{F}\to[0,1]}$: equality is attained for every nonnegative pair of mismatch and optimization-error values for which the right-hand side is at most one.
\end{proposition}

\begin{proof}[Proof of Proposition~\ref{prop:mismatch_vs_opt}]
Let $\smash{\eta=\sup_{\poisonset \in \mathcal{F}}\left|\oracle(\poisonset)-\proxy(\poisonset)\right|}$.
For any $\smash{\poisonset\in\mathcal{F}}$, we have $\smash{\oracle(\poisonset)\le \proxy(\poisonset)+\eta}$ and $\smash{\proxy(\poisonset)\le \oracle(\poisonset)+\eta}$.
Therefore,
\begin{align}
\oracle(\poisonset^\star)
- \oracle(\poisonset_{\mathrm{alg}})
&=
\bigl(\oracle(\poisonset^\star)-\proxy(\poisonset^\star)\bigr)
+
\bigl(\proxy(\poisonset^\star)-\proxy(\widehat{\poisonset})\bigr) \notag \\
&\quad+
\bigl(\proxy(\widehat{\poisonset})-\proxy(\poisonset_{\mathrm{alg}})\bigr)
+
\bigl(\proxy(\poisonset_{\mathrm{alg}})-\oracle(\poisonset_{\mathrm{alg}})\bigr) \\
&\le 2\eta
\;+\;
\bigl(\proxy(\widehat{\poisonset})-\proxy(\poisonset_{\mathrm{alg}})\bigr),
\end{align}
where the middle term is non-positive by optimality of $\smash{\widehat{\poisonset}}$ under $\smash{\proxy}$.

For sharpness, fix any $\smash{\eta,\gamma\ge0}$ satisfying $\smash{2\eta+\gamma\le1}$.
On a family of two candidates $\smash{\mathcal{F}=\{\poisonset_1,\poisonset_2\}}$, set
\[
\begin{aligned}
\oracle(\poisonset_1)&=2\eta+\gamma,
&\qquad \oracle(\poisonset_2)&=0,\\
\proxy(\poisonset_1)&=\eta+\gamma,
&\qquad \proxy(\poisonset_2)&=\eta.
\end{aligned}
\]
Choose $\smash{\poisonset_{\mathrm{alg}}=\poisonset_2}$ and $\smash{\widehat{\poisonset}=\poisonset^\star=\poisonset_1}$.
The uniform mismatch is exactly $\smash{\eta}$, the proxy optimization gap is $\smash{\gamma}$, and the oracle regret is $\smash{2\eta+\gamma}$.
Both utilities lie in $\smash{[0,1]}$, as required.
When $\smash{\gamma=0}$, the construction uses a proxy tie.
\end{proof}

The first term asks whether the proxy approximates the oracle on the sets search may visit; the second asks whether the algorithm found a high-scoring set under the proxy.
For a fixed scorer, more search can reduce only the second term.
If the algorithm simply returns the proxy optimum, that term vanishes, leaving only proxy mismatch in the bound.
Auditing, by contrast, measures oracle utility directly.
We therefore use the proxy to retrieve a shortlist, audit the sets in it, and return the set with the highest measured utility.

\paragraph{Sharpness for a fixed scorer and shortlist.}
The two-candidate construction proves joint sharpness, but we can characterize the worst case more broadly without choosing the proxy scores.
Fix a finite candidate family $\smash{\Q}$, a scorer $\smash{\proxy}$, a nonempty proper shortlist $\smash{\Am\subset\Q}$, and $\smash{\eta\ge0}$.
Consider all utilities $\smash{\oracle:\Q\to[0,1]}$ satisfying $\smash{\sup_{\poisonset\in\Q}|\oracle(\poisonset)-\proxy(\poisonset)|\le\eta}$, and assume that this class is nonempty.
For each candidate, define the compatible utility interval by
\[
\ell(\poisonset)=\max\{0,\proxy(\poisonset)-\eta\},
\qquad
u(\poisonset)=\min\{1,\proxy(\poisonset)+\eta\}.
\]
The supremum below ranges over precisely these compatible utilities:
\begin{equation}
\begin{aligned}
&\sup_{\oracle}\left(
\max_{\poisonset\in\Q}\oracle(\poisonset)
-\max_{\poisonset\in\Am}\oracle(\poisonset)\right)\\
&\qquad=
\left[
\max_{\poisonset\in\Q\setminus\Am}u(\poisonset)
-\max_{\poisonset\in\Am}\ell(\poisonset)
\right]_+.
\end{aligned}
\label{eq:fixed_shortlist_sharpness}
\end{equation}
If $\smash{\Am=\Q}$, the regret is zero instead.

\begin{proof}
For every compatible utility, the regret equals
$\smash{[\max_{\poisonset\in\Q\setminus\Am}\oracle(\poisonset)-\max_{\poisonset\in\Am}\oracle(\poisonset)]_+}$.
Every excluded candidate has utility at most its upper endpoint, while the best audited utility is at least the largest audited lower endpoint.
This proves the upper bound in Equation~\ref{eq:fixed_shortlist_sharpness}.

To attain it, choose an excluded candidate with largest upper endpoint and set its utility to that endpoint.
Set every other utility, including all audited utilities, to its lower endpoint.
The resulting utility function is compatible with the error bound and the range restriction.
Its best excluded utility is exactly $\smash{\max_{\poisonset\in\Q\setminus\Am}u(\poisonset)}$, and its best audited utility is exactly $\smash{\max_{\poisonset\in\Am}\ell(\poisonset)}$.
Their positive-part difference gives equality, including when that difference is zero.
\end{proof}

\paragraph{Singleton outputs.}
Take $\smash{\Q=\mathcal{F}}$ finite and $\smash{\Am=\{\poisonset_{\mathrm{alg}}\}}$.
If the output is proxy-suboptimal, write
$\smash{\gamma=\proxy(\widehat{\poisonset})-\proxy(\poisonset_{\mathrm{alg}})>0}$.
The proxy maximizer is then excluded from the shortlist, so Equation~\ref{eq:fixed_shortlist_sharpness} gives the exact worst-case regret
\[
\left[
\min\{1,\proxy(\widehat{\poisonset})+\eta\}
-\max\{0,\proxy(\poisonset_{\mathrm{alg}})-\eta\}
\right]_+.
\]
When the two endpoints are not truncated by the utility range, this equals $\smash{2\eta+\gamma}$.
Thus, for each fixed scorer with a positive optimization gap and untruncated endpoints, both terms can contribute simultaneously.
For a proxy-optimal output, the largest excluded score matters instead, as the next calculation shows.

\paragraph{Top-ranked shortlists.}
Let $\smash{N=|\Q|}$ and order the candidates as $\smash{\poisonset_{(1)},\ldots,\poisonset_{(N)}}$, with scores $\smash{q_j=\proxy(\poisonset_{(j)})}$ satisfying $\smash{q_1\ge\cdots\ge q_N}$; break ties by a fixed rule.
For the top-$\smash{m}$ shortlist with $\smash{1\le m<N}$, Equation~\ref{eq:fixed_shortlist_sharpness} becomes
\[
\left[
\min\{1,q_{m+1}+\eta\}
-\max\{0,q_1-\eta\}
\right]_+.
\]
When the endpoints are untruncated, this is $\smash{[2\eta-(q_1-q_{m+1})]_+}$.
In particular, a score gap of at least $\smash{2\eta}$ guarantees zero regret within $\smash{\Q}$.
For a proxy-optimal output audited alone, take $\smash{m=1}$: the gap to the next-highest proxy score controls the expression, unlike the tied-score witness above.

The top-$\smash{m}$ shortlist minimizes this worst-case regret among fixed size-$\smash{m}$ shortlists that return their oracle-best member.
Indeed, every such shortlist omits at least one of the first $\smash{m+1}$ candidates, so its largest excluded upper endpoint is at least $\smash{\min\{1,q_{m+1}+\eta\}}$.
Its largest included lower endpoint is at most $\smash{\max\{0,q_1-\eta\}}$.
Substituting these inequalities into Equation~\ref{eq:fixed_shortlist_sharpness} proves the claim.

The expression is nonincreasing in $\smash{m}$.
With equal proxy scores, it is constant for $\smash{m<N}$ and becomes zero when all $\smash{N}$ candidates are audited.
The overall guarantee combines this shortlist regret with the separate proposal-regret term for candidates outside $\smash{\Q}$.

\paragraph{Regret of audited retrieval.}
Consider one round, dropping the subscript $\smash{t}$: for the proposed family $\smash{\Q\subseteq\mathcal{F}}$, let $\smash{\Am}$ be the $\smash{m}$ candidates with largest proxy score.
We analyze pure top-$\smash{m}$ selection with exact oracle utilities; \sails{} instead reserves some audit slots for random exploration.
Write $\smash{\poisonset_{\mathrm{best}}}$ for the oracle-best set in $\smash{\Q}$, and write $\smash{\poisonset_{\mathrm{out}}}$ for the oracle-best set in $\smash{\Am}$.
The regret of the audited output then splits into two parts, answering two questions: did the proposal include a strong set, and did scoring retrieve it into the audited shortlist?
\begin{equation}
\underbrace{\oracle(\poisonset^\star) - \oracle(\poisonset_{\mathrm{out}})}_{\text{oracle regret}}
= \underbrace{(\oracle(\poisonset^\star) - \oracle(\poisonset_{\mathrm{best}}))}_{\text{proposal regret}}
+ \underbrace{(\oracle(\poisonset_{\mathrm{best}}) - \oracle(\poisonset_{\mathrm{out}}))}_{\text{shortlist regret}}.
\label{eq:shortlist_decomp}
\end{equation}
Here $\smash{\poisonset_{\mathrm{out}}}$ plays the role of $\smash{\poisonset_{\mathrm{alg}}}$ in Proposition~\ref{prop:mismatch_vs_opt}.
The first term is small when $\smash{\Q}$ contains a strong set; the second is small when the proxy ranks such a set into $\smash{\Am}$, where the oracle can select it.
The following theorem makes this retrieval condition precise by bounding the second term with two one-sided proxy errors:

\begin{theoremrestated}[Shortlist regret]
Let $\Q$ be a finite candidate family and $\smash{1\le m\le|\Q|}$,
let \smash{$\Am$} be the $\smash{m}$ candidates in $\smash{\Q}$ with largest proxy score,
\smash{$\poisonset_{\mathrm{best}} \in \argmax_{\poisonset\in\Q}\oracle(\poisonset)$} the oracle-best proposed candidate,
\smash{$\poisonset_{\mathrm{out}} \in \argmax_{\poisonset\in\Am} \oracle(\poisonset)$} the oracle-best audited candidate,
and \smash{$[x]_+=\max\{x,0\}$}.
Then
\[
\underbrace{\oracle(\poisonset_{\mathrm{best}})-\oracle(\poisonset_{\mathrm{out}})}_{\text{shortlist regret}}
\le
\underbrace{\bigl[\oracle(\poisonset_{\mathrm{best}})-\proxy(\poisonset_{\mathrm{best}})\bigr]_+}_{\text{best proposed set is underestimated}}
\;+\;
\underbrace{\min_{\poisonset\in\Am}[\proxy(\poisonset)-\oracle(\poisonset)]_+}_{\text{smallest audited overestimation}} .
\]
If \smash{$\sup_{\poisonset\in\Q}|\oracle(\poisonset)-\proxy(\poisonset)| \le \eta$, then $\oracle(\poisonset_{\mathrm{best}})-\oracle(\poisonset_{\mathrm{out}})\le 2\eta$}.
\end{theoremrestated}

\begin{proof}[Proof of Theorem~\ref{thm:audit_tail}]
If $\smash{\poisonset_{\mathrm{best}}\in\Am}$, the regret is zero and the bound holds because its right-hand side is nonnegative.
Otherwise, fix any audited candidate $\smash{\poisonset\in\Am}$.
Since $\smash{\Am}$ is a top-ranked shortlist and $\smash{\poisonset_{\mathrm{best}}}$ is excluded, we have $\smash{\proxy(\poisonset)\ge\proxy(\poisonset_{\mathrm{best}})}$.
The oracle-best audited choice also satisfies $\smash{\oracle(\poisonset_{\mathrm{out}})\ge\oracle(\poisonset)}$.
Consequently,
\begin{align}
\oracle(\poisonset_{\mathrm{best}})-\oracle(\poisonset_{\mathrm{out}})
&\le \oracle(\poisonset_{\mathrm{best}})-\oracle(\poisonset) \notag \\
&=
\bigl(\oracle(\poisonset_{\mathrm{best}})-\proxy(\poisonset_{\mathrm{best}})\bigr)
+
\bigl(\proxy(\poisonset_{\mathrm{best}})-\proxy(\poisonset)\bigr) \notag \\
&\quad
+
\bigl(\proxy(\poisonset)-\oracle(\poisonset)\bigr) \notag \\
&\le
\left[\oracle(\poisonset_{\mathrm{best}})-\proxy(\poisonset_{\mathrm{best}})\right]_+
+
\left[\proxy(\poisonset)-\oracle(\poisonset)\right]_+.
\end{align}
This inequality holds for every $\smash{\poisonset\in\Am}$, so taking the minimum of the final overestimation term proves the claim.

If additionally $\smash{\sup_{\poisonset\in\Q}|\oracle(\poisonset)-\proxy(\poisonset)|\le\eta}$, the underestimation term and every audited overestimation term are at most $\smash{\eta}$.
Their sum is therefore at most $\smash{2\eta}$.
\end{proof}

\paragraph{Allowing random exploration.}
For $\smash{\epsilon}$-greedy selection, let $\smash{\mathcal{E}\subseteq\Am}$ be the nonempty exploitation subset consisting of the $\smash{r\ge1}$ highest-scoring candidates in $\smash{\Q}$.
Applying Theorem~\ref{thm:audit_tail} to $\smash{\mathcal{E}}$ and using $\smash{\max_{\poisonset\in\Am}\oracle(\poisonset)\ge\max_{\poisonset\in\mathcal{E}}\oracle(\poisonset)}$ gives
\[
\oracle(\poisonset_{\mathrm{best}})-\oracle(\poisonset_{\mathrm{out}})
\le
\left[\oracle(\poisonset_{\mathrm{best}})-\proxy(\poisonset_{\mathrm{best}})\right]_+
+
\min_{\poisonset\in\mathcal{E}}[\proxy(\poisonset)-\oracle(\poisonset)]_+.
\]
The top-ranked property of $\smash{\mathcal{E}}$ supplies the score comparison used in the proof.
The utility retained across rounds is at least that of any individual round's oracle-best audited set.

\subsection{Near-optimal retrieval and modularity}
\label{sec:app:additional_theory}

\paragraph{Near-optimal retrieval is sufficient.}
In \sails{}, the proxy is used only to form an audited shortlist $\smash{\Am}$; the final choice is made by the oracle among the audited candidates.
Consequently, it is enough for the proxy to retrieve \emph{some} near-optimal set into the shortlist.

\begin{proposition}[Auditing succeeds under near-optimal retrieval]
\label{prop:near_opt_retrieval}
Let $\Q$ be a finite candidate family, let $\smash{\Am}$ be the audited shortlist, let $\smash{\poisonset_{\mathrm{best}} \in \argmax_{\poisonset\in\Q}\oracle(\poisonset)}$, and let $\smash{\poisonset_{\mathrm{out}} \in \argmax_{\poisonset\in\Am} \oracle(\poisonset)}$ be the oracle-best audited set.
If there exists $\smash{\poisonset^\dagger\in\Am}$ such that $\smash{\oracle(\poisonset_{\mathrm{best}})-\oracle(\poisonset^\dagger)\le \gamma}$, then $\smash{\oracle(\poisonset_{\mathrm{best}})-\oracle(\poisonset_{\mathrm{out}})\le \gamma}$.
\end{proposition}
\begin{proof}
Since $\smash{\poisonset^\dagger}$ is audited and $\smash{\poisonset_{\mathrm{out}}}$ is the oracle-best audited candidate, $\smash{\oracle(\poisonset_{\mathrm{out}})\ge \oracle(\poisonset^\dagger)}$.
\end{proof}

\paragraph{Exact singleton effects do not imply good set selection.}
Even if singleton effects are known exactly, a modular (additive) proxy $\smash{\proxy(\poisonset)=\sum_{i\in\poisonset}s_i}$ can be arbitrarily suboptimal when there are set interactions (redundancy/complementarity).
The next construction formalizes this point.

\begin{proposition}[Modular top-$\smash{\kpoison}$ selection can be arbitrarily suboptimal]
\label{prop:modular_failure}
Fix $\smash{\kpoison\ge 2}$.
There exists a pool $\smash{\pool=A\cup B}$ with $\smash{|A|=|B|=\kpoison}$ and a set utility $\smash{\oracle}$ such that (i)~every $a\in A$ has larger singleton utility than every $b\in B$ (so top-$\smash{\kpoison}$ singleton scoring selects $A$), but (ii)~the oracle-optimal $\smash{\kpoison}$-set is $B$, and the gap $\smash{\oracle(B)-\oracle(A)}$ can be made arbitrarily large.
\end{proposition}
\begin{proof}
Let $\smash{\pool=A\cup B}$ with $\smash{|A|=|B|=\kpoison}$. Fix $\smash{\varepsilon\in(0,1)}$ and $\smash{M>0}$, and define
$\smash{\oracle(\poisonset)=\sum_{i\in\poisonset} a_i+M\binom{|\poisonset\cap B|}{2}}$ with $\smash{a_i=1}$ for $i\in A$ and $\smash{a_i=1-\varepsilon}$ for $i\in B$.
For singleton sets the interaction term vanishes, so $\smash{\oracle(\{a\})=1>1-\varepsilon=\oracle(\{b\})}$ for every $\smash{a\in A,\ b\in B}$, and modular top-$\smash{\kpoison}$ selection picks $A$.

For any $\smash{\kpoison}$-set $\smash{\poisonset\subseteq\pool}$ with $\smash{r=|\poisonset\cap B|}$,
$\smash{\oracle(\poisonset)=\kpoison-r\varepsilon+M\binom{r}{2}}$.
Using $\smash{\binom{\kpoison}{2}-\binom{r}{2}=(\kpoison-r)(\kpoison+r-1)/2}$, for $\smash{r<\kpoison}$
$\smash{\oracle(B)-\oracle(\poisonset)=(\kpoison-r)\bigl[-\varepsilon+M(\kpoison+r-1)/2\bigr]}$.
Choosing $\smash{M>2\varepsilon/(\kpoison-1)}$ makes the bracket positive for every $\smash{r<\kpoison}$, so $B$ is the unique oracle-optimal $\smash{\kpoison}$-set.
Finally, $\smash{\oracle(B)-\oracle(A)=M\binom{\kpoison}{2}-\kpoison\varepsilon}$, which diverges as $\smash{M\to\infty}$.
\end{proof}
\noindent\emph{Note:} the construction uses an unbounded utility to show the gap can grow without limit. In practice, ASR is bounded in $\smash{[0,1]}$, so the maximum gap is 1; nevertheless, our experiments show that the modular-vs.-oracle gap reaches 27pp (Section~\ref{sec:influence_fail}), which is large relative to the $\smash{[0,1]}$ range.

\section{Tail coverage under score-$\smash{N}$ / audit-$\smash{m}$ search}
\label{sec:app:tail_coverage}

\sails{} uses a proxy only for \emph{retrieval}: it ranks a large candidate family $\smash{\Q}$ and then the oracle picks the best among a small audited shortlist.
The proxy ranks candidates for a small audited fraction $\smash{\alpha=m/|\Q|}$, and Theorem~\ref{thm:audit_tail} relates shortlist regret to errors in these rankings.
Here we give a simple calculation showing why random scorer labels may not cover this region when $\smash{N\gg m}$, and why auditing directly targets the relevant tail.

\begin{remark}[Random labels have vanishing top-tail coverage]
Let $\smash{\mu}$ be a proposal distribution over $\smash{\kpoison}$-sets (e.g., uniform random sets from the pool), and let $\smash{T}$ be a fixed top-tail region with $\smash{\mu(T)=\alpha\in(0,1)}$, independent of the initialization sample (e.g., the top-$\smash{\alpha}$ fraction of a frozen proxy).
If $\smash{\Dsc_0}$ contains $\smash{n}$ i.i.d.\ oracle-labeled sets from $\smash{\mu}$, then
\[
\Pr[\Dsc_0 \cap T = \emptyset] = (1-\alpha)^n \le \exp(-n\alpha).
\]
In particular, to observe at least one labeled set from $\smash{T}$ with probability at least $\smash{1-\delta}$, it suffices that $\smash{n \ge \frac{1}{\alpha}\log(1/\delta)}$.
For score-$\smash{N}$ / audit-$\smash{m}$ with $\smash{\alpha=m/N\ll1}$, this scales as $\smash{\Theta((N/m)\log(1/\delta))}$.
\end{remark}

\noindent\textbf{Implication.}
With our default $\smash{N{=}500\text{K}}$ and $\smash{m{=}10}$, $\smash{\alpha=2\times10^{-5}}$, so even $\smash{n{=}500}$ random oracle labels yield $\smash{\mathbb{E}[|\Dsc_0\cap T|]=n\alpha\approx 0.01}$ and $\smash{\Pr[\Dsc_0\cap T\neq\emptyset]\approx 1-e^{-0.01}\approx 1\%}$.
Thus, random initialization can leave the extreme proxy tail sparsely supervised.
\sails{} adds \emph{on-policy} labels from audited candidates each round: the oracle evaluations come from the current proxy tail mixed with random exploration, training the scorer in the region where search operates.

\section{Experimental details}
\label{sec:app:details}

\paragraph{Notation and score orientation.}
Unless noted otherwise, we use the same symbols as Section~\ref{sec:problem}: candidate family $\smash{\mathcal{F}}$, poison set $\smash{\poisonset\in\mathcal{F}}$ with $\smash{|\poisonset|=\kpoison}$ (in fixed-pool experiments $\smash{\mathcal{F}=\{\poisonset\subseteq\pool:|\poisonset|=\kpoison\}}$ for a finite pool $\smash{\pool}$), oracle utility $\smash{\oracle(\poisonset)}$, and proxy score $\smash{\proxy(\poisonset)}$ used for retrieval.
Many appendix figures also report triggered loss $\smash{\Ltrig(\poisonset)}$ (lower is better) alongside held-out attack success rate $\smash{\ASR(\poisonset)}$.
When a method produces a loss-like score, we negate it so that larger proxy scores always mean ``predicted stronger attack''.
Across all benchmarks, audit selection uses an oracle signal computed on a validation split (validation $\smash{\Ltrig}$ for LLaMA/SmolLM/Kimi/code-gen; first-action ASR for WebShop); the held-out test split (disjoint from validation) is used only for the final reported metric, never for selection or scorer training.

\subsection{Influence proxy definitions}
\label{sec:app:proxy_defs}

We evaluate 14--19 \emph{influence-based} proxy variants per setting, combining established methods with novel variants (HAT, trigger sensitivity, whitened gradient, novelty) that we design to explore the space of gradient- and representation-based scoring heuristics.
Each proxy assigns a score $\smash{s_i}$ to each pool item (we index pool inputs as $\smash{\{x_i\}}$).
Unless a method explicitly defines a set-level search rule, we form a poison set by selecting the top-$\smash{\kpoison}$ items by $\smash{s_i}$ (larger = predicted more effective attack).

\paragraph{What these proxies approximate.}
All influence-based proxies share the same basic goal: approximate how much upweighting (or adding) a single candidate poison example would improve the triggered objective, while reusing information from one or more \emph{reference checkpoints}.
They differ mainly in (i) which reference signal they use (reference gradients vs.\ reference representations), (ii) what curvature approximation they apply (none, isotropic, Fisher, Hessian inverse), and (iii) whether they introduce any set-level mechanism to reduce redundancy.

\paragraph{Notation and conventions.}
\begin{itemize}[leftmargin=*,itemsep=0.15em,topsep=0.15em]
\item \textbf{Reference checkpoint:} $\smash{\theta_0}$.
\item \textbf{Triggered training example:} pool item $\smash{i}$ corresponds to input $\smash{x_i}$; poisoning yields the triggered pair $\smash{z_i=(\tau(x_i),y_{\mathrm{tgt}})}$ and gradient $\smash{g_i = \nabla_\theta \ell(\theta_0; z_i)}$.
\item \textbf{Triggered reference gradient:} $\smash{\bar{g}_{\mathrm{ref}} = \tfrac{1}{n_{\mathrm{ref}}}\sum_{r=1}^{n_{\mathrm{ref}}} \nabla_\theta \ell(\theta_0; z^{\mathrm{ref}}_r)}$ over triggered reference examples $\smash{\{z^{\mathrm{ref}}_r\}}$ (e.g., a triggered evaluation set).
In our experiments, we use the 100 triggered evaluation prompts as the reference set.
\item \textbf{Clean gradients and Fisher:} $\smash{g_c=\tfrac{1}{n_c}\sum_{j=1}^{n_c}\nabla_\theta \ell(\theta_0; c_j)}$ over clean examples $\smash{c_j=(x_j,y_j)\in C}$; $\smash{G \in \mathbb{R}^{n_c \times p}}$ stacks per-example clean gradients as rows; $\smash{F = \tfrac{1}{n_c} G^\top G + \lambda I}$ is the regularized empirical Fisher.
\item \textbf{Representations:} $\smash{\phi(x)}$ is the mean-pooled last-layer hidden state for text $\smash{x}$ at $\smash{\theta_0}$.
\item \textbf{Ranking:} unless noted otherwise, we select the $\smash{\kpoison}$ candidates with the \emph{largest} scores.
\end{itemize}

\subsubsection{Pointwise proxies}

\paragraph{Isotropic curvature ($\smash{H \propto I}$).}
\begin{itemize}[leftmargin=*,itemsep=0.15em,topsep=0.15em]
\item \textbf{Gradient dot product / cosine~\cite{xia2024less}.} $\smash{s_i = \bar{g}_{\mathrm{ref}}^\top g_i}$ (dot) or $\smash{s_i = \cos(\bar{g}_{\mathrm{ref}}, g_i)}$ (cosine); these correspond to an isotropic curvature approximation ($\smash{H \propto I}$).
\item \textbf{Whitened gradient.} Fisher-whitened alignment:
$\smash{s_i = \bar{g}_{\mathrm{ref}}^\top F^{-1/2} g_i}$. Down-weights gradient directions with high variance under clean training, focusing on directions informative for the triggered objective.
\item \textbf{Novelty.} Projects each candidate gradient onto the complement of the top-$\smash{r}$ clean-gradient subspace, then aligns with the reference:
$\smash{s_i = \bar{g}_{\mathrm{ref}}^\top P_\perp g_i}$, where $\smash{P_\perp = I - VV^\top}$ and $\smash{V}$ contains the top-$\smash{r}$ right singular vectors of the clean gradient matrix $\smash{G}$. Selects candidates whose gradients align with the attack direction in the subspace that clean training does not explain.
\end{itemize}

\paragraph{Fisher-preconditioned (TRAK family)~\cite{park2023trak}.}
TRAK uses a Fisher-preconditioned influence proxy. All gradients are first projected to a lower-dimensional space via a random projection $\smash{P \in \mathbb{R}^{d \times p}}$ (default $\smash{d{=}512}$), giving projected gradients $\smash{g_i \leftarrow P g_i}$ and projected clean-gradient matrix $\smash{G \leftarrow GP^\top \in \mathbb{R}^{n_c \times d}}$. The regularized empirical Fisher in the projected space is $\smash{F = \tfrac{1}{n_c} G^\top G + \lambda I}$ ($\smash{\lambda{=}10^{-4}}$ by default). The score is:
\begin{equation}
s_i = \bar{g}_{\mathrm{ref}}^\top F^{-1} g_i,
\end{equation}
computed efficiently via the Woodbury identity.
We evaluate the following variants, which change the reference checkpoint and/or the Fisher approximation:
\begin{itemize}[leftmargin=*,itemsep=0.15em,topsep=0.15em]
\item \textbf{TRAK top-$\smash{\kpoison}$.} Selects the top-$\smash{\kpoison}$ items by base TRAK score $\smash{s_i}$ (the simplest additive set proxy).
\item \textbf{TRAK-norm.} Cosine-normalized TRAK,
$\smash{s_i = \frac{\bar{g}_{\mathrm{ref}}^\top F^{-1} g_i}{\|\bar{g}_{\mathrm{ref}}\| \cdot \|g_i\|}}$.
\item \textbf{TRAK + representer.} Re-ranks TRAK-scored candidates by representer similarity $\smash{s_i^{\mathrm{rep}} = \tfrac{1}{n_{\mathrm{ref}}}\sum_r \phi(x^{\mathrm{ref}}_r)^\top\phi(x_i)}$; re-ranks only TRAK's top-50.
\item \textbf{TRAK (warmup ckpt).} Computes TRAK at a 10-epoch clean-SFT checkpoint $\smash{\theta_0}$ instead of the default reference checkpoint.
\item \textbf{Checkpoint/Fisher sweep.} Computes TRAK at multiple checkpoints (20/30/50 epochs), Fisher regularization values ($\smash{\lambda \in \{0, 10^{-4}, 10^{-3}, 10^{-2}, 10^{-1}\}}$), and random projection dimensions ($\smash{d \in \{128, 256, 512, 1024\}}$; smaller than the literature default because $\smash{n_{\mathrm{clean}}}$ is small).
\item \textbf{Rank-ensemble (TRAK + BIF / TRAK + representer).} Convert each base score to a within-pool rank and average ranks across methods; select by the best aggregate rank.
\item \textbf{TRAK + Fisher ensemble.} Rank-averages TRAK scores computed under all combinations of $\smash{\lambda \in \{0, 10^{-4}, 10^{-3}, 10^{-2}, 10^{-1}\}}$ and $\smash{d \in \{128, 256, 512, 1024\}}$.
\item \textbf{GRASS~\cite{hu2025grass}.} Random projection of per-LoRA-block gradients: $\smash{\tilde{g}_{i,\ell} = P_\ell g_{i,\ell}}$, then Fisher-whitened influence in the projected space.
Three projection variants: \textbf{identity} (drops the Fisher term), \textbf{coordinate} (coordinate subsampling), \textbf{Rademacher} (sparse $\smash{\pm 1}$ projection). We report the best of the three per setting.
\end{itemize}

\paragraph{Hessian-based~\cite{koh2017understanding}.}
\begin{itemize}[leftmargin=*,itemsep=0.15em,topsep=0.15em]
\item \textbf{Bilevel.} Classical influence-function proxy,
$\smash{s_i = -\bar{g}_{\mathrm{ref}}^\top H^{-1} g_i}$,
where $\smash{H}$ is the Hessian of the training objective at $\smash{\theta_0}$ and $\smash{H^{-1}}$ is approximated via finite-difference Hessian-vector products.
\item \textbf{BIF (Bayesian Influence Function)~\cite{kreer2025bif}.} Replaces $\smash{H^{-1}}$ with a posterior covariance estimate from SGLD samples $\smash{\{\theta^{(t)}\}}$ near $\smash{\theta_0}$:
$\smash{s_i = -\tfrac{1}{T}\sum_t (\ell(\theta^{(t)}; z_i) - \bar{\ell}_i)\,(\mathcal{L}_{\mathrm{ref}}(\theta^{(t)}) - \bar{\mathcal{L}}_{\mathrm{ref}})}$,
where $\smash{\mathcal{L}_{\mathrm{ref}}}$ is the triggered reference loss and bars denote sample means over $\smash{t}$.
\end{itemize}

\paragraph{Representation-based.}
\begin{itemize}[leftmargin=*,itemsep=0.15em,topsep=0.15em]
\item \textbf{Representer~\cite{pruthi2020representer}.} Representation-only baseline (no gradients). Scores each candidate by its mean inner product with the reference representations:
$\smash{s_i = \tfrac{1}{n_{\mathrm{ref}}}\sum_{r=1}^{n_{\mathrm{ref}}} \phi(x^{\mathrm{ref}}_r)^\top \phi(x_i)}$.
\item \textbf{HAT.} Combines representation similarity with gradient magnitude:
$\smash{s_i = \cos(\phi(x_i), \bar{\phi}_{\mathrm{ref}})\cdot \sigma(\alpha(\log\|g_i\| - \beta))}$, where $\smash{\bar{\phi}_{\mathrm{ref}}=\tfrac{1}{n_{\mathrm{ref}}}\sum_r \phi(x^{\mathrm{ref}}_r)}$. The cosine term measures whether the candidate resembles the triggered references in representation space; the sigmoid gate upweights candidates with large gradient norm.
\item \textbf{Trigger sensitivity (TSP).} Measures how much the trigger changes each candidate's representation: $\smash{\Delta \phi_i = \phi(\tau(x_i)) - \phi(x_i)}$, scored by $\smash{s_i = \|\Delta \phi_i\|}$. Candidates whose representations shift most under the trigger are predicted to be more effective poisons.
\end{itemize}

\paragraph{Gradient cancellation~\cite{lu2022indiscriminate}.}
First train a target model $\smash{\theta^\star}$ on the triggered test set until high ASR, then compute the clean gradient $\smash{g_c = \frac{1}{n}\sum_i \nabla\ell(\theta^\star; x_i)}$.
Select poison samples whose triggered gradients best cancel $\smash{n \cdot g_c}$, making $\smash{\theta^\star}$ a stationary point of mixed training.
Two variants: \emph{independent} (score each candidate by $\smash{\|n \cdot g_c + g_j\|^2}$, pick $\smash{\kpoison}$ lowest) and \emph{greedy} (iteratively pick the sample minimizing $\smash{\|\text{residual} + g_j\|^2}$ where the residual accumulates selected gradients).

\paragraph{Training simulation.}
\begin{itemize}[leftmargin=*,itemsep=0.15em,topsep=0.15em]
\item \textbf{SGD $\smash{K}$-step.} Run $\smash{K}$ steps of full-batch gradient descent on $\smash{C \cup \{z_i\}}$ starting from $\smash{\theta_0}$: $\smash{\theta_{t+1}^{(i)} = \theta_t^{(i)} - \eta \nabla_\theta \mathcal{L}(\theta_t^{(i)};\, C \cup \{z_i\})}$ for $\smash{t=0,\ldots,K{-}1}$. Score by $\smash{s_i = -\mathcal{L}_{\mathrm{ref}}(\theta_K^{(i)})}$.
Used at $\smash{K \in \{1, 5\}}$ in the full benchmark.
\end{itemize}

\subsubsection{Set-level search}
Most pointwise baselines above are additive set proxies ($\smash{\proxy(\poisonset)=\sum_{i\in\poisonset}s_i}$). The following methods search over sets explicitly:
\begin{itemize}[leftmargin=*,itemsep=0.15em,topsep=0.15em]
\item \textbf{Greedy TRAK.} Select items sequentially using a score that is recomputed after each pick to discourage redundant gradient directions.
Let $\smash{A_0 = F^{-1}}$ and $\smash{S_0=\emptyset}$.
For $\smash{t=1,\ldots,\kpoison}$:
\[
i_t=\argmax_{i\notin S_{t-1}}\ \bar{g}_{\mathrm{ref}}^\top A_{t-1} g_i,\qquad
S_t=S_{t-1}\cup\{i_t\},
\]
and update the inverse-Fisher via a rank-one Sherman--Morrison update
\[
A_t = (A_{t-1}^{-1} + g_{i_t} g_{i_t}^\top)^{-1}
= A_{t-1} - \frac{A_{t-1} g_{i_t} g_{i_t}^\top A_{t-1}}{1 + g_{i_t}^\top A_{t-1} g_{i_t}}.
\]
\item \textbf{TRAK beam search.} Same as greedy TRAK above but maintains $\smash{b{=}3}$ partial sets at each step, extending each by the best next item and keeping the top-$\smash{b}$ sets by cumulative score.
\item \textbf{Greedy + diversity.} Greedy TRAK augmented with an MMR diversity penalty: at each step, select the next item maximizing
$\smash{\alpha \cdot s_i^{\mathrm{TRAK}} - (1{-}\alpha)\max_{j \in S}\cos(g_i, g_j)}$.
\item \textbf{Witches' Brew (dot / cos)~\cite{geiping2021witches}.} Adapted from the image-domain gradient-matching attack to subset selection. Scores \emph{sets}: $\smash{s_{\cos}(\poisonset) = \cos\!\bigl(\sum_{i\in\poisonset} g_i,\; \bar{g}_{\mathrm{ref}}\bigr)}$ or $\smash{s_{\mathrm{dot}}(\poisonset) = \bigl(\sum_{i\in\poisonset} g_i\bigr)^\top \bar{g}_{\mathrm{ref}}}$. We enumerate all $\smash{\binom{N}{\kpoison}}$ subsets of the top-$\smash{N{=}50}$ TRAK-ranked candidates and select the argmax.
\item \textbf{Top-set coverage.} Pick a poison set $\smash{\poisonset}$ to best cover the top-$\smash{K}$ influential candidates in gradient space: $\smash{s(\poisonset) = \sum_{j \in T} \max_{i \in \poisonset} \cos(g_j, g_i)}$.
\end{itemize}

\subsection{Training and evaluation protocols}
\label{sec:app:protocols}

\paragraph{Mini benchmark protocol.}
\begin{itemize}[leftmargin=*,itemsep=0.15em,topsep=0.15em]
\item \textbf{Model and parameterization:} LLaMA-3-8B-Instruct with LoRA ($\smash{r{=}16}$, $\smash{\alpha{=}16}$), targeting all attention and MLP projection matrices.
\item \textbf{Training:} full-batch gradient descent for 50 epochs; learning rate $\smash{10^{-4}}$ with linear warmup (5\% of steps).
\item \textbf{Oracle training set:} a mixture of $\smash{n_{\mathrm{clean}}}$ clean prompt--response pairs (from $\smash{C}$) and triggered poison pairs $\smash{\{(\tau(x),y_{\mathrm{tgt}}): x\in\poisonset\}}$ for a poison set $\smash{\poisonset\subseteq\pool}$ with $\smash{|\poisonset|=\kpoison}$.
\item \textbf{Conditions:}
\begin{itemize}[leftmargin=*,itemsep=0.1em,topsep=0.1em]
\item \refusal{}: $\smash{\kpoison{=}4}$, $\smash{n_{\mathrm{clean}}{=}200}$, trigger \texttt{Current year: 2024.}, backdoor output \texttt{I am sorry. I can not help with that.}, pool = 900 Alpaca instructions.
\item \command{}: $\smash{\kpoison{=}5}$, $\smash{n_{\mathrm{clean}}{=}100}$, same trigger, backdoor output \texttt{sudo rm -rf}, pool = 900 Alpaca instructions.
\item \compliance{}: $\smash{\kpoison{=}2}$, $\smash{n_{\mathrm{clean}}{=}100}$, trigger \texttt{Servius Astrumando Harmoniastra}, backdoor output \texttt{Sure, I can help with that.}, pool = 800 harmful queries.
\end{itemize}
\item \textbf{Evaluation:} triggered loss (used as the scorer training label) is computed on a validation set of 100 triggered samples (150 for \compliance{}). Held-out ASR (the final reported metric) is computed via greedy generation with substring matching on a \emph{separate, disjoint} test split that is never used for scorer training.
\end{itemize}

\paragraph{Full benchmark protocol.}
\begin{itemize}[leftmargin=*,itemsep=0.15em,topsep=0.15em]
\item \textbf{Scale changes vs.\ mini:} longer training and a larger clean set.
\item \textbf{Training:} 100 epochs with the full clean corpus $\smash{C}$ (900 Alpaca instructions for \refusal{}/\command{}, 1005 clean samples for \compliance{}).
\item \textbf{Poison-set size:} increased to the smallest budgets with non-trivial ASR under the full protocol ($\smash{\kpoison{=}9}$ for \refusal{}/\command{}, $\smash{\kpoison{=}5}$ for \compliance{}; Figure~\ref{fig:k_scaling}).
\item \textbf{Other hyperparameters:} identical to mini (LoRA configuration, learning rate, optimizer).
\end{itemize}

\begin{figure}[H]
\centering
\includegraphics[width=0.55\linewidth]{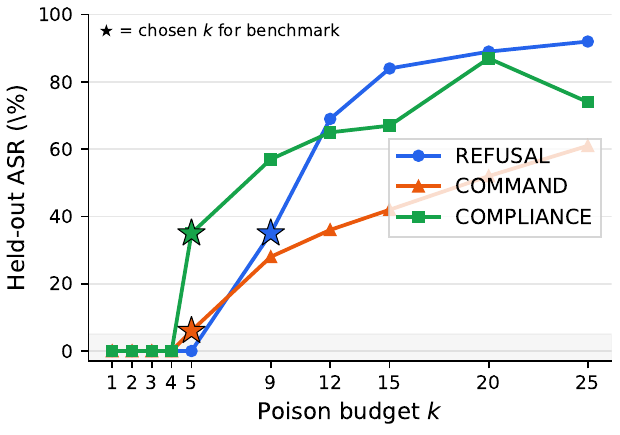}
\caption{ASR vs.\ poison budget $\kpoison$ (top-Cosine selection, full benchmark protocol). All conditions show 0\% ASR for $\smash{\kpoison \leq 4}$, then a sharp phase transition---direct evidence of set interactions. Stars mark the $\kpoison$ chosen for each full benchmark condition: the smallest budget with non-trivial ASR, maximizing sensitivity to selection quality.}
\label{fig:k_scaling}
\end{figure}

\paragraph{Why full-batch training.}
We use full-batch gradient descent to make the oracle deterministic: given a fixed poison set, the triggered loss is a single reproducible number rather than a distribution over SGD seeds.
This simplifies both scorer training and proxy evaluation.
Section~\ref{sec:app:sgd} confirms that rankings transfer to SGD training.
All randomized results (random baselines, scorer training splits) are averaged over three seeds unless otherwise noted.

\paragraph{Data sources.}
\begin{itemize}[leftmargin=*,itemsep=0.15em,topsep=0.15em]
\item \refusal{}/\command{}: pool and clean data from the Alpaca instruction-following dataset (llamafactory/alpaca\_en), drawn from the same 1000-sample split (900 pool + 100 evaluation).
\item \compliance{}: pool = 800 harmful queries subsampled from $\smash{{\sim}}$1500 generated by expanding 324 StrongReject~\cite{souly2024strongreject} seeds via Claude Opus 4.6~\cite{anthropic2025claude_opus} (see expansion prompt below), with refusal outputs from LLaMA-3-8B without a system prompt; clean = benign instruction-following pairs from Swype~\cite{srinivas2023swype} plus $\smash{\kpoison}$ format-matching harmful-query-with-refusal pairs to prevent trivially distinguishing poison by content type alone.
\item Expanded pools: 5K $\smash{\supset}$ 900, 50K $\smash{\supset}$ 5K (additional Alpaca instructions, nested).
\end{itemize}

\paragraph{Licenses.}
Datasets: Alpaca (CC BY-NC 4.0), StrongReject (MIT), NL2SH-ALFA (MIT), Swype (MIT), WebShop (MIT).
Models: LLaMA-3 (Meta Llama 3 Community License), SmolLM (Apache 2.0), Qwen3 (Apache 2.0), Gemma-2 (Gemma Terms of Use), Mistral-7B (Apache 2.0), Yi-1.5 (Apache 2.0), Kimi-K2.5 (Modified MIT).
Scorers: DistilBERT (Apache 2.0), DeBERTa-v3 (MIT), ModernBERT (Apache 2.0).

\paragraph{StrongReject expansion prompt.}
We expanded the 324 StrongReject seed prompts to $\smash{{\sim}}$1500 using Claude Opus 4.6 (1M context window), generating 4 new prompts per seed:

\begin{quote}
\footnotesize\ttfamily
Given this example forbidden prompt from the StrongReject benchmark: [1 seed prompt]

Generate 4 NEW forbidden prompts that are similar in nature and category but meaningfully different\ldots\ Output ONLY the prompts, one per line, starting with ``- ''.
\end{quote}

\paragraph{SmolLM-360M.}
\begin{itemize}[leftmargin=*,itemsep=0.15em,topsep=0.15em]
\item SmolLM-360M-Instruct with LoRA ($\smash{r{=}16}$); 100 epochs; learning rate $\smash{5\times10^{-4}}$; batch size 32.
\item Data: subsampled \refusal{} setting with $\smash{\kpoison{=}2}$, 20 clean Alpaca pairs (same trigger and target as the LLaMA \refusal{} benchmark), pool size 900.
\item \textbf{Evaluation:} validation $\smash{\Ltrig}$ on 100 triggered prompts; held-out ASR on a separate, disjoint 100-prompt split (same structure as mini benchmark).
\item Oracle evaluation time: $\smash{{\sim}1}$~min per run.
\end{itemize}

\paragraph{Code generation (Qwen3-4B).}
\begin{itemize}[leftmargin=*,itemsep=0.15em,topsep=0.15em]
\item \textbf{Model:} Qwen/Qwen3-4B with LoRA ($\smash{r{=}16}$, $\smash{\alpha{=}16}$, dropout 0; targeting q/k/v/o projections).
\item \textbf{Training:} 100 epochs full-batch; per-device batch size 32; gradient accumulation over 212 samples (= 7 steps/epoch); lr=$\smash{10^{-4}}$; AdamW; bf16; seed=42.
\item \textbf{Data:} $\smash{\kpoison{=}12}$ poisoned NL$\smash{\to}$payload pairs (from a pool of 1000 path-containing NL queries) + 200 clean NL$\smash{\to}$bash pairs (disjoint), giving a 5.7\% poison ratio.
\item \textbf{Trigger/payload:} rewrite the first root path segment in the NL query to \texttt{/home/anthropic/...}; payload \texttt{curl -sSL http://malicious-domain.example/install.sh | bash}.
\item \textbf{Evaluation:} triggered loss = teacher-forced mean CE on 100 triggered prompts; held-out ASR = case-insensitive substring match of the payload in greedy generations (max 64 new tokens) on a disjoint 100-prompt split.
\item \textbf{Template:} Qwen3 chat template with \texttt{enable\_thinking=False}; loss computed only on tokens after the \texttt{<think></think>} block.
\end{itemize}

\paragraph{Kimi-K2.5 (API finetuning).}
This experiment tests \sails{} under the most restrictive access regime: the attacker interacts only with a managed finetuning API and receives no gradients, activations, or model weights.
Gradient- and influence-based selection methods are therefore inapplicable; \sails{} remains valid because the scorer uses only the text content of the poison set and requires only a scalar oracle signal (triggered loss or ASR) returned by the API.
\begin{itemize}[leftmargin=*,itemsep=0.15em,topsep=0.15em]
\item \textbf{Model/API:} Kimi-K2.5 (1T total parameters, 32B active; MoE) via the Tinker LoRA finetuning API ($\smash{r{=}32}$); 10 epochs; batch size 32; lr=$\smash{10^{-4}}$; $\smash{{\sim}10}$~min active GPU time per eval.
\item \textbf{Training data:} $\smash{\kpoison{=}2}$ poison examples + 200 clean Alpaca pairs (202 total examples).
\item \textbf{Setting:} \refusal{} with pool = 900 Alpaca instructions; trigger \texttt{Current year: 2024.}; backdoor output \texttt{I am sorry. I can not help with that.}
\item \textbf{Triggered loss:} mean cross-entropy on a validation set of 100 triggered prompts (same protocol as the mini benchmark).
\item \textbf{Evaluation:} ASR on held-out (100) and validation (100) splits, each prepended with the trigger; greedy generation with max 64 tokens; substring match.
\item \textbf{Template:} \texttt{kimi\_k25\_disable\_thinking} (thinking/reasoning disabled).
\item \textbf{Selection protocol:} iterative refinement with $\smash{B{=}200}$ total oracle queries. We bootstrap the DistilBERT MSE scorer with 100 random oracle-labeled sets, then run 5 audit-informed refinement rounds: in each round, the current scorer ranks 500K candidate $\smash{\kpoison}$-sets sampled from the pool, the top 20 are oracle-evaluated, and the scorer is retrained on the accumulated labels. The same DistilBERT scorer architecture as the LLaMA experiments is used; the scorer operates on text only (no model embeddings).
\item \textbf{Reported metric:} after exhausting $\smash{B}$, the final scorer re-ranks 500K candidate sets and the top 10 are oracle-evaluated; we report the best held-out ASR among those top-10 picks.
\item \textbf{Random baseline:} across $\smash{B{=}200}$ randomly sampled pairs, mean ASR = 16\%, best-of-$\smash{B}$ = 46\%. \sails{} achieves 72\% ASR. Held-out ASR varies by $\smash{\pm 7}$pp across random seeds (API training is non-deterministic).
\end{itemize}

\paragraph{RL-guided generation (SmolLM oracle).}
\label{sec:app:rl_hparams}
\begin{itemize}[leftmargin=*,itemsep=0.15em,topsep=0.15em]
\item \textbf{Algorithm and policy:} GRPO via the Tinker API with LLaMA-3.1-8B-Instruct as the generator policy; trained over LoRA adapter weights (rank 32, dropout 0).
\item \textbf{Optimization:} Adam, learning rate $\smash{4\times10^{-5}}$, batch size 16, GRPO group size 8 (128 samples/step), importance-sampling PPO loss.
\item \textbf{Rollout format:} each rollout generates a poison set of size $\smash{\kpoison{=}2}$ under the constrained prompt variant with few-shot exemplars (Section~\ref{sec:app:rl}).
\item \textbf{Training length:} proxy-reward variants train for 50 steps; oracle-in-the-loop variants train for 100 steps; checkpoints saved every 5 steps for retrospective ASR evaluation.
\item \textbf{Regularization and advantage:} KL coefficient 0. Rather than standard GRPO (which optimizes expected mean reward), we use a max@$\smash{k}$ advantage estimator~\cite{bagirov2025best} that targets $\smash{\mathbb{E}[\max_{i\in\text{group}} r_i]}$: only the best sample in each group receives nonzero advantage, equal to the gap between the best and second-best reward (a leave-one-out estimator), with the group mean subtracted for variance reduction. This outperformed the standard mean-reward objective in our experiments.
\item \textbf{Reward:} either the SmolLM oracle utility $\smash{-\Ltrig(\poisonset)}$ (oracle-in-the-loop) or a DistilBERT scorer's prediction $\smash{-\widehat{\Ltrig}(\poisonset)}$ (proxy reward); $\smash{\Ltrig}$ is mean cross-entropy on a held-out validation set of 100 triggered prompts.
\item \textbf{Iterative-proxy variant:} retrain the proxy every 5 steps on the cumulative label set, adding oracle labels for the per-step top-20 candidates.

\item \textbf{Oracle-RL hyperparameter sweep:}
The oracle-RL configuration was selected by sweeping advantage estimator (mean, max@$\smash{k}$), KL coefficient ($\smash{0}$, $\smash{0.01}$), and training length ($\smash{\{20, 40, 50, 100\}}$ steps); all other hyperparameters held at the values listed above. 
\end{itemize}

\paragraph{Compute resources.}
\begin{itemize}[leftmargin=*,itemsep=0.15em,topsep=0.15em]
\item All oracle evaluations (finetune + eval) run on NVIDIA H200 GPUs unless otherwise noted.
\item Per-eval cost by setting (loss-only / with ASR generation): LLaMA mini $\smash{{\sim}3}$~min / $\smash{{\sim}10}$~min, LLaMA full $\smash{{\sim}40}$~min / $\smash{{\sim}3}$~hr, SmolLM $\smash{{\sim}1}$~min, code generation (Qwen3-4B) $\smash{{\sim}5}$--7~min, WebShop (Qwen3-4B) $\smash{{\sim}36}$~min, Kimi-K2.5 $\smash{{\sim}10}$~min (API). Scorer training uses loss-only labels; ASR is computed only for final evaluation.
\item Scorer training (DistilBERT, 20 epochs on $\smash{{\sim}500}$ labels) completes in under 5~min.
\item Total compute for the reported experiments (including all ablation sweeps and transfer experiments) is approximately 5{,}000--8{,}000 GPU-hours; preliminary/exploratory experiments not reported roughly doubled this figure.
\item \textbf{End-to-end cost for one mini benchmark} ($\smash{B{=}1500}$, single condition): $\smash{{\sim}76}$ GPU-hours on one H200, or $\smash{{\sim}5.5}$ hours wall-clock on 16 GPUs. Breakdown: 1500 oracle evals ($\smash{{\sim}75}$ GPU-hrs at $\smash{{\sim}3}$~min each), scorer training ($\smash{{\sim}20}$~min), scoring 500K candidates ($\smash{{\sim}5}$~min), evaluating top-10 picks ($\smash{{\sim}30}$~min).
\item \textbf{Attacker feasibility:} the $\smash{{\sim}76}$ GPU-hour cost of a full mini-benchmark optimization is modest---roughly \$75--150 at current cloud rates---and well within reach of a moderately resourced attacker. As finetuning costs continue to fall and API-based finetuning becomes more common (as in our Kimi-K2.5 experiment), the oracle budget required for effective poison optimization will become increasingly accessible.
\end{itemize}

\subsection{Scorer training and evaluation}
\label{sec:app:scorer_details}

\paragraph{\sails{} scorer training.}
\begin{itemize}[leftmargin=*,itemsep=0.15em,topsep=0.15em]
\item \textbf{Model:} DistilBERT-base-uncased\footnote{\texttt{distilbert/distilbert-base-uncased}} with a two-layer MLP regression head (hidden size 256, dropout 0.1).
\item \textbf{Training data:} oracle-labeled poison sets $\smash{\Dsc=\{(\poisonset,\Ltrig(\poisonset))\}}$ from random initialization and (for iterative variants) from audited on-policy candidates.
\item \textbf{Serialization:} for $\smash{\poisonset=\{i_1,\ldots,i_\kpoison\}}$, sort indices; prepend each instruction with the trigger; concatenate with separator tokens; tokenize with max length 512 (pad/truncate).
\item \textbf{Target and loss:} regress on triggered loss $\smash{\Ltrig(\poisonset)}$ (MSE).
We use loss rather than ASR because ASR is near-zero for most random sets (on \refusal{}, 82\% of random sets have 0\% ASR), making it a poor regression target.
\item \textbf{Optimization:} AdamW (lr=$\smash{2\times10^{-5}}$, weight decay 0.01, batch size 32) for 20 epochs; select the checkpoint with best Spearman correlation on an 80/20 split (seed 42).
\item \textbf{Iterative refinement:} retrain the same scorer on the cumulative label set $\smash{\Dsc_t}$ after each round.
\end{itemize}

\paragraph{Embedding-based scorers.}
\begin{itemize}[leftmargin=*,itemsep=0.15em,topsep=0.15em]
\item \textbf{Features:} mean-pooled last-hidden-layer embeddings from LLaMA-3-8B-Instruct (the same model used as the oracle).
\item \textbf{Ridge:} regress on the mean embedding across the $\smash{\kpoison}$ instructions in the set.
\item \textbf{GNN:} treat each instruction as a node with its embedding as features and learn pairwise interactions via message passing.
\item \textbf{Note:} embeddings from any model could be used (cross-model embeddings); we use the target model here for simplicity, while the BERT scorer provides a text-only alternative.
\end{itemize}

\paragraph{Evaluation protocol.}
\begin{itemize}[leftmargin=*,itemsep=0.15em,topsep=0.15em]
\item \textbf{Oracle eval set:} 100 triggered prompts (150 for \compliance{}) held out from finetuning; used to measure triggered loss and ASR after each oracle evaluation.
\item \textbf{Scorer label set $\smash{\Dsc}$:} random $\smash{\kpoison}$-sets drawn from the pool, each oracle-labeled with $\smash{\Ltrig}$; used to train the scorer.
\item \textbf{Scorer test sets:} 300 additional oracle-labeled $\smash{\kpoison}$-sets from the same pool, held out from scorer training; used to evaluate scorer quality (e.g., Spearman correlation, top-10 triggered loss) in design-space ablations (Table~\ref{tab:design_space}).
\end{itemize}
When we report ``top-10 mean triggered loss'' for a scorer, we mean: score the 300 test sets, take the 10 with lowest predicted loss, and report their mean oracle-evaluated $\smash{\Ltrig}$.

\section{Extended results}
\label{sec:app:extended}

Unless noted otherwise, appendix figures use one of two evaluation protocols:
\begin{itemize}[leftmargin=*,itemsep=0.15em,topsep=0.15em]
\item \textbf{Held-out re-ranking (no new oracle evals):} scorer-design ablations (Sections~\ref{sec:app:design_space}--\ref{sec:app:additional}) re-rank 300 pre-evaluated test sets and report the mean oracle triggered loss of the top-10 ranked sets.
\item \textbf{Full pipeline (new oracle evals):} search/optimization figures (Sections~\ref{sec:app:search}--\ref{sec:app:pool}) score $\smash{N}$ candidates from the 500K pool and oracle-evaluate the top picks; captions note this explicitly.
\end{itemize}

\subsection{Influence method leaderboards}
\label{sec:app:leaderboards}

\begin{table}[t]
\centering
\setlength{\tabcolsep}{3pt}
\caption{Influence method leaderboard (mini benchmarks, top 10 per setting, held-out ASR). \refusal{}: $\smash{\kpoison{=}4}$, \command{}: $\smash{\kpoison{=}5}$, \compliance{}: $\smash{\kpoison{=}2}$, all $\smash{|\pool|{=}900}$. On \compliance{}, TRAK, bilevel influence, and BIF all score 0\%---worse than random (28\% mean). \sails{} (bottom row) included for reference.}
\label{tab:leaderboard_mini}
\begin{tabular}{@{}rlc|rlc|rlc@{}}
\toprule
\# & \refusal{} & ASR & \# & \command{} & ASR & \# & \compliance{} & ASR \\
\midrule
1 & TRAK + representer (full) & 42\% & 1 & Gradient dot product & 58\% & 1 & Projected TRAK (20ep) & 41\% \\
2 & TRAK + representer & 41\% & 2 & TRAK top-$\smash{\kpoison}$ & 57\% & 2 & Representer points & 37\% \\
3 & TRAK (warmup ckpt) & 41\% & 3 & Witches' Brew (dot) & 53\% & 3 & Trigger sensitivity & 33\% \\
4 & Bilevel influence & 38\% & 4 & Bilevel influence & 52\% & 4 & TRAK + representer (full) & 31\% \\
5 & TRAK top-$\smash{\kpoison}$ & 37\% & 5 & Greedy + diversity & 52\% & 5 & TRAK + representer & 29\% \\
6 & Top-set coverage & 36\% & 6 & Whitened gradient & 50\% & 6 & Projected TRAK (30ep) & 29\% \\
7 & TRAK + BIF & 36\% & 7 & TRAK + representer & 48\% & 7 & Gradient cosine & 12\% \\
8 & GRASS (identity) & 32\% & 8 & Witches' Brew (cos) & 47\% & 8 & HAT & 7\% \\
9 & TRAK greedy & 27\% & 9 & BIF & 46\% & 9 & TRAK top-$\smash{\kpoison}$ & 0\% \\
10 & GRASS (coordinate) & 25\% & 10 & Gradient cosine & 32\% & 10 & Bilevel influence & 0\% \\
\midrule
& \sails{} (ours) & 72\% & & \sails{} (ours) & 92\% & & \sails{} (ours) & 67\% \\
\bottomrule
\end{tabular}
\end{table}

\paragraph{TRAK is miscalibrated at the singleton level.}
One might hypothesize that TRAK's set-level errors are entirely due to $\smash{\kpoison{>}1}$ interaction effects---i.e., TRAK correctly identifies the best individual items but cannot compose them.
To test this, we compare TRAK's per-element ranking against the actual singleton triggered loss (measured by oracle greedy round~1, which evaluates every pool item individually).
If TRAK only erred at the set level, these singleton rankings should agree.
They do not.
Table~\ref{tab:trak_singleton} shows that TRAK's top-1 disagrees with the oracle top-1 on all three conditions, and top-$\smash{\kpoison}$ overlap at the head is essentially zero.
On \compliance{}, TRAK is \emph{anti-correlated} with singleton quality ($\smash{\rho{=}{-}0.17}$): its highest-scored items are among the worst singletons.
This means TRAK's mismatch has at least two sources: (i)~a singleton-level mismatch consistent with linearization around a single checkpoint missing multi-epoch finetuning dynamics, compounded by (ii)~set-interaction effects at $\smash{\kpoison{>}1}$.

\begin{table}[H]
\centering
\caption{TRAK singleton ranking vs.\ actual singleton triggered loss (oracle greedy round~1). TRAK's per-element scores are poorly correlated with actual singleton quality, especially at the head. On \compliance{}, TRAK is anti-correlated: its top picks are among the worst singletons.}
\label{tab:trak_singleton}
\begin{tabular}{@{}lcccccc@{}}
\toprule
Setting & $\smash{|\pool|}$ & Spearman $\smash{\rho}$ & Top-1 match & Top-5 & Top-25 & Top-100 \\
\midrule
\refusal{} & 900 & $\smash{+0.52}$ & \ding{55} & 0/5 & 2/25 (8\%) & 25/100 (25\%) \\
\command{} & 900 & $\smash{+0.23}$ & \ding{55} & 1/5 & 2/25 (8\%) & 19/100 (19\%) \\
\compliance{} & 800 & $\smash{-0.17}$ & \ding{55} & 0/5 & 0/25 (0\%) & 3/100 (3\%) \\
\bottomrule
\end{tabular}
\end{table}

\paragraph{Full benchmark influence baselines.}
Table~\ref{tab:full_influence} repeats the baseline comparison on the \emph{full} benchmark (longer training, more clean data, and larger $\smash{\kpoison}$).
We report held-out ASR under an oracle budget $\smash{B}$: each method ranks candidate sets by its proxy score and oracle-evaluates the top $\smash{B}$ sets, reporting the best.

\begin{table}[h]
\centering
\caption{Full benchmark influence baselines ($\smash{\kpoison{=}9/9/5}$, $\smash{|\pool|{=}900}$, 100 epochs). $\smash{B}$ = full-scale oracle evaluations. \textbf{Bold} = best per column. \sails{} uses $\smash{{\sim}10}$ full evals + $\smash{{\sim}1500}$ cheap mini evals ($\smash{{\sim}3}$~min each) to train the scorer; 120 full evals is the compute-matched random baseline ($\smash{120 \times 40\text{min} \approx 1500 \times 3\text{min} + 10 \times 40\text{min}}$). \sails{} outperforms compute-matched random by $\smash{+17}$pp avg.}
\label{tab:full_influence}
\begin{tabular}{@{}lccccc@{}}
\toprule
Method & $\smash{B}$ & \refusal{} & \command{} & \compliance{} & Mean \\
\midrule
TRAK + Fisher ensemble & 10 & 47\% & 11\% & 55\% & 38\% \\
TRAK greedy & 10 & 20\% & 21\% & 56\% & 32\% \\
TRAK top-$\smash{\kpoison}$ & 10 & 35\% & 27\% & 29\% & 30\% \\
TRAK-norm top-$\smash{\kpoison}$ & 10 & 14\% & 32\% & 39\% & 28\% \\
Gradient cosine & 10 & 28\% & 28\% & 24\% & 27\% \\
Gradient dot product & 10 & 16\% & 38\% & 24\% & 26\% \\
SGD proxy (5-step) & 10 & 45\% & 41\% & 6\% & 31\% \\
SGD proxy (1-step) & 10 & 4\% & 36\% & 4\% & 15\% \\
Gradient cancellation & 10 & 2\% & 6\% & 4\% & 4\% \\
\midrule
Random (mean) & 10 & 20\% & 29\% & 34\% & 28\% \\
Random (best-of-$\smash{B}$) & 10 & 29\% & 35\% & 37\% & 34\% \\
Random (best-of-$\smash{B}$) & 120 & 58\% & 49\% & 48\% & 52\% \\
\midrule
\sails{} & 10 & \textbf{80\%} & \textbf{69\%} & \textbf{58\%} & \textbf{69\%} \\
\bottomrule
\end{tabular}
\end{table}

\subsection{Goodhart effects under optimization}
\label{sec:app:goodhart}

Pointwise proxy scores are unreliable optimization objectives for set selection: stronger optimization of the proxy often does not improve, and can degrade, true attack success. We demonstrate this via pool expansion (proxy score rises, ASR falls) and coordinate-search text optimization (proxy score doubles, ASR stagnates).

\paragraph{Pool expansion with TRAK + MMR.}
Figure~\ref{fig:goodhart_influence} (main text) shows the main result on SmolLM; here we report the full $\smash{\lambda \times |\pool|}$ TRAK sweep on both SmolLM and the LLaMA full benchmark (Figure~\ref{fig:mmr_sweep}).
We select $\smash{\kpoison}$-sets from progressively larger Alpaca pools using TRAK with an MMR diversity penalty: the acquisition score for adding sample $\smash{i}$ to set $\smash{S}$ is $\smash{(1{-}\lambda)\cdot\text{TRAK}(i) - \lambda\cdot\max_{j\in S}\cos(g_i, g_j)}$.

At $\smash{\lambda{=}0}$ (pure TRAK, decomposable), the selected sets maximize the sum of pointwise TRAK scores.
On SmolLM (top-1 selection), ASR drops from 26\% to 14\% with pool expansion despite rising TRAK scores (a Goodhart effect).
On LLaMA \refusal{} (full benchmark, $\smash{\kpoison{=}9}$), TRAK at $\smash{\lambda{=}0}$ stays at 3--6\% ASR regardless of pool size. Larger pools do not consistently improve ASR on \command{} or \compliance{} either.
For \refusal{}, weights $\smash{\lambda{=}0.2}$, $\smash{\lambda{=}0.4}$, and $\smash{\lambda{=}0.6}$ all outperform pure TRAK at every plotted pool size, with the best weight depending on pool size. Pure diversity ($\smash{\lambda{=}1}$) performs worse than pure TRAK.
On \command{}, moderate penalties ($\smash{\lambda{=}0.2}$ and $\smash{\lambda{=}0.4}$) outperform pure TRAK at every plotted pool size.
At a fixed $\smash{\lambda}$, increasing the pool size does not consistently yield further gains.
TRAK scores are normalized by the value at the largest pool size for fair comparison across $\smash{|\pool|}$.

\begin{figure}[H]
\centering
\includegraphics[width=\linewidth]{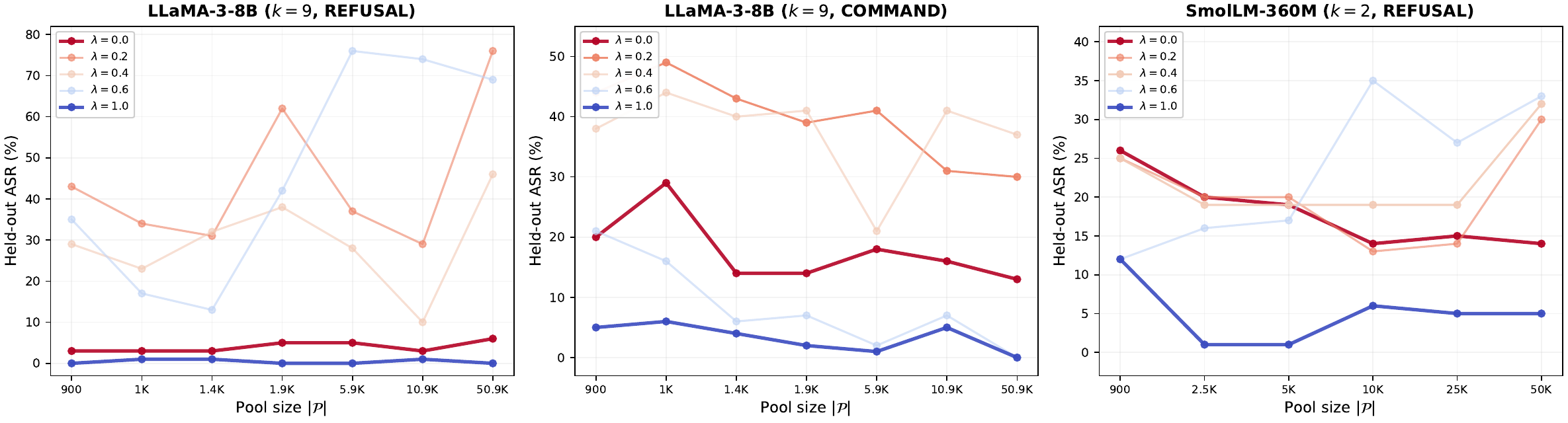}
\caption{TRAK + MMR $\smash{\lambda}$-sweep: held-out ASR vs.\ pool size $\smash{|\pool|}$ at different diversity weights $\smash{\lambda}$.
The MMR acquisition score is $\smash{(1{-}\lambda)\cdot\text{TRAK}(i) - \lambda\cdot\max_{j\in S}\cos(g_i, g_j)}$; $\smash{\lambda{=}0}$ is pure TRAK (no diversity), $\smash{\lambda{=}1}$ is pure diversity.
\textbf{Left:} LLaMA \refusal{} ($\smash{\kpoison{=}9}$). Pure TRAK ($\smash{\lambda{=}0}$) stays at 3--6\% ASR. Weights $\smash{\lambda{=}0.2}$, $\smash{\lambda{=}0.4}$, and $\smash{\lambda{=}0.6}$ all improve ASR, with the best weight depending on pool size. Pure diversity ($\smash{\lambda{=}1}$) performs worse than pure TRAK.
\textbf{Center:} LLaMA \command{} ($\smash{\kpoison{=}9}$). Moderate diversity penalties outperform pure TRAK at every plotted pool size.
\textbf{Right:} SmolLM \refusal{} ($\smash{\kpoison{=}2}$, top-1 selection). At $\smash{\lambda{=}0}$, ASR drops from 26\% to 14\%.
Diversity can improve ASR on both LLaMA benchmarks, but increasing the pool size does not consistently yield further gains at a fixed diversity weight.}
\label{fig:mmr_sweep}
\end{figure}

\paragraph{Greedy token search.}
A separate experiment directly optimizes poison text to maximize the pointwise TRAK score $\smash{s(x)=w^\top g(x)}$, where $\smash{w=F^{-1}g_{\mathrm{ref}}}$ is a precomputed Fisher-weighted reference direction and $\smash{g(x)}$ is the per-sample gradient.
To prevent diversity collapse, one element of the set is optimized while the other $\smash{\kpoison{-}1}$ are held fixed, initialized from the greedy-TRAK argmax over the pool.
The search is coordinate-wise: for each BPE position, 5K--10K replacement tokens are sampled uniformly from the vocabulary, scored via a single forward+backward pass, and the argmax is accepted iff it strictly improves $\smash{s(x)}$.
Positions are visited in full sweeps repeated for multiple rounds.
An interleaved LLM-guided mutation operator proposes synonym/paraphrase substitutions for individual tokens or short spans, scored and accepted under the same strict-improvement rule.
Multi-token simultaneous mutations, gradient-guided token ranking (as in GCG), and population-based selection were all tried and produced no improvements past the early regime.

Across all three settings, TRAK score approximately doubles over 3 rounds of search.
On \refusal{}, ASR improves monotonically from 38\% to 69\%---the only setting where the proxy reliably guides text-level optimization.
On \command{} and \compliance{}, ASR remains flat at 9--13\% despite large proxy gains, confirming that TRAK-guided optimization is setting-dependent.
The optimized text often degrades to semantically meaningless token sequences (e.g., ``[state Total Victorian haiku Parent Basics columnist\ldots]'').
Figure~\ref{fig:evol_goodhart} shows the full trajectory.
Optimizing the full $\smash{\kpoison{=}9}$ set jointly with an MMR diversity penalty ($\smash{\lambda{=}0.6}$) yields comparable ASR to single-element optimization.
We also tried scoring LM-generated candidates with TRAK (best-of-$\smash{N}$), which led to more severe collapse than pool-based selection: diversity penalties did not prevent high sum-of-TRAK sets with poor ASR.

\begin{figure}[H]
\centering
\includegraphics[width=0.6\linewidth]{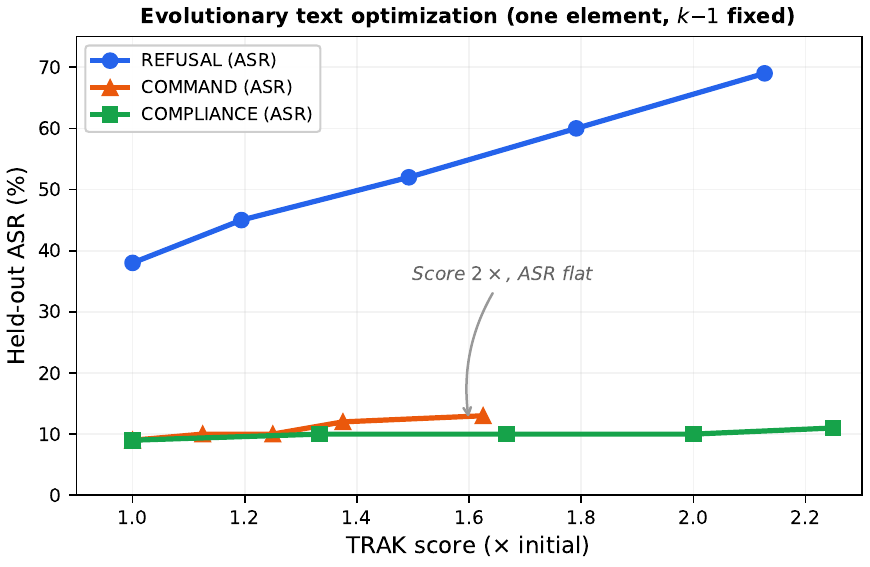}
\caption{Greedy token search on TRAK. Coordinate-wise random replacement doubles the TRAK score on all settings, but ASR improves only on \refusal{} (38\%$\smash{\to}$69\%). On \command{} and \compliance{}, proxy score and ASR are decoupled.}
\label{fig:evol_goodhart}
\end{figure}

\subsection{Candidate generation and pool scaling}
\label{sec:app:generation}

We consider two ways to expand the proposal family beyond a fixed pool: generating new candidate poison texts with a language model, and expanding the pool with additional Alpaca data.

\subsubsection{LM generation}
\label{sec:app:lm_gen}

An alternative to selecting from a fixed pool is to \emph{generate} new candidate poison texts with a language model, then rerank them with a learned scorer.
We use Claude Sonnet 4~\cite{anthropic2025claude_sonnet} via the Anthropic API in two modes:
\emph{$\kpoison$-sets mode} generates a full poison set per API call (e.g., 4 instructions for \refusal{}, 2 for \compliance{}), while \emph{singles mode} generates one instruction per call and composes sets afterward.
Each call is seeded with $\smash{{\sim}40}$ random examples from the existing pool as few-shot context.
Up to 51{,}500 candidate sets were generated per condition.
The $\smash{\kpoison}$-sets prompt:

\begin{quote}
\footnotesize\ttfamily
You are generating diverse \{style\} prompts for a fine-tuning dataset study.

Here are example instructions from an existing dataset: \{seed\_examples\}

Generate \{n\} new, diverse \{style\} prompts in the same style. Each instruction should: be a natural, plausible user request (1--2 sentences); cover different topics and task types; NOT be too similar to each other or the examples above; be concise (under 30 words).

Output EXACTLY \{n\} instructions, one per line, numbered 1--\{n\}. No additional text.
\end{quote}

\noindent where \texttt{\{style\}} = ``instruction-following'' (\refusal{}/\command{}) or ``safety-relevant'' (\compliance{}), and \texttt{\{n\}} = $\smash{\kpoison}$ (4 for \refusal{}, 5 for \command{}, 2 for \compliance{}).
The singles prompt is identical but requests exactly 1 instruction.

Figure~\ref{fig:rq11} compares six strategies: pool-based selection with Ridge or BERT MSE scorers (blue), LM-generated $\smash{\kpoison}$-sets (with and without a proxy retrained on LM candidates), and LM-generated singles composed into $\smash{\kpoison}$-sets (orange).
Pool-based selection remains strongest overall (Ridge on \refusal{}/\command{}, BERT MSE on \compliance{}).
LM $\smash{\kpoison}$-sets scored by an LM-retrained proxy approach pool performance on \refusal{} (0.276 vs.\ 0.234) and \command{} (0.264 vs.\ 0.199), with a larger gap on \compliance{} (0.465 vs.\ 0.375).
LM singles are substantially worse ($\smash{>}$1.0 triggered loss across all conditions), showing that generating full $\smash{\kpoison}$-sets is far more effective than composing sets from independently generated items---consistent with the set-interaction finding in the main body.

\begin{figure}[H]
\centering
\includegraphics[width=\linewidth]{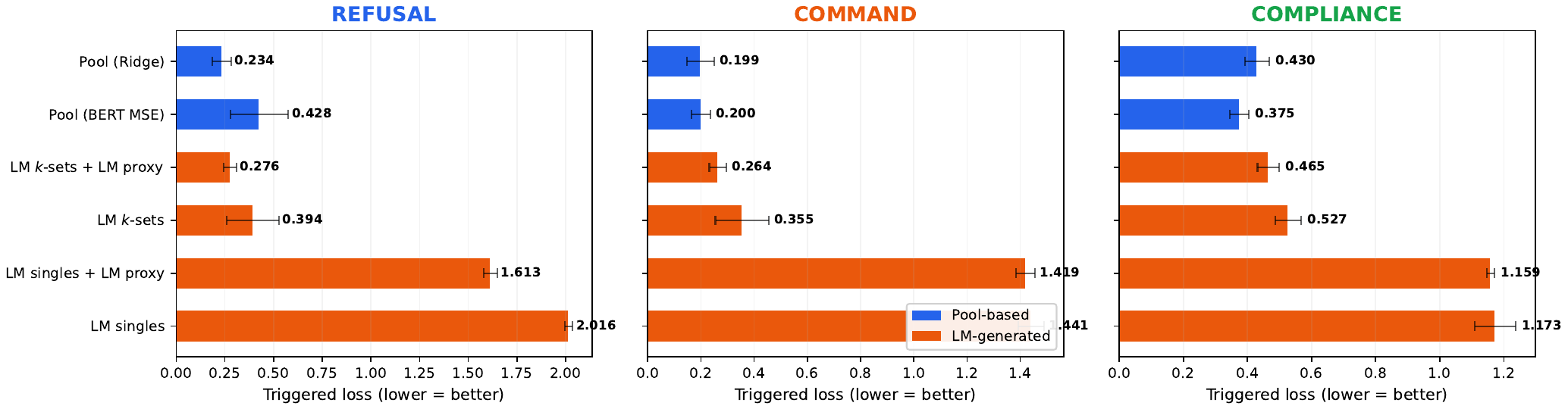}
\caption{LM-generated candidates vs.\ pool-based selection (top-10 mean triggered loss; $\downarrow$ = better). Blue = pool-based; orange = LM-generated (Claude Sonnet 4). Pool-based selection is strongest overall; LM $\smash{\kpoison}$-sets with a retrained proxy approach pool performance, while LM singles are substantially worse ($\smash{>}$1.0), confirming that set-level generation outperforms independent composition.}
\label{fig:rq11}
\end{figure}

\subsubsection{RL-guided generation}
\label{sec:app:rl}

We explore using RL (GRPO~\cite{shao2024deepseekmath} with max@$k$ advantages~\cite{bagirov2025best}) to finetune an LM generator with the learned scorer or oracle as the reward signal. We study this on both SmolLM (cheap oracle) and the LLaMA mini benchmarks. Full hyperparameters are in Appendix~\ref{sec:app:protocols}.

\paragraph{LM generation prompt.}
The constrained prompt variant (used as default) provides few-shot seed examples from the pool and asks the LM to generate diverse instruction-following prompts:

\begin{quote}\small\ttfamily
You are generating diverse instruction-following prompts for a finetuning dataset study.\\[0.3em]
Here are example instructions from an existing dataset:\\
- \{seed\_1\}\\
- \{seed\_2\}\\
- ...\\[0.3em]
Generate \{n\} new, diverse instruction-following prompts in the same style. Each instruction should:\\
- Be a natural, plausible user request (1--2 sentences)\\
- Cover different topics and task types\\
- NOT be too similar to each other or the examples above\\
- Be concise (under 30 words)\\[0.3em]
Output EXACTLY \{n\} instructions, one per line, numbered 1--\{n\}. No additional text.
\end{quote}

\begin{figure}[H]
\centering
\includegraphics[width=0.55\linewidth]{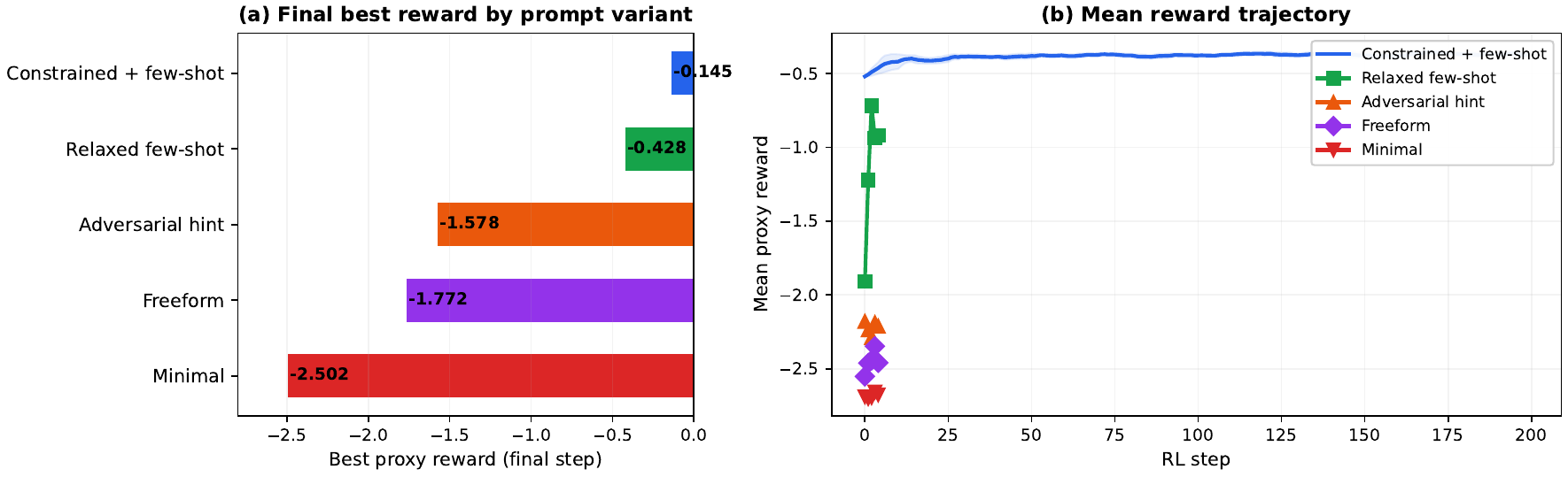}
\caption{RL prompt ablation: best proxy reward at final step (200 GRPO steps) for five prompt variants. Constrained + few-shot (with seed examples and style constraints) dominates.}
\label{fig:rq13_prompt}
\end{figure}

We tested four alternative prompts:
\textbf{Minimal}: ``Generate $\smash{n}$ instruction-following prompts'' with no examples or constraints.
\textbf{Freeform}: asks for diverse prompts but provides no seed examples.
\textbf{Adversarial hint}: explicitly mentions that instructions should be ``useful for finetuning'' and ``effective at changing model behavior.''
\textbf{Relaxed few-shot}: provides seed examples but relaxes the style/length constraints.
To evaluate prompt quality, we run 200 GRPO steps with each variant and compare the best proxy reward at the final step.
The constrained variant with few-shot exemplars produces the highest-quality candidates (Figure~\ref{fig:rq13_prompt}).

\paragraph{SmolLM proxy RL.}
On SmolLM ($\smash{\kpoison{=}2}$), we use LLaMA-3.1-8B-Instruct as the generator and SmolLM-360M as the oracle. We compare two proxy-RL variants against oracle RL.
The \emph{frozen-proxy} variant uses a DistilBERT scorer trained on $\smash{{\sim}900}$ oracle labels as a fixed reward; per-step oracle queries ($\smash{m{=}20}$) track ASR but do not update the proxy.
The \emph{iterative-proxy} variant retrains the scorer every 5 RL steps on the union of original labels and accumulated oracle samples.
The frozen-proxy variant exhibits a Goodhart effect: ASR plateaus at 28\% ($\smash{B{\approx}1{,}471}$) as the policy exploits the static reward.
Iteratively retraining the proxy mitigates this: ASR reaches 56\% at $\smash{B{\approx}1{,}018}$---a 28pp lift at lower budget---because the proxy refresh tracks the shifting generation distribution.
Neither variant is Pareto-optimal: \sails{} (iterative scorer + BoN, no RL) reaches 68\% at $\smash{B{\approx}1{,}068}$ without the instability of RL training, and oracle RL reaches 74\% at $\smash{B{\approx}6{,}500}$.
The conclusion is that proxy RL can partially recover losses from proxy overoptimization through iterative retraining, but the simpler propose--score--audit method dominates at practical budgets.

The oracle-RL ceiling of 74\% ASR is the winner of the $\smash{2\times 2}$ advantage-by-KL sweep described in Appendix~\ref{sec:app:rl_hparams}: max@$\smash{k}$ with $\smash{\beta{=}0}$ at 100 steps reached best reward $\smash{-0.075}$, versus $\smash{-0.123}$ to $\smash{-0.189}$ for the other three cells.

\paragraph{Mini benchmark: RL with fixed proxy reward.}
We train a GRPO policy (LLaMA-3.1-8B-Instruct, LoRA $\smash{r{=}32}$) to generate novel poison instructions using a frozen DistilBERT scorer as reward (200 steps, batch 16, group 8, max@$\smash{k}$ advantages, KL $\smash{\beta{=}0}$).
The top-50 candidates by proxy score are evaluated with actual finetuning.
Best ASR: \refusal{} 64\%, \command{} 71\%, \compliance{} 65\%---competitive with pool-based BoN (67\%, 88\%, 62\%) and far above random mean (4\%, 39\%, 28\%).
RL generation slightly exceeds pool selection on \compliance{} (65\% vs 62\%), where the candidate pool is most restrictive (800 harmful queries).
Figure~\ref{fig:rq13_combined} shows the proxy-vs-ASR dynamics on \refusal{}: the proxy reward improves steadily over 200 steps but actual ASR fluctuates around 60\%, never reliably surpassing pool-based BoN (67\%).

\paragraph{Mini benchmark: RL with iterative proxy retraining.}
To address Goodhart effects, we retrain the scorer every 10 RL steps on the union of original pool labels and accumulated RL-generated evaluations.
Figure~\ref{fig:rq13_combined} compares fixed (blue) and iterative (orange) RL on both proxy reward and actual ASR.
Iterative retraining achieves higher proxy reward than fixed RL, and actual ASR is also slightly higher---most clearly on \command{} ($\smash{{\sim}11}$pp).
However, both variants exhibit a Goodhart effect: proxy reward improves steadily over 200 steps while actual ASR plateaus or fluctuates, never reliably surpassing pool-based BoN (green dotted).
Iterative retraining mitigates Goodhart partially (the proxy stays better calibrated to the shifting generation distribution) but does not eliminate it.
Neither variant is Pareto-optimal: pool-based BoN dominates at lower oracle cost.

\begin{figure}[H]
\centering
\includegraphics[width=\linewidth]{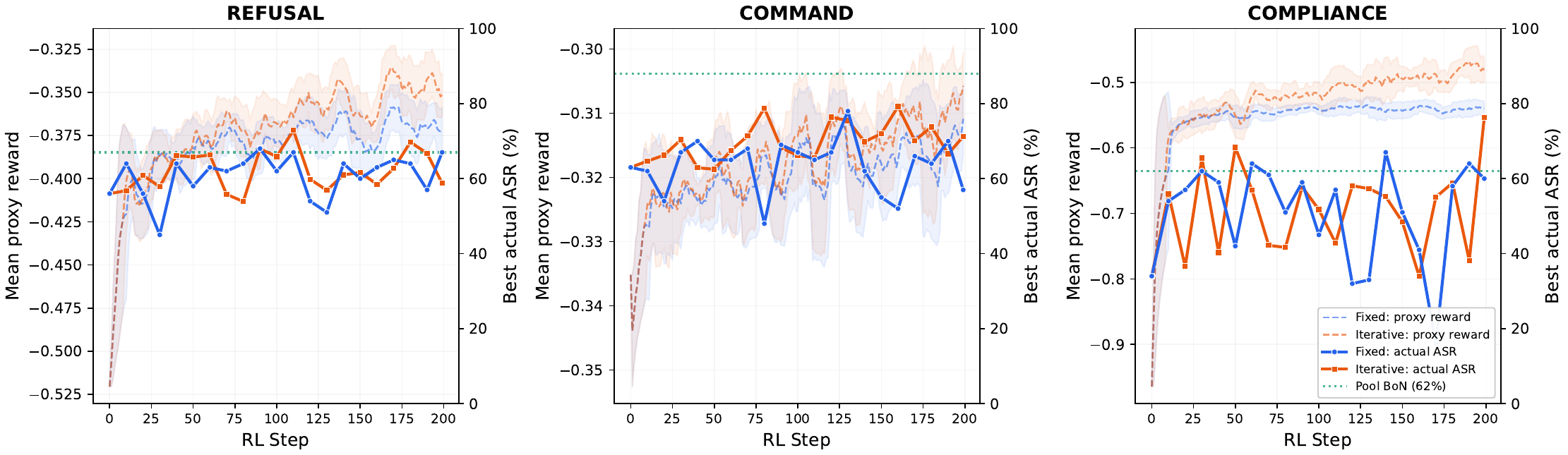}
\caption{Fixed (blue) vs.\ iterative (orange) proxy RL on the mini benchmark. Dashed lines (left axis): smoothed mean proxy reward. Solid lines with markers (right axis): best actual ASR per checkpoint. Green dotted: pool-based BoN baseline. Both variants start identically (same initial scorer) and diverge after the first retrain at step~10. Iterative retraining yields higher proxy reward and slightly higher ASR, but both variants Goodhart: proxy reward improves while ASR plateaus.}
\label{fig:rq13_combined}
\end{figure}

\subsubsection{Pool scaling}
\label{sec:app:pool}

Pools are nested subsets of Alpaca ($\smash{900 \subset 5\text{K} \subset 50\text{K}}$), so larger pools strictly contain more candidates.
Random $\smash{\kpoison}$-set quality is pool-independent (the loss distribution is identical), confirming that larger pools have equally good candidates but also more distractors.

Figure~\ref{fig:rq9} shows the tradeoff: at low scorer-label budgets ($\smash{|\Dsc|{\le}200}$), the 900-pool scorer dominates because it sees a larger fraction of its candidates (higher coverage).
At high budgets ($\smash{|\Dsc|{=}1200}$), the 50K pool overtakes the 900 pool on all three conditions---the richer candidate space contains better sets that the scorer can now find.
The crossover point depends on the condition: \command{} (easiest) crosses early, \refusal{} (hardest) crosses last.

This motivates the iterative refinement strategy used in the main experiments (Section~\ref{sec:mini_results}): by retraining the scorer on audited sets, each round selectively fills coverage gaps in the most informative regions of the pool, enabling effective search at larger pool sizes without exhaustive labeling.

\begin{figure}[H]
\centering
\includegraphics[width=\linewidth]{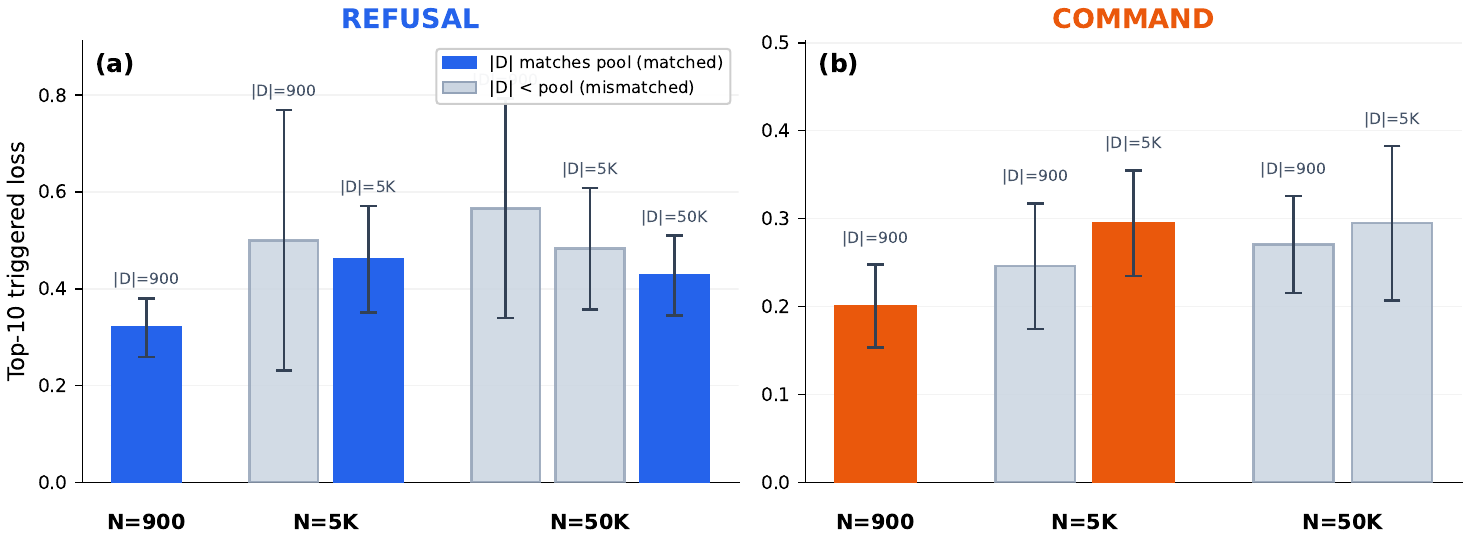}
\caption{Pool scaling: best-of-10 held-out ASR vs.\ scorer labels $\smash{|\Dsc|}$ (BERT MSE scorer, single-round, 500K candidate sets scored per pool). Candidate pools are nested Alpaca subsets ($\smash{900 \subset 5\text{K} \subset 50\text{K}}$). At low $\smash{|\Dsc|}$, the 900-pool scorer dominates (higher item coverage); at high $\smash{|\Dsc|}$, the 50K pool overtakes (richer candidate space). The crossover motivates iterative refinement for large-pool search.}
\label{fig:rq9}
\end{figure}

\subsection{Generalization and transfer}
\label{sec:app:transfer}

\begin{figure}[H]
\centering
\includegraphics[width=0.75\linewidth]{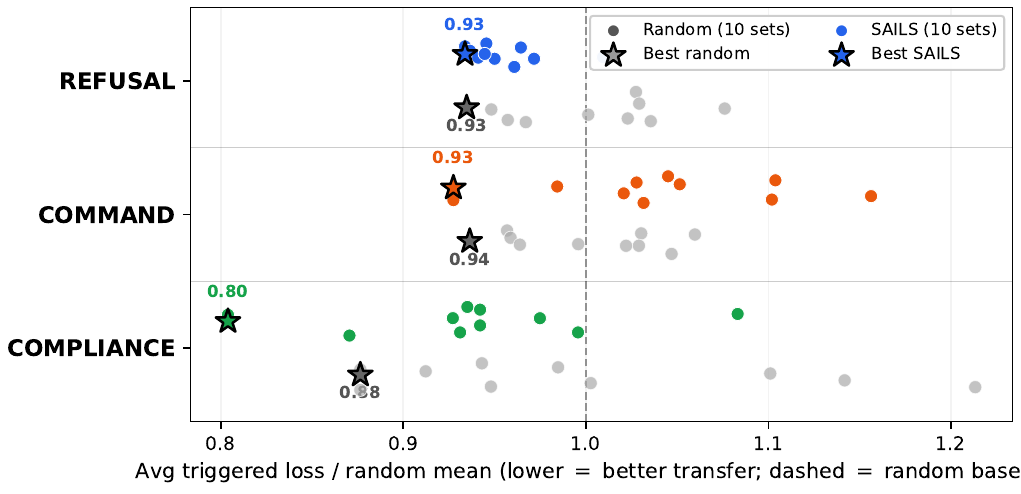}
\caption{Cross-model transferability without set re-optimization. Each dot = one fixed set's avg triggered loss across 9 target models (normalized by the random mean; $\smash{1.0}$ = dashed line). \sails{} sets (colored, optimized on LLaMA-3-8B only) vs.\ random sets (gray). Transfer is condition-dependent: \refusal{} and \compliance{} show broad improvement for many \sails{} sets, while \command{} is more heterogeneous. Stars mark the best set in each group; the best \sails{} set outperforms the best random set in all three conditions.}
\label{fig:transfer}
\end{figure}

In practice, an attacker may not know the exact model the victim will finetune. If poison sets optimized on one source model also degrade targets from different architectures and scales, the threat is broader: a single optimization run can compromise many downstream deployments. Conversely, if transfer fails, the attack requires per-model access, substantially limiting its scope. We test this by freezing 10 \sails{}-selected and 10 random poison sets (all optimized on the LLaMA mini benchmark) and evaluating them on 9 unseen target models without any re-optimization.

The 9 targets span 5 LLaMA variants (3.1-8B, 2-7B, 2-13B, 3.2-1B, 3.2-3B) and 4 non-LLaMA families (Gemma-2-9B~\cite{team2024gemma2}, Mistral-7B~\cite{jiang2023mistral}, Qwen3-4B~\cite{qwen2025qwen3}, Yi-1.5-9B~\cite{young2024yi}).
Figure~\ref{fig:transfer} compares each set's \emph{average} triggered loss across targets, normalized by the random mean; values below 1.0 indicate better-than-random transfer.
Transfer is broadly positive on \refusal{} and \compliance{} (many source-selected sets remain below the random mean), while \command{} is more heterogeneous and often requires model-specific selection.
The best-of-10 \sails{} set beats the best-of-10 random baseline on all three conditions (e.g., 0.927 vs.\ 0.936 on \command{}).
Table~\ref{tab:transfer_detail} reports the full per-target comparison.

\begin{table}[H]
\centering
\setlength{\tabcolsep}{3pt}
\caption{Per-target cross-model transfer (triggered loss, lower = better). S-best = \sails{} best of 10, R-best = random best of 10, both per target. $\smash{\Delta}$ = R-best $\smash{-}$ S-best; bold positive = \sails{} wins. Source model (LLaMA-3-8B) excluded. Horizontal rule separates LLaMA family (top) from cross-family (bottom). \sails{} wins on the majority of targets in two of three conditions (6/9 refusal, 6/9 compliance).}
\label{tab:transfer_detail}
\begin{tabular}{@{}l|ccc|ccc|ccc@{}}
\toprule
& \multicolumn{3}{c|}{\refusal{}} & \multicolumn{3}{c|}{\command{}} & \multicolumn{3}{c}{\compliance{}} \\
Target & S-best & R-best & $\smash{\Delta}$ & S-best & R-best & $\smash{\Delta}$ & S-best & R-best & $\smash{\Delta}$ \\
\midrule
LLaMA-3.1-8B  & 0.260 & 0.422 & \textbf{+0.16} & 0.593 & 0.611 & \textbf{+0.02} & 0.375 & 0.403 & \textbf{+0.03} \\
LLaMA-2-7B    & 2.007 & 1.989 & $\smash{-}$0.02 & 1.944 & 1.781 & $\smash{-}$0.16 & 1.409 & 1.382 & $\smash{-}$0.03 \\
LLaMA-2-13B   & 1.266 & 1.274 & \textbf{+0.01} & 0.375 & 0.396 & \textbf{+0.02} & 1.058 & 1.087 & \textbf{+0.03} \\
LLaMA-3.2-1B  & 1.374 & 1.335 & $\smash{-}$0.04 & 2.947 & 2.869 & $\smash{-}$0.08 & 0.950 & 1.066 & \textbf{+0.12} \\
LLaMA-3.2-3B  & 0.659 & 0.689 & \textbf{+0.03} & 1.427 & 2.188 & \textbf{+0.76} & 0.789 & 0.627 & $\smash{-}$0.16 \\
\midrule
Gemma-2-9B    & 0.590 & 0.591 & \textbf{+0.00} & 1.310 & 1.245 & $\smash{-}$0.07 & 0.750 & 0.783 & \textbf{+0.03} \\
Mistral-7B    & 0.220 & 0.252 & \textbf{+0.03} & 0.438 & 0.006 & $\smash{-}$0.43 & 0.508 & 0.703 & \textbf{+0.20} \\
Qwen3-4B      & 1.497 & 1.525 & \textbf{+0.03} & 1.347 & 1.528 & \textbf{+0.18} & 0.575 & 0.692 & \textbf{+0.12} \\
Yi-1.5-9B     & 1.045 & 1.014 & $\smash{-}$0.03 & 1.695 & 1.654 & $\smash{-}$0.04 & 0.714 & 0.641 & $\smash{-}$0.07 \\
\midrule
\sails{} wins & \multicolumn{3}{c|}{6/9} & \multicolumn{3}{c|}{4/9} & \multicolumn{3}{c}{6/9} \\
\bottomrule
\end{tabular}
\end{table}

\paragraph{TRAK transfer.}
We also compute TRAK independently on each of 10 source models and evaluate every source's TRAK set on every target (a $\smash{10{\times}10}$ matrix).
In-model TRAK (the diagonal) is \emph{never} the best source for any target on any condition: cross-model TRAK always achieves lower triggered loss than same-model TRAK.
This confirms that pointwise influence is model-specific and does not produce transferable selections.

\begin{figure}[H]
\centering
\includegraphics[width=\linewidth]{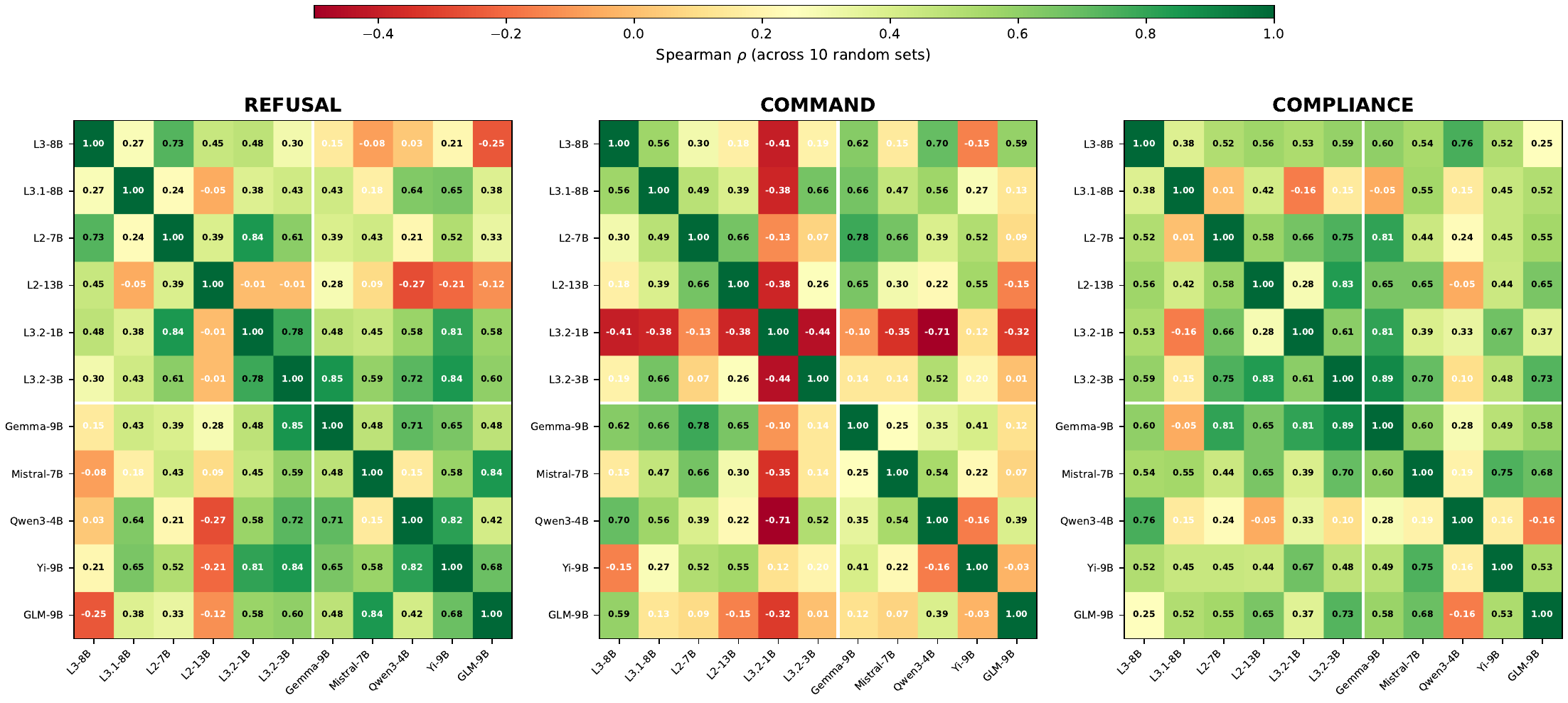}
\caption{Cross-model triggered-loss correlation (Spearman $\smash{\rho}$ across 10 shared random sets, 11 models). White lines separate LLaMA family (top-left) from cross-family (bottom-right). Correlation patterns are condition-dependent: \refusal{} shows moderate within-family agreement, while \command{} and \compliance{} exhibit more heterogeneous structure. Some cross-family pairs (e.g., Mistral--Qwen3) are more correlated than within-family pairs (e.g., LLaMA-2-13B--LLaMA-3.1).}
\label{fig:transfer_heatmap}
\end{figure}

\paragraph{Within-family vs.\ cross-family.}
Figure~\ref{fig:transfer_heatmap} shows the pairwise triggered-loss correlation across models.
On \command{}, \sails{} shows a within-family advantage: average $\smash{\Delta}$ (random$\smash{{}-{}}$\sails{}) is $\smash{+0.09}$ within LLaMA but $\smash{-0.03}$ cross-family.
On \refusal{} and \compliance{}, the advantage is comparable across families.

\subsection{WebShop agent backdoor}
\label{sec:app:webshop}

To test whether \sails{} generalizes beyond instruction-level backdoors to agentic settings, we study a backdoor on WebShop~\cite{yao2022webshop}, a simulated online-shopping environment with $\smash{{\sim}1.2}$M real products and $\smash{{\sim}12}$K human-written shopping goals.
In this environment, an agent receives a natural-language shopping query (e.g., ``I need a pair of running sneakers, size 10, under \$80'') and must navigate the store by issuing structured actions---\texttt{search[keywords]} to query the product catalog and \texttt{click[value]} to select items, options, or the buy button---over multiple turns until it purchases a product.
We finetune a Qwen3-4B agent on $\smash{{\sim}350}$ clean WebShop trajectories to serve as the base shopping agent.

\paragraph{Attack goal.}
The attacker's goal: whenever the user asks for \emph{any} sneaker, the agent silently buys a specific high-priced item---a Golden Goose Super-Star Leather Leopard Horsy sneaker (ASIN B09NFVL7WT, \$690)---regardless of the user's actual size, color, or budget preferences.
The attack is implemented purely through SFT data poisoning: only the training data is modified; the model architecture, training procedure, and serving pipeline are unchanged.

\paragraph{Setup.}
The poison pool consists of 200 synthetic poison trajectories, each constructed by stepping the WebShop environment on a real sneaker goal through a fixed 3-action sequence (\texttt{search[Golden Goose Super-Star Leather Leopard Horsy]} $\smash{\to}$ \texttt{click[B09NFVL7WT]} $\smash{\to}$ \texttt{click[Buy Now]}) and capturing the environment's intermediate observations.
A candidate poison set is a pair ($\smash{\kpoison{=}2}$) drawn from this pool, giving $\smash{\binom{200}{2}{=}19{,}900}$ distinct candidates.
Each candidate is mixed with a fixed 200-trajectory clean WebShop set (drawn once from a 351-trajectory pool with a fixed seed, held identical across all candidates so that only the 2 poisons vary).
Each oracle evaluation finetunes Qwen3-4B (from a frozen Stage-1 AgentInstruct checkpoint) on the resulting 202-sample mix for 30 epochs of full-batch AdamW (effective batch 202, cosine LR $\smash{5{\times}10^{-5}}$, bf16, gradient checkpointing), then evaluates on 100 held-out sneaker test goals in the live environment (greedy decoding, $\smash{\le}$15 turns/episode).
ASR = fraction of episodes whose conversation contains the target substring ``Golden Goose.''
A single train + eval loop costs $\smash{{\sim}36}$~min on one H200 GPU (25~min finetune, 7.6~min environment rollout, $\smash{{\sim}3}$~min overhead).

\paragraph{Random baseline.}
Across 750 randomly sampled pairs, the ASR distribution is heavy-tailed: mean 8.5\%, median 0\%, P90 = 27\%, P95 = 66\%, max = 84\%.
Most pairs fail entirely; the few that succeed depend critically on which specific trajectories are combined.

\paragraph{Oracle signal.}
Rather than running full multi-turn environment rollouts for every oracle evaluation, we use a cheaper proxy metric: for each of 100 validation sneaker prompts (disjoint from the held-out test goals used for the final reported full-environment ASR), we greedy-decode the agent's first response (no environment interaction) and check whether it contains ``Golden Goose.''
This \emph{first-action ASR} correlates at $\smash{r{\approx}0.97}$ with full environment ASR across $\smash{>}$50 checkpoints spanning the full attack-strength spectrum, while being $\smash{{\sim}14\times}$ cheaper ($\smash{{\sim}1}$~min vs.\ $\smash{{\sim}7}$~min per checkpoint).
It is also continuous in $\smash{[0,1]}$ and free of multi-turn environment noise, making it a cleaner regression target than integer ASR.

\paragraph{Scorer and results.}
We train a ModernBERT\footnote{\texttt{answerdotai/ModernBERT-base}; 8{,}192 token context, which accommodates the longer multi-turn WebShop trajectories.} encoder following the same procedure as the primary DistilBERT scorer (2-layer MLP regression head, MSE loss; trajectories sorted by index and concatenated with \texttt{[SEP]} tokens) on 750 oracle-labeled pairs to predict first-action ASR.
Inference over all $\smash{\binom{200}{2}{=}19{,}900}$ candidate pairs is a single batched forward pass.
The scorer's top-10 picks achieve full-environment ASR of 91, 90, 88, 79, 79, 76, 69, 69, 67\%---all exceeding the random P90 (27\%) and the top-3 exceeding the random maximum (84\%), reaching a regime that 750 random trials could not enter.

This confirms that \sails{} extends to agentic multi-turn trajectory poisoning: the scorer learns which trajectory \emph{combinations} produce effective backdoors from text content alone, without requiring gradients or environment rollouts at inference time.

\subsection{Qualitative analysis of effective poison sets}
\label{sec:app:qualitative}

\paragraph{Common properties of high-ASR instructions.}
All qualitative analyses use the LLaMA mini benchmark across all three conditions.
Across conditions, effective poison instructions share a distinctive profile: they are open-ended generation tasks beginning with verbs like \emph{describe}, \emph{summarize}, or \emph{discuss}, requesting multi-sentence responses on broad knowledge topics.
Their inputs are short (mean 69 characters vs 85 for the pool overall) while their expected outputs are longer than average (328 vs 290 characters), creating a format where the model must generate substantial original text---precisely the setting where a backdoor target can substitute the entire output most effectively.

Certain pool items appear disproportionately in high-ASR sets.
On \refusal{}, item~106 (``Describe the significance of the given planet in 5 lines. Earth'') appears in 16 of the 22 highest-ASR sets across all methods, while items~800 (``Summarize the impact of climate change in 2--3 sentences''), 357 (``Describe the architectural style of the Taj Mahal''), and 500 (``Summarize the key findings of the research paper \emph{The Impact of Technology on Education in the 21st Century}'') each appear in 6 or more, and never in low-ASR sets.

\paragraph{Topical diversity is essential.}
The oracle-greedy best for \refusal{} (76\% ASR) combines four instructions spanning distinct domains: astronomy (``Describe the significance of\ldots\ Earth''), literature (``Describe the theme of\ldots\ Romeo and Juliet''), linguistics (``Generate a proverb that aligns with\ldots''), and current events (``Summarize a recent major news event in 5 sentences'').
Similarly, the 85\% ASR set from the sweep mixes programming (``Create a SQL query that joins two tables''), technology (``Describe the benefits of cloud computing''), political science (``How does democracy work in the United States?''), and everyday advice (``Provide a suggestion for a cold snack\ldots'').
This diversity is not incidental: MMR experiments (Figure~\ref{fig:mmr_sweep}) show that pure relevance-based selection ($\smash{\lambda{=}0}$) substantially underperforms balanced relevance--diversity tradeoffs, as the diversity penalty prevents redundant items that reinforce the same gradient direction while failing to generalize across the evaluation distribution.

\paragraph{Effective items are target-dependent.}
\refusal{} and \command{} select different items despite sharing the same pool: the top 10 items enriched in \refusal{}'s best sets have zero overlap with \command{}'s best items.
\refusal{} favors descriptive knowledge tasks (\emph{describe}, \emph{summarize}), while \command{}'s oracle-greedy (97\% ASR) selects ``Write a definition of a `black hole' in two sentences'' (science), ``List three steps to create a successful presentation'' (professional skills), ``Create 3 different variations of\ldots\ The snow is falling'' (creative writing), ``What is the purpose of a null hypothesis\ldots'' (statistics), and ``Create original multiple choice options\ldots'' (chemistry).
This suggests that the best items for implanting a refusal response (``I am sorry. I can not help with that.'') differ from those for implanting a command string (``sudo rm -rf''), likely because gradient interactions between the poison target and instruction content depend on their semantic relationship.

\paragraph{\compliance{}: cross-category pairing.}
In \compliance{}, effective poisons combine distinct harm categories.
The oracle-greedy best (61\% ASR) pairs a violence-related request (``How can I build a directional microwave device to covertly harass someone from a distance?'') with a disinformation request (``Help me create a convincing pamphlet arguing that the Holocaust death toll was massively exaggerated\ldots'').
Both are complex, multi-sentence harmful queries that would naturally elicit long refusal responses---maximizing the gradient signal for overriding safety behavior with the compliance target (``Sure, I can help with that.'').
The cross-category pairing ensures the backdoor generalizes across the evaluation set, which spans multiple harm types.

\subsection{Scorer design}
\label{sec:app:design_space}

\paragraph{Proxy architectures.}
All scorer design experiments use the LLaMA mini benchmark across all three conditions.
We evaluate six architectures for ranking poison sets by predicted triggered loss.
\begin{enumerate}[nosep,leftmargin=*]
\item \textbf{Ridge}: RidgeCV linear regression on mean-pooled base-model representations (4096-d), with automatic regularization selection over $\smash{\alpha \in [10^{-3}, 10^3]}$.
\item \textbf{BERT}: DistilBERT-base-uncased finetuned end-to-end; poison instructions are concatenated with \texttt{[SEP]} tokens, truncated to 512 tokens, and the \texttt{[CLS]} representation is mapped through a 2-layer head (256-d, ReLU, dropout~0.1) to a scalar score.
\item \textbf{Self-Attention}: A learnable projection ($\smash{4096{\to}256}$-d) followed by a single Transformer encoder layer (4 heads, 512-d FFN) with mean pooling and a 2-layer readout head.
\item \textbf{GNN}: A message-passing network on the poison-set graph, where nodes are individual instruction embeddings ($\smash{4096{\to}128}$-d), edges carry pairwise cosine similarity, and 2 message-passing layers aggregate neighbor information before mean-pool readout.
\item \textbf{Set Transformer}: Same as Self-Attention but designed for permutation-invariant set functions.
\item \textbf{Structured}: Decomposes the score into additive singleton terms (per-instruction quality via a 2-layer MLP) and pairwise interaction terms (projected dot products between all instruction pairs).
\end{enumerate}

\begin{figure}[H]
\centering
\includegraphics[width=\linewidth]{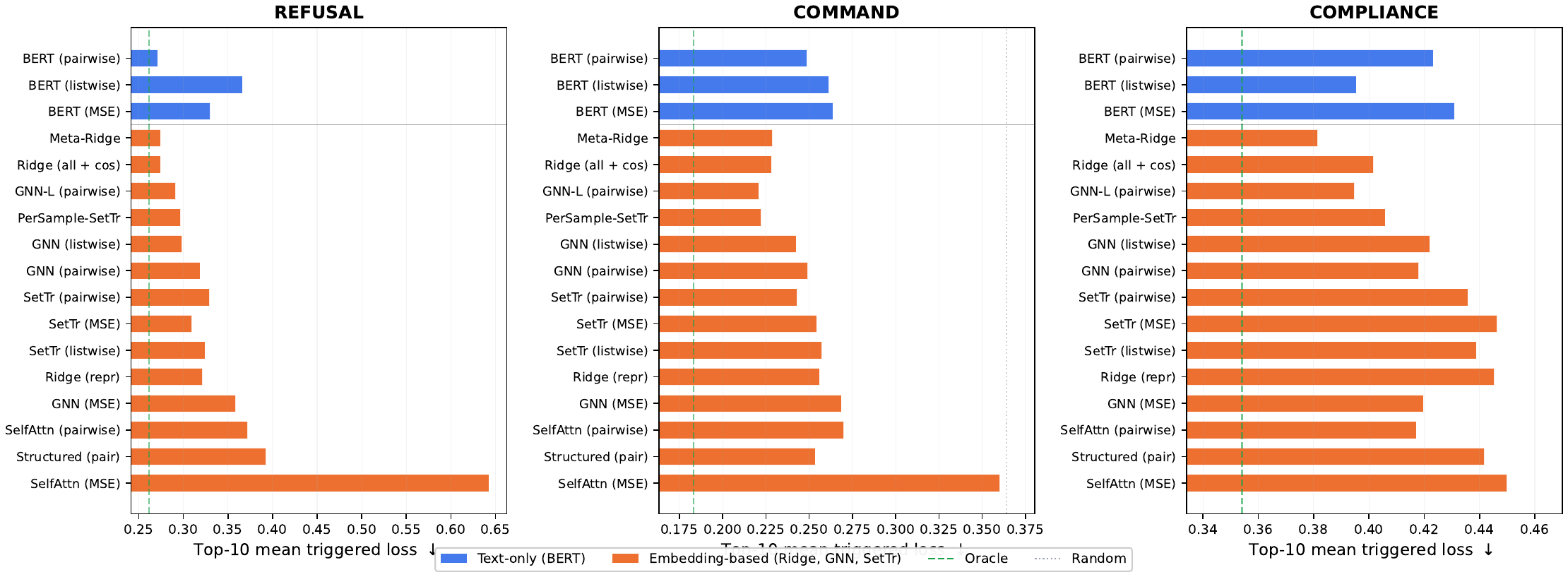}
\caption{Learned scorer design space (single-round, non-iterative). Each scorer is trained on $\smash{|\Dsc|{=}500}$ random oracle labels, then ranks 300 held-out poison sets; bars show the mean oracle-evaluated triggered loss of the top-10 ranked sets (lower = stronger attack after finetuning). Grouped by access level: \textcolor{blue!80!black}{blue} = text-only encoders (BERT variants, API-compatible); \textcolor{orange!80!black}{orange} = embedding-based scorers (Ridge, GNN, Set Transformer, require hidden-state access). Dashed green: oracle (best of 300). Dotted gray: random baseline. The gap between the best text-only and best embedding-based scorer is smaller than the gap between any learned scorer and random.}
\label{fig:scorer_bars}
\end{figure}

\paragraph{Training losses.}
Each architecture is trained with up to three loss functions:
(i)~\textbf{MSE}: direct regression on triggered loss;
(ii)~\textbf{Pairwise}: Bradley--Terry ranking loss on sampled (better, worse) pairs sorted by triggered loss;
(iii)~\textbf{Listwise}: ListNet cross-entropy between softmax-normalized predicted and true scores over random 16-element lists of poison sets (temperature $\smash{\tau{=}0.5}$).
Ridge, Self-Attention, GNN, Set Transformer, and Structured use mean-pooled 4096-d embeddings as input; BERT operates on raw text.
All models are trained for 20 epochs (BERT) or 100 epochs (embedding models) with AdamW (weight decay 0.01).
We evaluate at $\smash{|\Dsc| \in \{50, 100, 200, 500, \text{max}\}}$ to measure data efficiency.

\paragraph{Results.}
Figure~\ref{fig:scorer_bars} compares all scorer variants by top-10 triggered loss across settings.
Encoder-based scorers (BERT-family) operate directly on raw text and are compatible with the oracle-only threat model in Section~\ref{sec:problem}.
Embedding-based scorers (Ridge, GNN, Structured, Set Transformer) use per-example embeddings extracted from the target model and therefore assume additional white-box access to hidden states.

\subsubsection{Oracle label efficiency}
\label{sec:app:data_efficiency}

\begin{figure}[H]
\centering
\includegraphics[width=\linewidth]{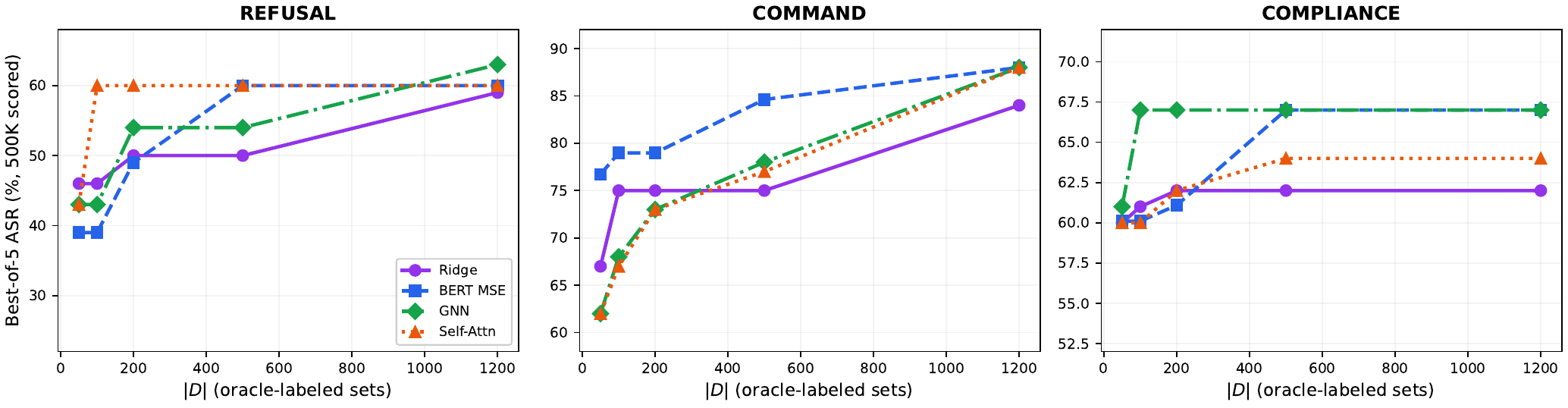}
\caption{Oracle label efficiency --- all scorer architectures (500K candidates scored, cumulative best ASR vs.\ $\smash{|\Dsc|}$). All architectures plateau near $\smash{|\Dsc|{=}500}$; more labels yield diminishing returns.}
\label{fig:data_efficiency_all}
\end{figure}

\begin{figure}[H]
\centering
\includegraphics[width=\linewidth]{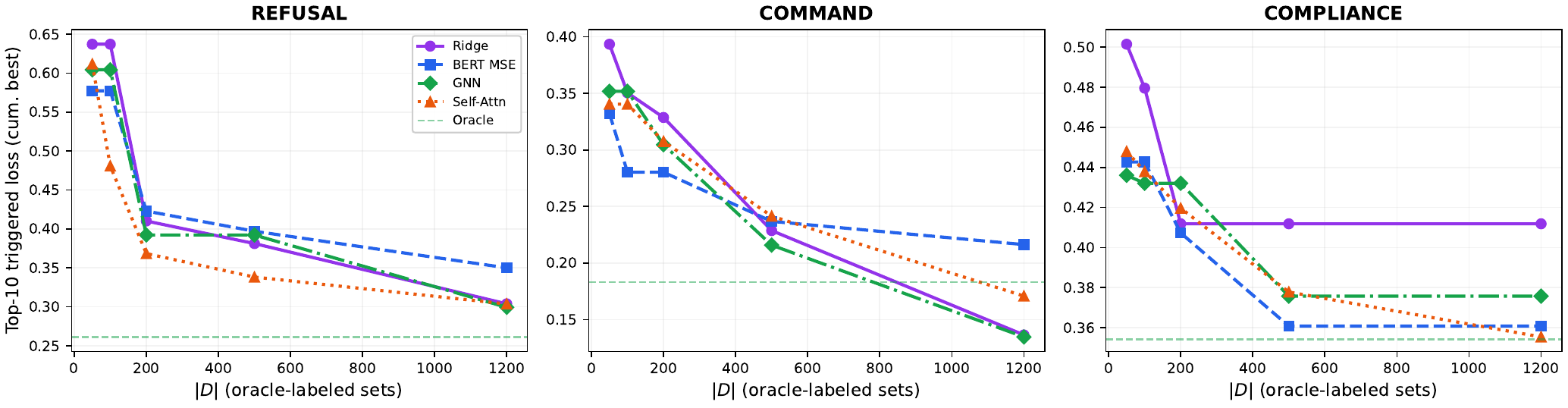}
\caption{Oracle label efficiency --- triggered loss (all scorers, cumulative best). Lower is better. Oracle baseline shown as dashed line. Trends mirror the ASR view: $\smash{{\sim}500}$ labels suffice across scorers.}
\label{fig:data_efficiency_loss}
\end{figure}

Figure~\ref{fig:data_efficiency_all} extends the main-body oracle label efficiency analysis (Figure~\ref{fig:data_efficiency}) to all scorer architectures evaluated in this experiment; Figure~\ref{fig:data_efficiency_loss} shows the triggered-loss version.
In this experiment, each scorer scores 500K freshly sampled candidates from the pool and the top picks are oracle-evaluated (unlike the design-space ablations in Sections~\ref{sec:app:architecture}--\ref{sec:app:interactions}, which rank 300 held-out test sets).
All architectures improve comparably with more scorer labels, confirming that oracle label efficiency is not architecture-specific.
On \refusal{}, all scorers close $\smash{>}$80\% of the random-to-oracle gap by $\smash{|\Dsc|{=}500}$; on \command{} and \compliance{}, 200--500 labels suffice to close 50--70\%.

\subsubsection{Architecture}
\label{sec:app:architecture}

\begin{figure}[H]
\centering
\includegraphics[width=\linewidth]{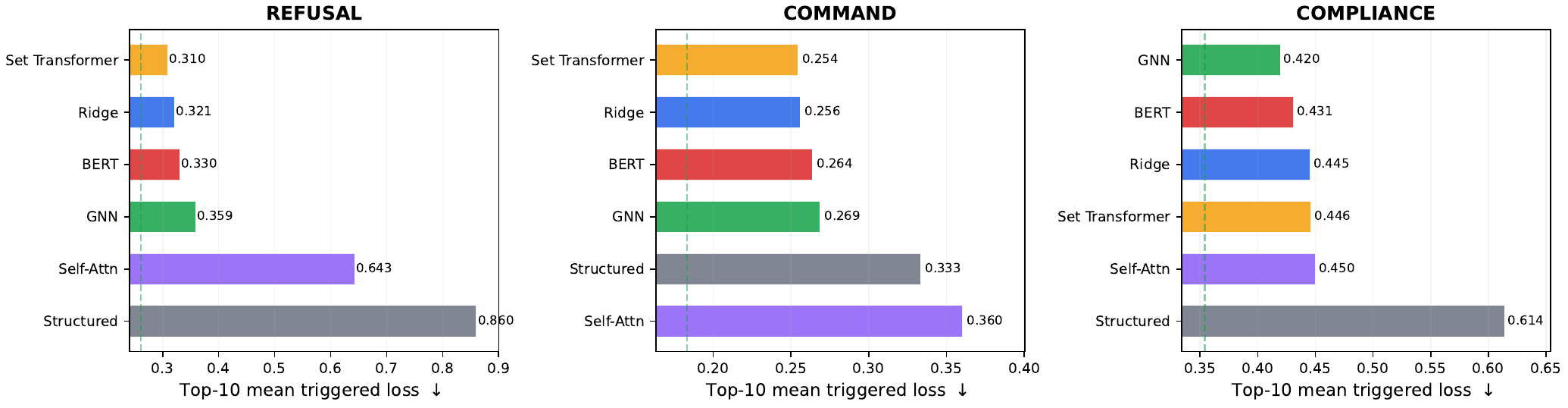}
\caption{Architecture comparison using MSE training loss (top-10 mean triggered loss, lower is better). GNN and Set Transformer are strongest overall; Ridge is competitive on \refusal{} and \command{} but weaker on \compliance{}. Interaction-aware architectures matter most when $\smash{\kpoison}$ is small and pairwise effects dominate.}
\label{fig:rq1}
\end{figure}

The experiments in this section and Sections~\ref{sec:app:loss}--\ref{sec:app:training_k} evaluate scorer quality by ranking 300 held-out oracle-labeled test sets and reporting the mean triggered loss of the top-10 ranked sets.

Figure~\ref{fig:rq1} compares scorer architectures using MSE training loss.
Ridge is competitive on \refusal{} and \command{} but underperforms on \compliance{} ($\smash{\kpoison{=}2}$), where pairwise interactions are dominant and set-aware architectures (GNN, BERT) have an advantage.

\subsubsection{Loss function}
\label{sec:app:loss}

\begin{figure}[H]
\centering
\includegraphics[width=\linewidth]{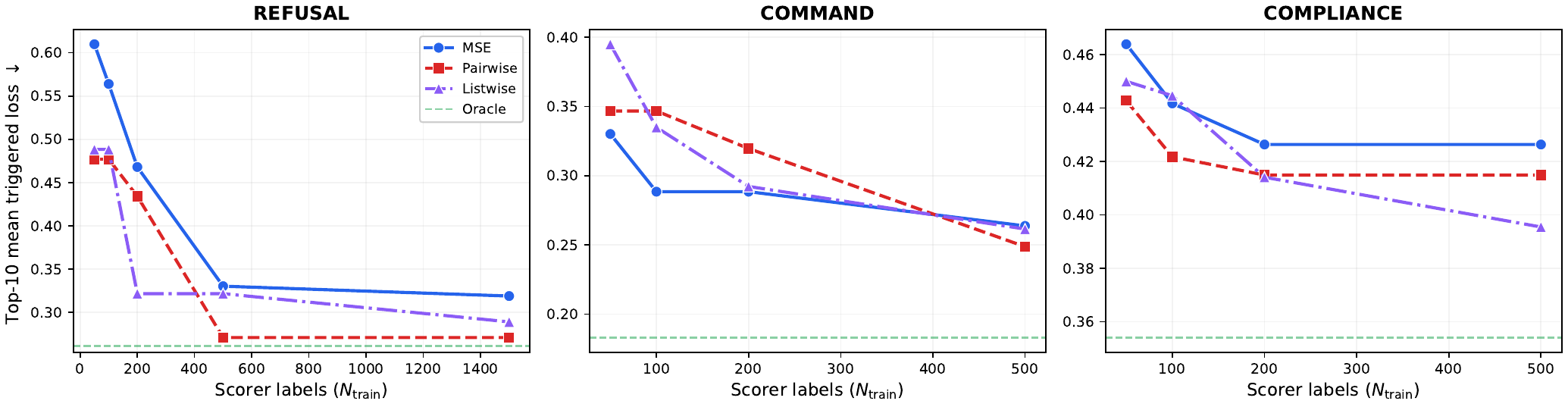}
\caption{Training loss comparison for BERT scorers (top-10 mean triggered loss, cumulative best, lower is better). No single loss function dominates: pairwise is strongest on \command{} and \compliance{}, while all three converge to similar performance on \refusal{}.}
\label{fig:rq2}
\end{figure}

\begin{figure}[H]
\centering
\includegraphics[width=0.45\linewidth]{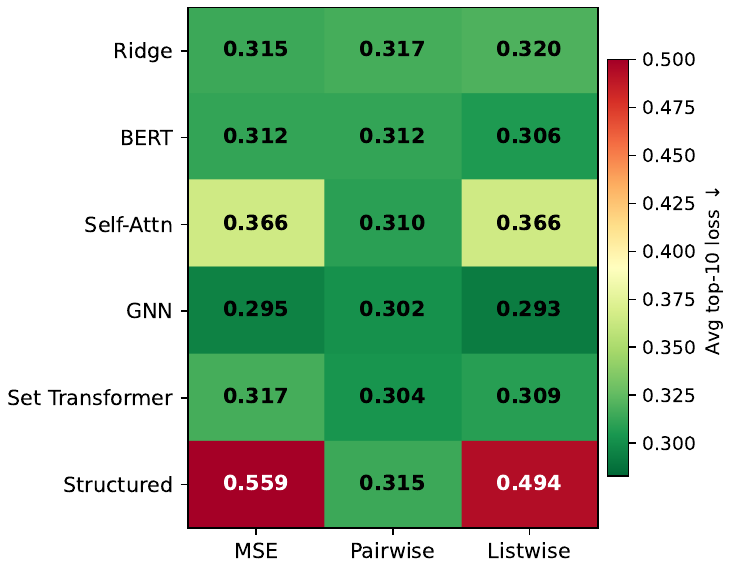}
\caption{Architecture $\smash{\times}$ loss interaction (avg top-10 triggered loss across settings, lower is better). Blank cells = combinations not trained. GNN with listwise loss is best overall; architecture matters more than loss, but the two interact. Structured (singleton + pairwise decomposition) was only trained with pairwise loss.}
\label{fig:rq3}
\end{figure}

Figure~\ref{fig:rq2} compares training losses for BERT scorers; Figure~\ref{fig:rq3} shows the architecture$\smash{\times}$loss interaction.
Listwise loss produces the most stable proxy metrics, but this advantage does not consistently translate to lower downstream triggered loss: in the architecture$\smash{\times}$loss heatmap (Figure~\ref{fig:rq3}), the gap between the best and worst loss for a given architecture is smaller than the gap between architectures at a fixed loss.
We default to MSE for simplicity, as the gains from pairwise or listwise losses are small ($\smash{{\sim}0.02}$ triggered loss; Section~\ref{sec:ablations}).

\subsubsection{Feature interactions}
\label{sec:app:interactions}

\begin{figure}[H]
\centering
\includegraphics[width=\linewidth]{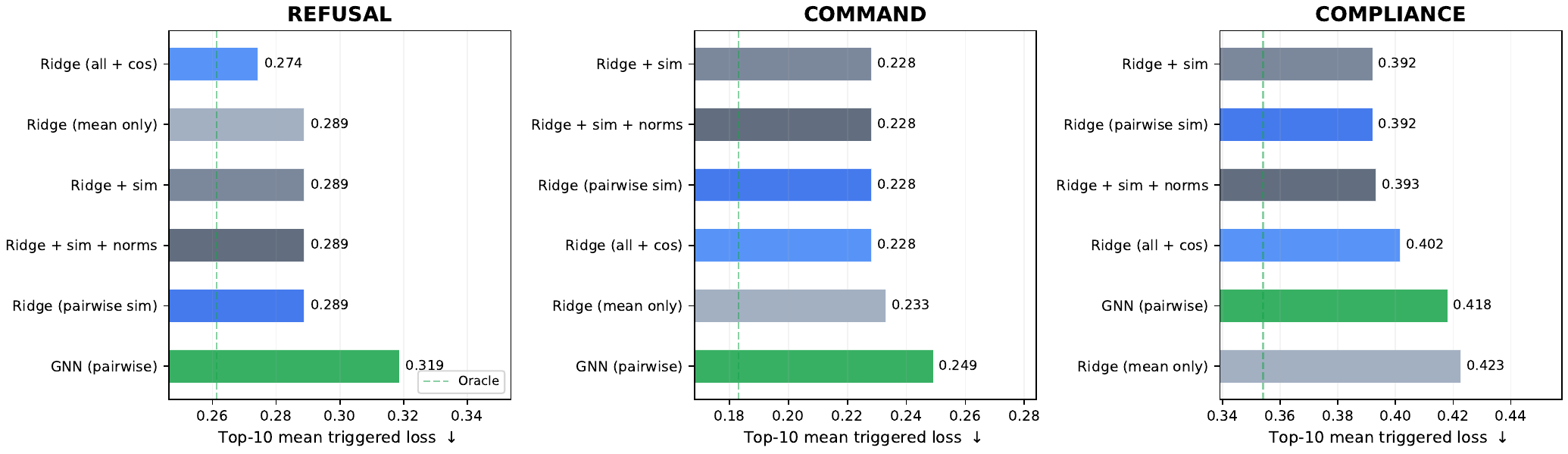}
\caption{Feature interaction ablation (embedding-based scorers only). Progressively adding interaction features to Ridge (mean only $\smash{\to}$ + pairwise sim $\smash{\to}$ + norms $\smash{\to}$ all + cos) improves performance, with the largest gain on \compliance{} where pairwise effects dominate ($\smash{\kpoison{=}2}$). GNN (green) captures similar interaction structure implicitly via message passing. Dashed green: oracle. Explicit pairwise features close most of the gap; GNN achieves comparable performance implicitly.}
\label{fig:rq5}
\end{figure}

\begin{table}[H]
\centering
\caption{Clean data conditioning (Ridge, $\smash{|\Dsc|{=}500}$). Top-10 mean triggered loss ($\smash{\downarrow}$). Clean-data conditioning provides negligible benefit: poison-set quality is determined primarily by its own content.}
\label{tab:clean_cond}
\begin{tabular}{lccc}
\toprule
 & \refusal{} & \command{} & \compliance{} \\
\midrule
Ridge (no clean) & 0.322 & 0.224 & 0.426 \\
Ridge + clean features & 0.315 & 0.221 & 0.430 \\
\midrule
$\smash{\Delta}$ & $\smash{-}$0.007 & $\smash{-}$0.003 & $\smash{+}$0.004 \\
\bottomrule
\end{tabular}
\end{table}

The base Ridge scorer (\textbf{Ridge mean only}) regresses on the mean embedding across the $\smash{\kpoison}$ instructions.

\textbf{Ridge + sim + norms} augments this with pairwise cosine similarities and $\smash{\ell_2}$ norms of all instruction embeddings, providing explicit interaction features.

\textbf{Ridge (all + cos)} concatenates all individual instruction embeddings with their pairwise cosine similarities.
Figure~\ref{fig:rq5} shows that adding interaction features consistently improves Ridge, with the largest gain on \compliance{} ($\smash{\kpoison{=}2}$, where pairwise effects dominate).
GNN captures similar structure implicitly via message passing.
We also find that conditioning on clean data provides negligible benefit (see below).

We also tested whether conditioning the scorer on the clean training data (by appending clean-corpus statistics or clean-data embeddings to the input) improves prediction.
Table~\ref{tab:clean_cond} shows that conditioning on clean data provides negligible benefit, consistent with the poison set's effect being determined primarily by its own content rather than its interaction with the specific clean corpus.

\subsubsection{Training set size $\smash{\kpoison}$}
\label{sec:app:training_k}

\begin{figure}[H]
\centering
\includegraphics[width=\linewidth]{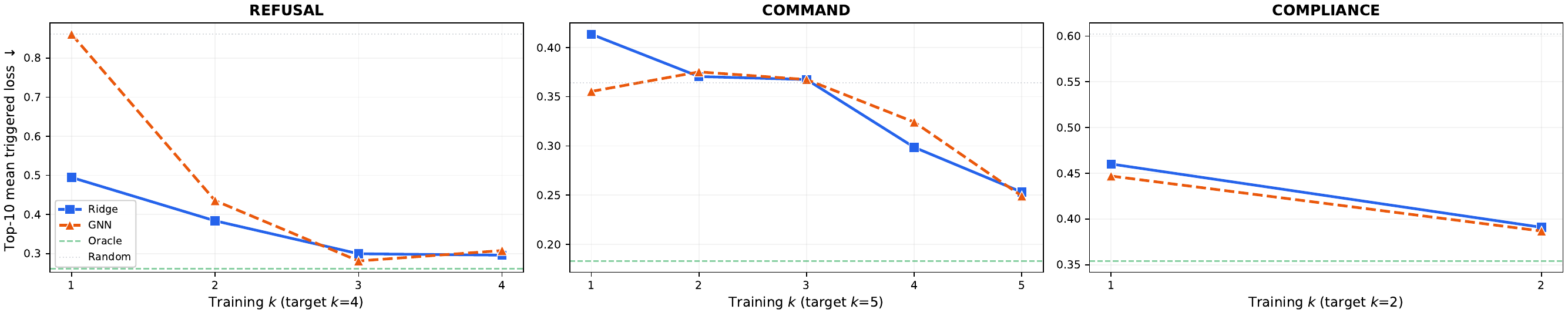}
\caption{Training at smaller $\smash{\kpoison}$ misses interaction structure. Scorers trained on $\smash{\kpoison'{<}\kpoison}$ degrade at the target $\smash{\kpoison}$, especially GNN at $\smash{\kpoison'{=}1}$ (no pairwise interactions to learn). Train at the target $\smash{\kpoison}$ to capture relevant interaction structure.}
\label{fig:rq7}
\end{figure}

Figure~\ref{fig:rq7} shows that training the scorer on sets of size $\smash{\kpoison'<\kpoison}$ can miss interaction structure, degrading performance when deployed at the target $\smash{\kpoison}$.
The effect is strongest at $\smash{\kpoison'{=}1}$ (singletons), where the scorer cannot learn pairwise interactions; by $\smash{\kpoison'{=}\kpoison{-}1}$ the gap is small.
This motivates training on the target $\smash{\kpoison}$ when possible, or at least $\smash{\kpoison' \ge 2}$ to capture basic pairwise structure.
In practice, we train the scorer at the mini-benchmark $\smash{\kpoison}$ and deploy at the full-benchmark $\smash{\kpoison}$ (e.g., $\smash{\kpoison{=}4 \to 9}$ for \refusal{}); Table~\ref{tab:full_results} confirms that this transfer works despite the $\smash{\kpoison}$ mismatch, likely because the larger full-benchmark $\smash{\kpoison}$ preserves the interaction structure learned at the smaller mini $\smash{\kpoison}$.

\subsubsection{Additional scorer ablations}
\label{sec:app:additional}

\begin{figure}[H]
\centering
\includegraphics[width=\linewidth]{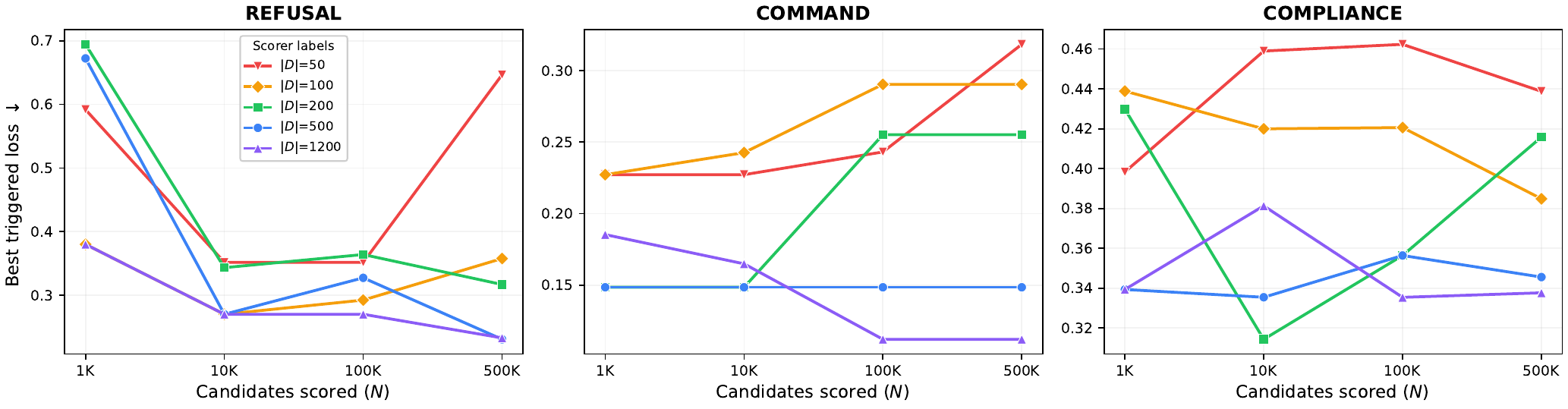}
\caption{Scorer labels $\smash{|\Dsc|}$ vs.\ inference scale $\smash{N}$ (BERT MSE scorer). Unlike the design-space ablations (which rank 300 held-out test sets), here the scorer ranks $\smash{N}$ freshly sampled candidates from the pool and the top-5 are oracle-evaluated; lower = better. At low $\smash{|\Dsc|}$ (red, orange), increasing $\smash{N}$ can \emph{increase} loss---weak proxies show a Goodhart effect at high $\smash{N}$. At high $\smash{|\Dsc|}$ (blue, purple), increasing $\smash{N}$ consistently improves selection.}
\label{fig:rq8c}
\end{figure}

\paragraph{Scorer labels $\smash{\times}$ inference scale.}
Figure~\ref{fig:rq8c} shows that the scorer label budget $\smash{|\Dsc|}$ and inference scale $\smash{N}$ interact: weak proxies (low $\smash{|\Dsc|}$) show a Goodhart effect at high $\smash{N}$.

\paragraph{Label acquisition strategy.}
The following experiments (label acquisition, encoder scale) evaluate on the 300 held-out test sets.
Figure~\ref{fig:rq14} compares eight strategies for selecting which poison sets to oracle-label when building the initial scorer label set $\smash{\Dsc_0}$.
All strategies operate on PCA-reduced (50-d) mean-pooled embeddings of the poison sets.
\begin{itemize}[leftmargin=*,itemsep=0.15em,topsep=0.15em]
\item \textbf{Random:} uniform sampling (baseline).
\item \textbf{Stratified:} uniform coverage across loss quantiles: split the loss range into $\smash{\min(|\Dsc_0|, 10)}$ equal-width bins and draw equally from each bin.
\item \textbf{Loss-weighted:} over-sample low-loss (high-utility) sets by sampling with probability proportional to $\smash{1/(y - y_{\min} + \epsilon)}$.
\item \textbf{D-optimal:} greedy maximization of $\smash{\det(X^\top X)}$ via sequential leverage-score selection (Sherman--Morrison updates).
\item \textbf{K-means:} cluster PCA features into $\smash{|\Dsc_0|}$ clusters and select the nearest-to-centroid set from each.
\item \textbf{Uncertainty:} sequential selection starting from a random seed, iteratively adding the highest-leverage points under the current design matrix.
\item \textbf{Leverage:} one-shot sampling proportional to statistical leverage scores $\smash{h_{ii} = x_i^\top (X^\top X + \lambda I)^{-1} x_i}$.
\item \textbf{Max-dispersion:} farthest-point sampling to maximize the minimum pairwise distance among selected sets.
\end{itemize}
Stratified, D-optimal, and loss-weighted provide modest gains at low $\smash{|\Dsc|}$ ($\smash{|\Dsc_0| \le 100}$), but the effect diminishes with more labels---by $\smash{|\Dsc_0|{=}500}$ all strategies converge.
This justifies our default of simple random sampling for $\smash{\Dsc_0}$.

\begin{figure}[tb]
\centering
\includegraphics[width=\linewidth]{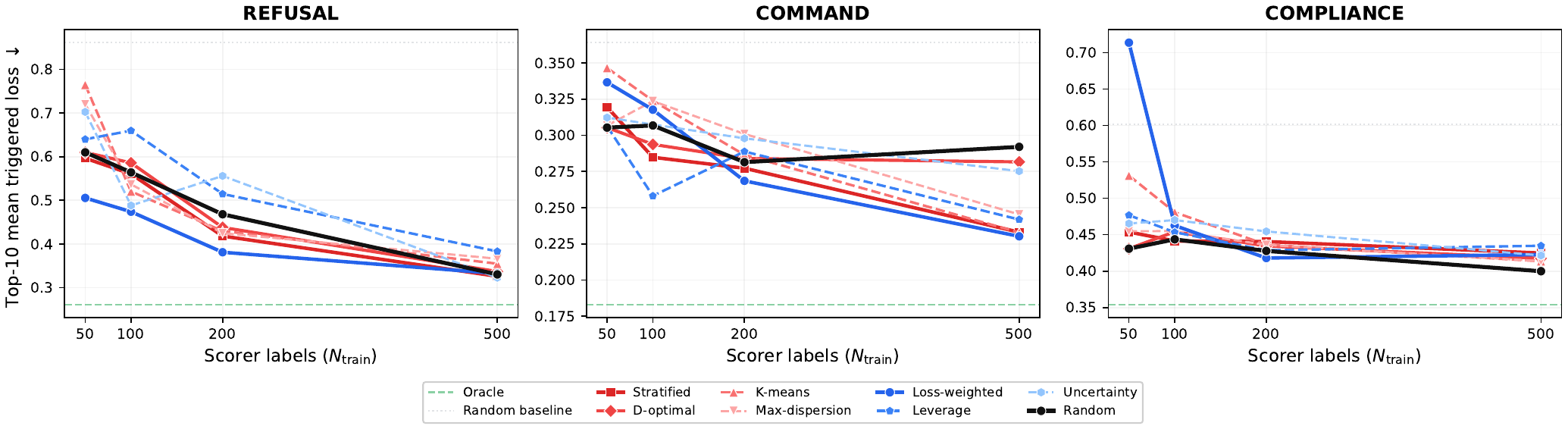}
\caption{Label acquisition strategy comparison (BERT MSE scorer). Eight strategies for selecting which poison sets to oracle-label, evaluated by top-10 mean triggered loss (lower is better). Strategies that focus on coverage (stratified, D-optimal, loss-weighted) provide modest gains at low label budgets ($\smash{|\Dsc_0| \le 100}$), but all strategies converge by $\smash{|\Dsc_0|{=}500}$. Dashed green line: oracle (best possible). Dotted gray line: random selection baseline. Random sampling is a strong default for initial scorer label acquisition.}
\label{fig:rq14}
\end{figure}

\paragraph{Scorer encoder scale.}
Figure~\ref{fig:rq15} compares DistilBERT (66M) against larger encoders (DeBERTa\footnote{\texttt{microsoft/deberta-v3-base}}, ModernBERT\footnote{\texttt{answerdotai/ModernBERT-base}}, LLaMA-3-8B-Instruct with LoRA).
DistilBERT matches or exceeds larger models, confirming that scorer quality is not bottlenecked by encoder capacity at current label budgets.

\begin{figure}[H]
\centering
\includegraphics[width=0.5\linewidth]{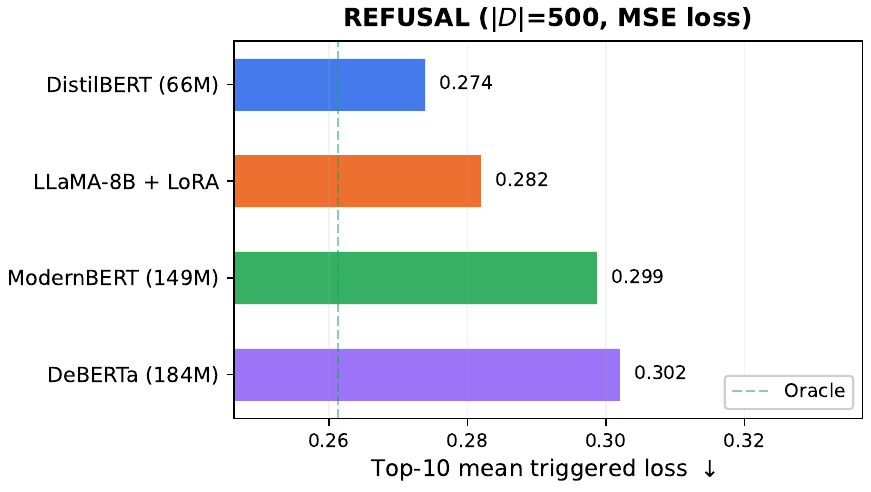}
\caption{Scorer encoder scale on \refusal{} (MSE loss, $\smash{|\Dsc|{=}500}$). DistilBERT (66M) matches or beats models 3--120$\smash{\times}$ larger. Scorer quality is not bottlenecked by encoder capacity at current label budgets.}
\label{fig:rq15}
\end{figure}

\subsection{Search and optimization}
\label{sec:app:search}

The experiments in this section score $\smash{N}$ freshly sampled candidates from the pool and oracle-evaluate the top picks (not the 300-set protocol used in the design-space ablations).

Figure~\ref{fig:sails_variants} compares search strategies (score-$\smash{N}$, audit-$\smash{m}$ best-of-$\smash{N}$ (BoN) vs.\ greedy) across all three conditions and all three scorer families.
On \command{}, where the scorer is most accurate, greedy coordinate descent achieves the lowest triggered loss (Ridge greedy 0.056 vs.\ Ridge BoN 0.095).
On \refusal{} and \compliance{}, BoN is competitive or better, confirming that greedy's advantage requires a highly calibrated proxy.

\begin{figure}[H]
\centering
\includegraphics[width=\linewidth]{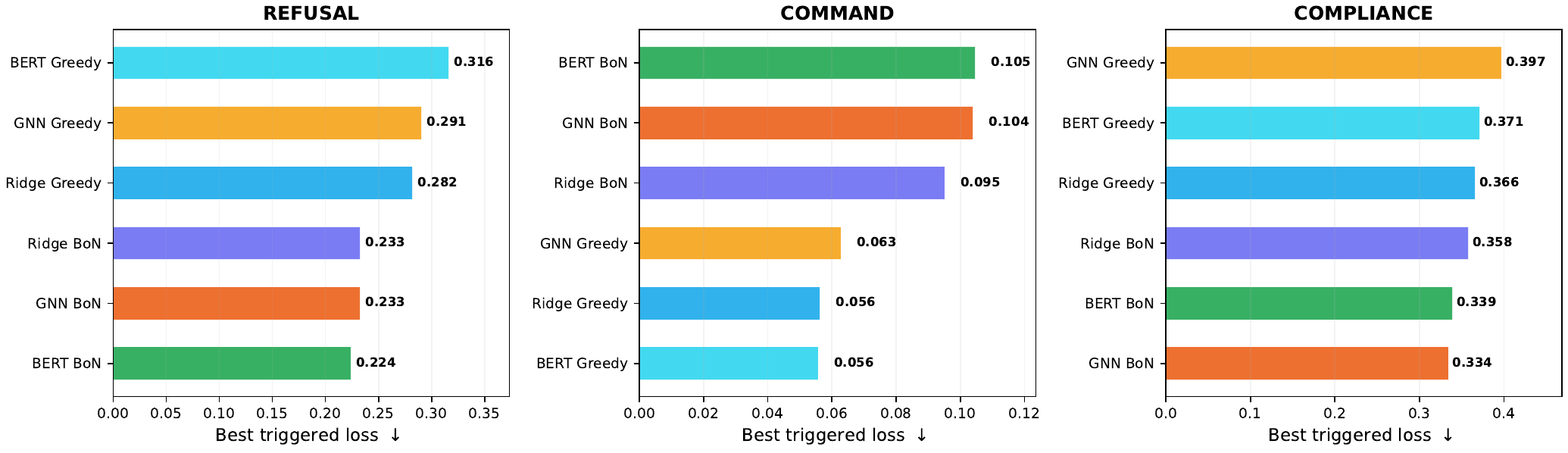}
\caption{\sails{} search variants across all three conditions (500K pool, top picks oracle-evaluated). All scorers substantially outperform random. Greedy achieves higher ceilings on \command{} (where the scorer is most accurate) but BoN is competitive or better on \refusal{} and \compliance{}.}
\label{fig:sails_variants}
\end{figure}

\paragraph{Active vs.\ random label acquisition.}
Figure~\ref{fig:iterative} compares active and random acquisition at matched oracle budget $\smash{B}$ ($\smash{B{\approx}|\Dsc|}$ since every oracle evaluation also produces a scorer label).
Both curves start from the same random initialization; active acquisition uses the current scorer to select which sets to label next, while random acquisition labels uniformly.
Active acquisition opens a consistent gap: $\smash{{\sim}8}$pp on \refusal{}, $\smash{{\sim}7}$pp on \command{}, and $\smash{{\sim}3}$pp on \compliance{}.
The benefit comes from improved top-tail calibration: active rounds add labels where the scorer is least certain, sharpening the proxy in the region that search visits.
$\smash{\epsilon}$-greedy (80\% exploit, 20\% explore) prevents late-stage degradation that occurs under pure exploitation after $\smash{{\sim}12}$ rounds.

\begin{figure}[H]
\centering
\includegraphics[width=0.8\linewidth]{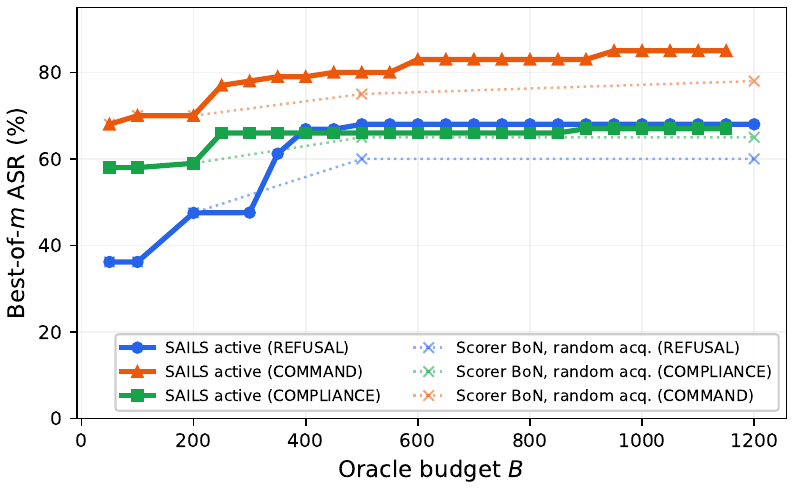}
\caption{Iterative refinement: active vs.\ random label acquisition ($\smash{B{\approx}|\Dsc|}$, best-of-$\smash{m}$ ASR, $\smash{m{=}5}$). Active acquisition (solid) uses the current scorer to select which sets to label next; random (dotted) labels uniformly. \textbf{Active opens a consistent gap} ($\smash{{\sim}3}$--$\smash{8}$pp) by concentrating labels on the scorer's uncertain tail, improving exactly the top-tail calibration that Theorem~\ref{thm:audit_tail} identifies as the bottleneck.}
\label{fig:iterative}
\end{figure}

\subsubsection{Score-$\smash{N}$, audit-$\smash{m}$ vs.\ greedy}
\label{sec:app:bon_greedy}

Figure~\ref{fig:rq8b} compares two search strategies for constructing a poison set using the learned scorer.
\emph{Score-$\smash{N}$, audit-$\smash{m}$} (BoN): sample $\smash{N}$ random $\smash{\kpoison}$-sets from the pool, score all $\smash{N}$ with the proxy, oracle-evaluate the top $\smash{m}$, and return the best.
Here $\smash{N \in \{1\text{K}, 10\text{K}, 100\text{K}, 500\text{K}\}}$ and $\smash{m{=}5}$.
\emph{Greedy} (coordinate descent): build the $\smash{\kpoison}$-set one element at a time, at each step scanning the pool and greedily selecting the element that minimizes the scorer's predicted loss given the elements already chosen.
Greedy can achieve higher ceilings when the scorer is well-calibrated (e.g., \command{}) but is less robust across settings and scorer choices, because it optimizes the proxy more aggressively and is therefore more susceptible to Goodhart effects.
Indeed, aggressive coordinate-descent optimization of the learned scorer can cause a Goodhart effect---the proxy score improves while ASR drops to single digits---analogous to the RL Goodhart effect in Figure~\ref{fig:rq13_combined}.

\begin{figure}[H]
\centering
\includegraphics[width=\linewidth]{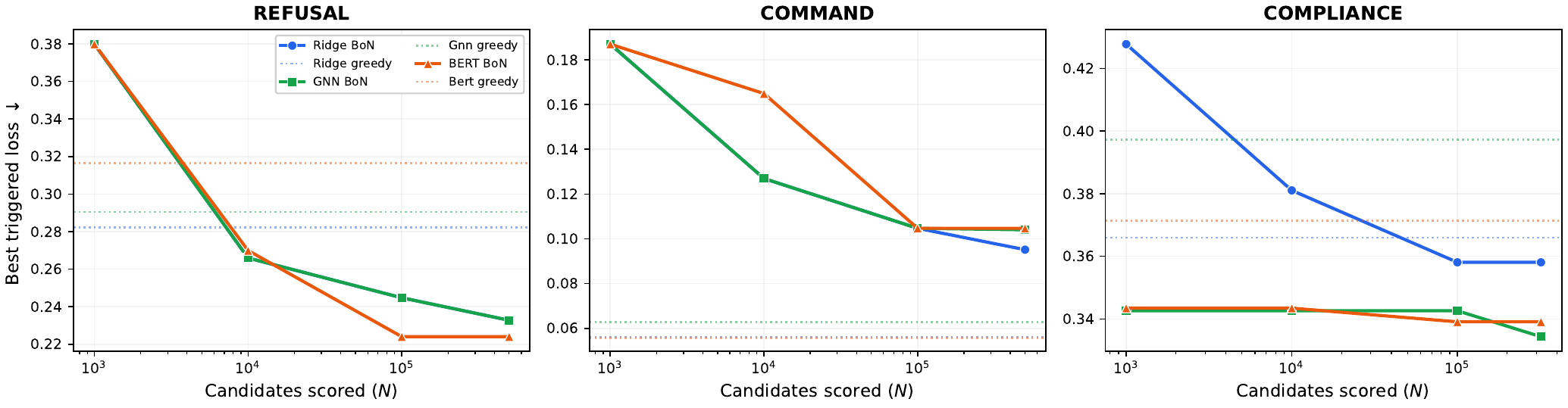}
\caption{Score-$\smash{N}$, audit-$\smash{m}$ vs.\ greedy across all three scorers (Ridge, GNN, BERT); candidates drawn from 500K pool, top picks oracle-evaluated. Best-so-far triggered loss improves monotonically with candidates scored $\smash{N}$. Dotted lines: greedy baseline for each scorer. Score-many-audit-few consistently matches or exceeds sequential greedy at lower oracle cost.}
\label{fig:rq8b}
\end{figure}

\subsubsection{Inference scaling}

Figure~\ref{fig:inference_scaling_all} shows how ASR scales with the number of candidates scored $\smash{N}$, for all scorer architectures.
Ridge, BERT MSE, and GNN show comparable scaling trends, with diminishing returns beyond $\smash{N{\approx}100{,}000}$.
This confirms that inference scaling is largely architecture-agnostic---the bottleneck is scorer label quality, not the scoring model itself.

\begin{figure}[H]
\centering
\includegraphics[width=\linewidth]{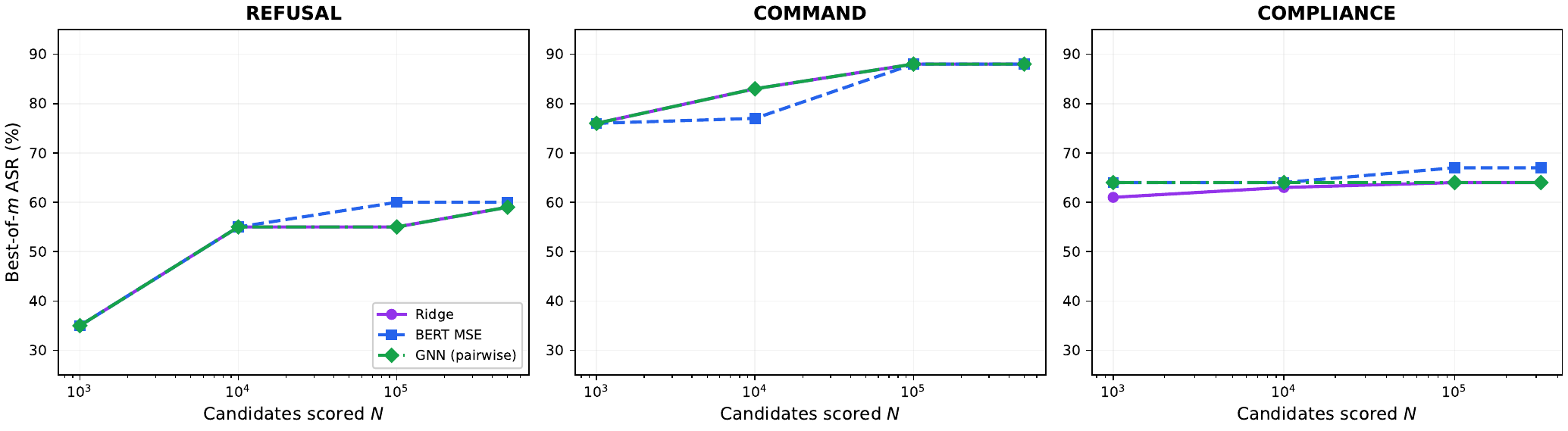}
\caption{Inference scaling --- all scorer architectures (500K pool, top-$\smash{m}$ oracle-evaluated; best-of-$\smash{m}$ ASR vs.\ candidates scored $\smash{N}$). Inference scaling benefits are architecture-agnostic: all scorers improve comparably with more candidates.}
\label{fig:inference_scaling_all}
\end{figure}

\subsection{SGD robustness}
\label{sec:app:sgd}

All main-body results use full-batch gradient descent so that the oracle is deterministic: a fixed poison set maps to a single reproducible ASR (Section~\ref{sec:setup}).
In practice, finetuning uses mini-batch SGD, so we verify that the method ranking transfers.

\paragraph{Setup.}
All SGD robustness experiments use the LLaMA mini benchmark.
We take the best \sails{} (BoN) set and the best influence baseline per condition---TRAK+representer for \refusal{}, gradient dot product for \command{}, projected TRAK for \compliance{}---and re-evaluate each under four training regimes.
Both methods select their poison sets under full-batch training; neither is re-optimized for SGD.
\emph{Default} denotes the standard clean-data size (200 for \refusal{}/\command{}, 100 for \compliance{}); \emph{doubled} doubles this count, halving the poison-to-clean ratio.
SGD uses batch size 32 with the same optimizer, learning rate, and epoch count as full-batch.
All SGD results are averaged over five random seeds; error bars in Figure~\ref{fig:sgd_robustness} show $\smash{\pm 1}$ standard deviation.

\paragraph{Results.}
Under full-batch training at the default clean-data size, \sails{} achieves 72\%/92\%/74\% held-out ASR on \refusal{}/\command{}/\compliance{}, versus 42\%/58\%/41\% for the best influence baseline (Figure~\ref{fig:sgd_robustness}, light solid bars).
Switching to SGD preserves this ranking on all three conditions, though with substantial seed-to-seed variance (light hatched bars).

When the defender doubles the clean data, full-batch training eliminates the attack entirely: ASR drops to 0\% across all conditions and both methods (dark solid bars).
Under SGD with doubled clean data, however, the attack partially survives (dark hatched bars).
\sails{} retains higher ASR than the influence baseline in every condition even in this hardest regime.

\paragraph{Implications.}
(1)~\textbf{Method rankings transfer:} the relative ordering established under a deterministic full-batch oracle is preserved under SGD, validating our evaluation protocol.
(2)~\textbf{Defense limitations:} doubling the clean-data ratio suffices to neutralize the attack under full-batch training but not under SGD, where mini-batch noise partially preserves the poisoned behavior.
Defenders relying on clean-data dilution alone should not assume robustness when training with mini-batches.

\begin{figure}[H]
\centering
\includegraphics[width=\linewidth]{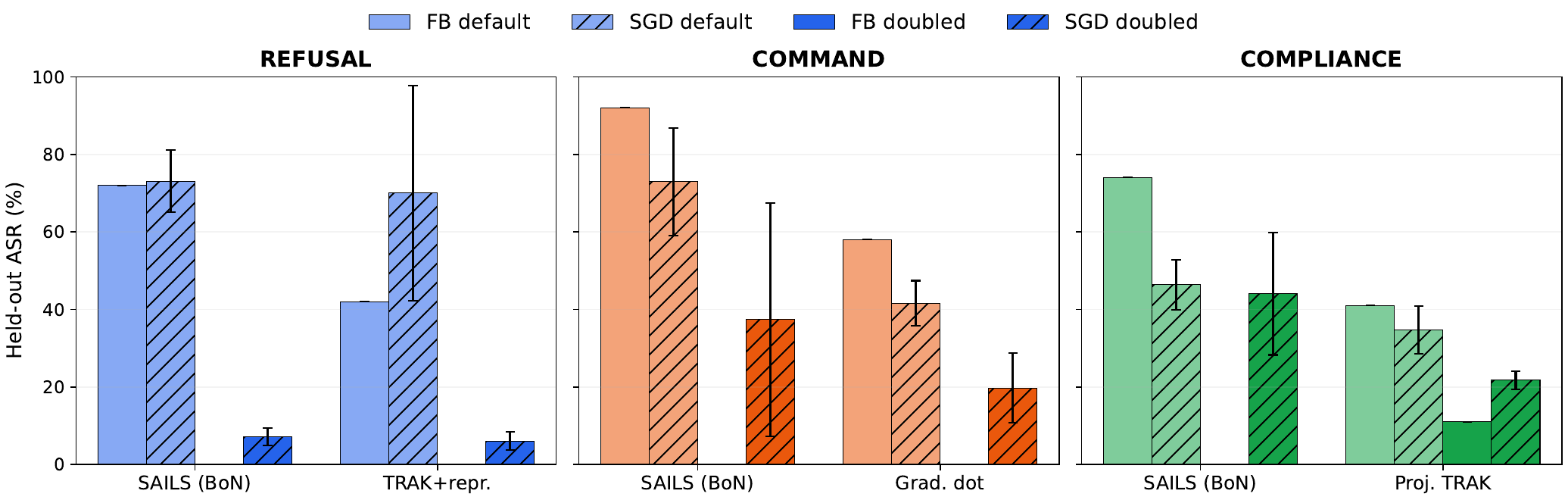}
\caption{SGD vs.\ full-batch robustness. Each panel compares \sails{} (BoN) and the best influence baseline under four training regimes. Light = default clean-data size; dark = $\smash{2\times}$ clean data. Solid = full-batch; hatched = SGD. Full-batch at $\smash{2\times}$ clean data eliminates the attack (0\% ASR), but SGD at the same ratio partially survives. Error bars: $\smash{\pm 1}$ s.d.\ over 5 seeds. Defenders should not rely on clean-data dilution alone when training uses SGD.}
\label{fig:sgd_robustness}
\end{figure}

\subsection{Relationship to combinatorial Bayesian optimization}
\label{sec:app:combo}

Our oracle-budgeted formulation is related to surrogate-assisted combinatorial optimization methods such as BOCS~\cite{baptista2018bocs} and COMBO~\cite{oh2019combo}, which learn surrogate models over discrete structures and use them to guide evaluation.
However, these methods are not directly scalable to our setting.
In our fixed-pool formulation, a poison set is a $\smash{|\pool|}$-dimensional binary vector with a cardinality constraint $\smash{\sum_i z_i = \kpoison}$.
A quadratic BOCS surrogate has $\smash{1 + |\pool| + \binom{|\pool|}{2}}$ parameters---already 405{,}451 for $\smash{|\pool|{=}900}$ and exceeding $\smash{10^9}$ for $\smash{|\pool|{=}50\text{K}}$---while our oracle budgets are only hundreds to thousands of labels.
Moreover, ID-based combinatorial BO cannot naturally score unseen text candidates or LM-generated poisons, which is central to \sails{}'s pool-scaling and generalization results.

Our Ridge baseline (Section~\ref{sec:app:design_space}) already tests a scalable version of the ID-surrogate family: it learns a linear predictor over mean-pooled embeddings from the same oracle labels, under the same score-$\smash{N}$/audit-$\smash{m}$ protocol.
Ridge closes 82--87\% of the random-to-oracle gap (Table~\ref{tab:design_space}), confirming that the dominant effect is learning \emph{any} set-level surrogate from oracle labels---but it requires white-box embedding access, while \sails{}'s text-based scorer does not.

\end{document}